%% file: main.tex
\documentclass{article}

\usepackage{float}
\usepackage{nicefrac}
\usepackage{graphicx}
\usepackage{xcolor}
\usepackage{subcaption}
\usepackage{makecell}
\usepackage{wrapfig}
\usepackage[inline]{enumitem}
\usepackage{booktabs,tabularx}
\usepackage{crossreftools}

\usepackage[accepted]{icml2024}

\def\independenT#1#2{\mathrel{\rlap{$#1#2$}\mkern2mu{#1#2}}}
\newcommand\indep{\protect\mathpalette{\protect\independenT}{\perp}}

\makeatletter
\newcommand{\customitemref}[2]{#2\def\@currentlabel{#2}\label{#1}}
\makeatother

\let\proof\relax 
\let\endproof\relax

\input{krkmath}

\usepackage{etoc}
\etocframedstyle[1]{\textbf{\textsc{Table of Contents}}}
\etocsettocdepth{3}

\definecolor{cherry1}{rgb}{0.215686, 0.215686, 0.215686}
\definecolor{cherry2}{rgb}{0.563899, 0.155919, 0.156577}
\definecolor{cherry3}{rgb}{0.747389, 0.178584, 0.180272}
\definecolor{cherry4}{rgb}{0.836168, 0.264453, 0.26819}
\definecolor{cherry5}{rgb}{0.880144, 0.397868, 0.404399}
\definecolor{cherry6}{rgb}{0.911942, 0.567676, 0.576412}

\usepackage{tikz}
\usepackage{pifont}
\usepackage{pgf}
\usepackage{pgfplots}
\usepackage{pgfplotstable}
\usepgfplotslibrary{groupplots}
\usetikzlibrary{math,calc}
\usetikzlibrary{positioning}
\usetikzlibrary{arrows}

\allowdisplaybreaks

\begin{document}
\twocolumn[
\icmltitle{Demystifying SGD with Doubly Stochastic Gradients}


\icmlsetsymbol{equal}{*}

\begin{icmlauthorlist}
\icmlauthor{Kyurae Kim}{penn}
\icmlauthor{Joohwan Ko}{kaist} 
\icmlauthor{Yi-An Ma}{ucsd}
\icmlauthor{Jacob R. Gardner}{penn}
\end{icmlauthorlist}

\icmlaffiliation{penn}{Department of Computer and Information Sciences, University of Pennsylvania, Philadelphia, PA, U.S.A.}
\icmlaffiliation{kaist}{KAIST, Daejeon, South Korea, Republic of}
\icmlaffiliation{ucsd}{Hal\i{}c\i{}o\u{g}lu Data Science Institute, University of California San Diego, San Diego, CA, U.S.A.}

\icmlcorrespondingauthor{Kyurae Kim}{kyrkim@seas.upenn.edu}

\icmlkeywords{Stochastic Optimization, Variational Inference, SGD, Stochastic Gradient Descent, Stochastic Optimization, Doubly Stochastic Optimization}

\vskip 0.3in
]



\printAffiliationsAndNotice{}  

\begin{abstract}
Optimization objectives in the form of a sum of intractable expectations are rising in importance (\textit{e.g.}, diffusion models, variational autoencoders, and many more), a setting also known as ``finite sum with infinite data.''
For these problems, a popular strategy is to employ SGD with \textit{doubly stochastic gradients} (doubly SGD): the expectations are estimated using the gradient estimator of each component, while the sum is estimated by subsampling over these estimators.
Despite its popularity, little is known about the convergence properties of doubly SGD, except under strong assumptions such as bounded variance.
In this work, we establish the convergence of doubly SGD with independent minibatching and random reshuffling under general conditions, which encompasses dependent component gradient estimators.
In particular, for dependent estimators, our analysis allows fined-grained analysis of the effect correlations.
As a result, under a per-iteration computational budget of \(b \times m\), where \(b\) is the minibatch size and \(m\) is the number of Monte Carlo samples, our analysis suggests where one should invest most of the budget in general.
Furthermore, we prove that random reshuffling (RR) improves the complexity dependence on the subsampling noise.
\end{abstract}

\section{Introduction}\label{section:introduction}
\input{section_introduction}

\vspace{-1ex}
\section{Preliminaries}
\vspace{-1ex}
\input{section_background}

\section{Main Results}
\input{section_main}

\section{Simulation}
\input{section_evaluation}

\vspace{-0.5ex}
\section{Discussions}
\vspace{-0.5ex}
\input{section_discussions}

\clearpage
\section*{Acknowledgements}
The authors would like to thank the anonymous reviewers for their comments, Jason Altschuler (UPenn) for numerous suggestions that strengthened the work, and Trevor Campbell (UBC) for pointing out a typo.

K. Kim was supported by a gift from AWS AI to Penn Engineering's ASSET Center for Trustworthy AI; Y.-A. Ma was funded by the NSF Grants [NSF-SCALE MoDL-2134209] and [NSF-CCF-2112665 (TILOS)], the U.S. Department Of Energy, Office of Science, as well as the DARPA AIE program; J. R. Gardner was supported by NSF award [IIS-2145644].

\section*{Impact Statement}
This paper presents a theoretical analysis of stochastic gradient descent under doubly stochastic noise to broaden our understanding of the algorithm.
The work itself is theoretical, and we do not expect direct societal consequences, but SGD with doubly stochastic gradients is widely used in various aspects of machine learning and statistics.
Therefore, we inherit the societal impact of the downstream applications of SGD.

\bibliographystyle{icml2024}
\bibliography{references}

\clearpage
\onecolumn
\appendix 

{\hypersetup{linkbordercolor=black,linkcolor=black}
\tableofcontents
}

\twocolumn
\newpage
\section{Gradient Variance Conditions}\label{section:gradient_variance_conditions}
\input{section_relationships_conditions}

\clearpage
\section{Proofs}
\subsection{Auxiliary Lemmas (\cref{thm:variancewithoutreplacement,thm:varianceofsum,thm:geometric_complexity,thm:geometric_complexity_squared,thm:average_bregman})}\label{section:proof_auxiliary}
\vspace{2ex}
\printProofs[variancewithoutreplacement]

\newpage
\printProofs[varianceofsum]

\newpage
\printProofs[geometriccomplexity]

\newpage
\printProofs[geometriccomplexitysquared]

\newpage
\printProofs[averagebregman]

\clearpage
\subsection{Convergence of SGD (\cref{thm:expected_residual_gradient_variance_bound,thm:strongly_convex_sgd_convergence,thm:strongly_convex_sgd_complexity})}\label{section:proof_sgd_convergence}

\vspace{2ex}
\printProofs[expectedresidualgradientvariancebound]
\newpage
\printProofs[stronglyconvexsgdconvergence]
\newpage
\printProofs[stronglyconvexsgdcomplexity]

\newpage
\subsection{General Variance Bound (\cref{thm:doubly_stochastic_variance})}\label{section:doubly_stochastic_variance}
\vspace{2ex}
\input{thm_expected_variance_lemma}


\newpage
\printProofs[doublystochasticvariance]

\clearpage
\subsection{Doubly Stochastic Gradients}
\subsubsection{Expected Residual Condition (\cref{thm:general_conditions})}
\vspace{2ex}
\printProofs[generalconditions]

\newpage
\subsubsection{Bounded Variance Condition (\cref{thm:bounded_variance})}
\vspace{2ex}
\printProofs[boundedvariance]

\newpage
\subsubsection{Complexity Analysis (\cref{thm:strongly_convex_sgd_sampling_without_replacement})}\label{section:proof_sgd_complexity}
\vspace{2ex}
\printProofs[expectedresidualwihtoutreplacement]

\newpage
\printProofs[stronglyconvexsgdsamplingwithoutreplacement]


\newpage
\subsection{Random Reshuffling of Stochastic Gradients}
\subsubsection{Gradient Variance Conditions (\cref{thm:reshuffling_expected_smoothness}, \cref{thm:reshuffle_variance})}
\vspace{2ex}
\printProofs[expectedsmoothnessreshuffling]

\newpage
\printProofs[reshufflevariance]

\newpage
\subsubsection{Convergence Analysis (\cref{thm:strongly_convex_reshuffling_sgd_convergence})}
\vspace{2ex}
\printProofs[stronglyconvexreshufflingsgdconvergence]

\clearpage
\subsubsection{Complexity Analysis (\cref{thm:strongly_convex_reshuffling_sgd_complexity})}
\vspace{2ex}
\printProofs[stronglyconvexreshufflingsgdcomplexity]

\newpage
\section{Applications}\label{section:applications}
\subsection{ERM with Randomized Smoothing}\label{section:erm_smoothing}
\input{section_erm_noisy_data}

\clearpage
\subsection{Reparameterization Gradient}\label{section:reparam_gradient}
\input{section_variational_inference}

\end{document}

%% file: krkmath.tex
\usepackage{amssymb}

\usepackage{amsmath,amsthm}
\usepackage{mathtools}
\usepackage{iftex}
\usepackage{etoolbox}

\ifxetex
  \usepackage{microtype}
  \usepackage{unicode-math}
  \newcommand{\mathmat}[1]{\mathbfit{#1}}
  \newcommand{\mathvec}[1]{\mathbfit{#1}}
  \newcommand{\mathvecgreek}[1]{\mathbfit{#1}}
  \newcommand{\mathmatgreek}[1]{\mathbfit{#1}}
  
  \newcommand{\boldupright}[1]{\symbfup{#1}}

  \newcommand{\mathrv}[1]{\mathsfit{#1}}
  \newcommand{\mathrvvec}[1]{\mathbfsfit{#1}}
  \newcommand{\mathrvgreek}[1]{\mathsfit{#1}}
  \newcommand{\mathrvvecgreek}[1]{\mathbfsfit{#1}}

  \usepackage{fontspec}
\else
  \usepackage[notext,lcgreekalpha]{stix2}
  \usepackage{textcomp}

  \newcommand{\mathmat}[1]{\mathbfit{#1}}
  \newcommand{\mathvec}[1]{\mathbfit{#1}}
  \newcommand{\mathvecgreek}[1]{\mathbfit{#1}}
  \newcommand{\mathmatgreek}[1]{\mathbfit{#1}}
  
  \newcommand{\boldupright}[1]{\mathbf{#1}}

  \newcommand{\mathrv}[1]{\mathsfit{#1}}
  \newcommand{\mathrvvec}[1]{\mathbfsfit{#1}}
  \newcommand{\mathrvgreek}[1]{\mathsfit{#1}}
  \newcommand{\mathrvvecgreek}[1]{\mathbfsfit{#1}}
\fi

\usepackage{booktabs,threeparttable,arydshln}
\usepackage{multirow}
\usepackage{nicematrix}

\usepackage[algo2e, ruled]{algorithm2e}

\usepackage{pgfplots}
\pgfkeys{/pgfplots/tuftelike/.style={
  semithick,
  tick style={major tick length=4pt,semithick,black},
  separate axis lines,
  axis x line*=bottom,
  axis x line shift=2pt,
  xlabel shift=-2pt,
  axis y line*=left,
  axis y line shift=2pt,
  ylabel shift=-2pt}}
  
\usepackage{proof-at-the-end}

\usepackage{mdframed}
\mdfdefinestyle {mdleftbarcoloredstyle} {%
    leftline          = true,%
    rightline         = false,%
    topline           = false,%
    bottomline        = false,%
    linewidth         = 2.5pt,%
    backgroundcolor   = gray!10,%
    linecolor         = gray,%
    innertopmargin    = 1.5ex,%
    innerbottommargin = 0.7ex,%
    innerleftmargin   = 1.ex,%
    innerrightmargin  = 1.ex,%
    ntheorem          = false,%
}

\mdfdefinestyle {coloredstyle} {%
    leftline          = false,%
    rightline         = false,%
    topline           = false,%
    bottomline        = false,%
    backgroundcolor   = gray!10,%
    innertopmargin    = 1.5ex,%
    innerbottommargin = 0.7ex,%
    innerleftmargin   = 1.ex,%
    innerrightmargin  = 1.ex,%
    skipabove         = \topskip,%
    skipbelow         = \parskip,%
}

\declaretheoremstyle[
    spaceabove    = \parsep,
    spacebelow    = \parsep,
    bodyfont      = \normalfont\itshape,
]{theoremsty}

\declaretheorem[name=Proposition,style=theoremsty, mdframed={style = coloredstyle}]{proposition}
\declaretheorem[name=Corollary,  style=theoremsty, mdframed={style = coloredstyle}]{corollary}
\declaretheorem[name=Lemma,      style=theoremsty, mdframed={style = coloredstyle}]{lemma}
\declaretheorem[name=Theorem,    style=theoremsty, mdframed={style = coloredstyle}, numbered=no]{theorem*}
\declaretheorem[name=Proposition,style=theoremsty, mdframed={style = coloredstyle}, numbered=no]{proposition*}
\declaretheorem[name=Corollary,  style=theoremsty, mdframed={style = coloredstyle}, numbered=no]{corollary*}
\declaretheorem[name=Lemma,      style=theoremsty, mdframed={style = coloredstyle}, numbered=no]{lemma*}

\declaretheorem[name=Open Problem,style=theoremsty, mdframed={style = coloredstyle}, numbered=no]{openproblem*}
\declaretheorem[name=Conjecture,  style=theoremsty, mdframed={style = coloredstyle}, numbered=no]{conjecture*}

\declaretheoremstyle[
    spaceabove=\parsep,
    spacebelow=\parsep,
    bodyfont=\normalfont,
]{normalsty}
\declaretheorem[name=Remark,      style=normalsty]{remark}
\declaretheorem[name=Definition,  style=normalsty, mdframed={style = 
 coloredstyle}]{definition}
\declaretheorem[name=Assumption,  style=normalsty, mdframed={style = 
 coloredstyle}]{assumption}

\declaretheorem[name=Remark,      style=normalsty, numbered=no]{remark*}
\declaretheorem[name=Definition,  style=normalsty, mdframed={style = 
 coloredstyle}, numbered=no]{definition*}
\declaretheorem[name=Assumption,  style=normalsty, mdframed={style = 
 coloredstyle}, mdframed={style = 
 coloredstyle}numbered=no]{assumption*}

\makeatletter
\providecommand\IfFormatAtLeastTF{\@ifl@t@r\fmtversion}
\IfFormatAtLeastTF{2020/10/02}{%
  \usepackage[bookmarksopen=true,bookmarks=true,colorlinks=true,backref=page,breaklinks]{hyperref}
  \usepackage[capitalise, nameinlink]{cleveref}
  \renewcommand*{\backref}[1]{}
  \renewcommand*{\backrefalt}[4]{({\footnotesize%
  \ifcase #1 Not cited.%
    \or page~#2%
    \else pages #2%
  \fi%
  })}

}{%
  \usepackage{hyperref}
  \usepackage[capitalise, nameinlink]{cleveref}%
}
\makeatother

\definecolor{linkcolor}{HTML}{6929C4}
\definecolor{citecolor}{HTML}{0043CE}
\hypersetup{colorlinks={true},linkcolor=linkcolor,citecolor=citecolor}

\crefname{assumption}{Assumption}{Assumption}
\crefname{condition}{Condition}{Condition}

\crefname{framedtheorem}{Theorem}{Theorems}
\crefname{framedproposition}{Proposition}{Propositions}
\crefname{framedlemma}{Lemma}{Lemmas}

\makeatletter
\def\adl@drawiv#1#2#3{%
        \hskip.5\tabcolsep
        \xleaders#3{#2.5\@tempdimb #1{1}#2.5\@tempdimb}%
                #2\z@ plus1fil minus1fil\relax
        \hskip.5\tabcolsep}
\newcommand{\cdashlinelr}[1]{%
  \noalign{\vskip\aboverulesep
           \global\let\@dashdrawstore\adl@draw
           \global\let\adl@draw\adl@drawiv}
  \cdashline{#1}
  \noalign{\global\let\adl@draw\@dashdrawstore
           \vskip\belowrulesep}}
\makeatother

\pgfkeys{/prAtEnd/global custom defaults/.style={
    end,
    restate,
    text proof={Proof},
    text link={See the \hyperref[proof:prAtEnd\pratendcountercurrent]{\textit{full proof}} in page~\pageref{proof:prAtEnd\pratendcountercurrent}.}
  }
}

\DeclareMathOperator*{\minimize}{minimize}

\DeclareMathOperator*{\argmax}{arg\,max}
\DeclareMathOperator*{\argmin}{arg\,min} 

\newcommand*\xbar[1]{%
  \hbox{%
    \vbox{%
      \hrule height 0.6pt 
      \kern0.33ex
      \hbox{%
        \kern-0.1em
        \ensuremath{#1}%
        \kern-0.1em
      }%
    }%
  }%
} 

\newcommand{\E}[1]{\mathbb{E}\left[ #1 \right]}
\newcommand{\Esub}[2]{\mathbb{E}_{#1}\left[ #2 \right]}
\newcommand{\V}[1]{\mathbb{V}\left[ #1 \right]}
\newcommand{\Vsub}[2]{\mathbb{V}_{#1}\left[ #2 \right]}
\newcommand{\Cov}[1]{\mathrm{Cov}\left( #1 \right)}

\newcommand{\norm}[1]{{\left\lVert #1 \right\rVert}}
\newcommand{\abs}[1]{{\left| #1 \right|}}

\def\do#1{\csdef{v#1}{\mathvec{#1}}}
\docsvlist{a,b,c,d,e,f,g,h,i,j,k,l,m,n,o,p,q,r,s,t,u,v,w,x,y,z}
\docsvlist{A,B,C,D,E,F,G,H,I,J,K,L,M,N,O,P,Q,R,S,T,U,V,W,X,Y,Z}

\def\do#1{\csdef{v#1}{\mathvecgreek{\csname#1\endcsname}}}
\docsvlist{alpha,beta,gamma,delta,epsilon,zeta,eta,theta,iota,kappa,lambda,mu,nu,xi,omikron,pi,rho,sigma,tau,upsilon,phi,chi,psi,omega}

\def\do#1{\csdef{rv#1}{\mathrv{#1}}}
\docsvlist{a,b,c,d,e,f,g,h,i,j,k,l,m,n,o,p,q,r,s,t,u,v,w,x,y,z}
\docsvlist{A,B,C,D,E,F,G,H,I,J,K,L,M,N,O,P,Q,R,S,T,U,V,W,X,Y,Z}

\def\do#1{\csdef{rv#1}{\mathrvgreek{\csname#1\endcsname}}}
\docsvlist{alpha,beta,gamma,delta,epsilon,zeta,eta,theta,iota,kappa,lambda,mu,nu,xi,omikron,pi,rho,sigma,tau,upsilon,phi,chi,psi,omega}

\def\do#1{\csdef{rvv#1}{\mathrvvec{#1}}}
\docsvlist{a,b,c,d,e,f,g,h,i,j,k,l,m,n,o,p,q,r,s,t,u,v,w,x,y,z}

\def\do#1{\csdef{rvv#1}{\mathrvvecgreek{\csname#1\endcsname}}}
\docsvlist{alpha,beta,gamma,delta,epsilon,zeta,eta,theta,iota,kappa,lambda,mu,nu,xi,omikron,pi,rho,sigma,tau,upsilon,phi,chi,psi,omega}

\def\do#1{\csdef{rvm#1}{\mathrvvecgreek{#1}}}
\docsvlist{A,B,C,D,E,F,G,H,I,J,K,L,M,N,O,P,Q,R,S,T,U,V,W,X,Y,Z}

\def\do#1{\csdef{m#1}{\mathmat{#1}}}
\docsvlist{A,B,C,D,E,F,G,H,I,J,K,L,M,N,O,P,Q,R,S,T,U,V,W,X,Y,Z}

\def\do#1{\csdef{m#1}{\mathmatgreek{\csname#1\endcsname}}}
\docsvlist{Sigma,Lambda,Phi,Gamma,Theta,Delta}

\newcommand{\inner}[2]{\left\langle #1, #2 \right\rangle}

%% file: section_introduction.tex
Stochastic gradient descent (SGD; \citealp{robbins_stochastic_1951,bottou_online_1999,nemirovski_robust_2009,shalev-shwartz_pegasos_2011}) is the \textit{de facto} standard for solving large scale optimization problems of the form of finite sums such as
{%
\setlength{\abovedisplayskip}{1ex} \setlength{\abovedisplayshortskip}{1ex}
\setlength{\belowdisplayskip}{1ex} \setlength{\belowdisplayshortskip}{1ex}
\begin{equation}
  \minimize_{\vx \in \mathcal{X} \subseteq \mathbb{R}^d}\; 
  \left\{\;
  F\left(\vx\right) \;\triangleq\;  
  {\textstyle\frac{1}{n}  \sum_{i=1}^n} f_i\left(\vx\right)
  \;\right\}.
  \label{eq:objective}
\end{equation}
}%
When \(n\) is large, SGD quickly converges to low-accuracy solutions by subsampling over components \(f_1, \ldots, f_n\).
The properties of SGD on the finite sum class have received an immense amount of interest~\citep{bottou_optimization_2018} as it includes empirical risk minimization (ERM; \citealp{vapnik_principles_1991}).

Unfortunately, for an emerging large set of problems in machine learning, we may not have direct access to the components \(f_1, \ldots, f_n\).
That is, each \(f_i\) may be defined as an intractable expectation, or an ``infinite sum'' 
{%
\setlength{\abovedisplayskip}{1ex} \setlength{\abovedisplayshortskip}{1ex}
\setlength{\belowdisplayskip}{1ex} \setlength{\belowdisplayshortskip}{1ex}
\begin{equation}
  f_i(\vx) = \mathbb{E}_{\rvveta \sim \varphi} f_i\left(\vx; \rvveta\right),
  \label{eq:expectation_subcomponent}
\end{equation}
}%
where we only have access to the noise distribution \(\varphi\) and the integrand \(f_i\left(\vx; \rvveta\right)\), and \(\rvveta\) is a potentially continuous and unbounded source of stochasticity;
a setting \citet{zheng_lightweight_2018,bietti_stochastic_2017} have previously called ``finite sum with infinite data.'' 
Such problems include the training of diffusion models~\citep{sohl-dickstein_deep_2015,ho_denoising_2020,song_generative_2019}, variational autoencoders~\citep{kingma_autoencoding_2014,rezende_stochastic_2014}, solving ERM under differential privacy~\citep{bassily_private_2014,song_stochastic_2013}, and also classical problems such as variational inference~\citep{ranganath_black_2014,titsias_doubly_2014,kucukelbir_automatic_2017}, and variants of empirical risk minimization~\citep{dai_scalable_2014,bietti_stochastic_2017,shi_triply_2021,orvieto_explicit_2023,liu_noisy_2021a}.
In contrast to the finite sum setting where SGD has traditionally been applied, our problem takes the form of
{%
\setlength{\abovedisplayskip}{1ex} \setlength{\abovedisplayshortskip}{1ex}
\setlength{\belowdisplayskip}{1ex} \setlength{\belowdisplayshortskip}{1ex}
\begin{equation*}
       \minimize_{\vx \in \mathcal{X} \subseteq \mathbb{R}^d}\; 
  \left\{\;
  F\left(\vx\right) \;\triangleq\;  
  {
  {\textstyle\frac{1}{n} \sum_{i=1}^n} \mathbb{E}_{\rvveta \sim \varphi} f_i\left(\vx; \rvveta\right)
  }
  \;\right\}.
\end{equation*}
}%
These optimization problems are typically solved using SGD with \textit{doubly stochastic gradients} (doubly SGD; coined by \citealt{dai_scalable_2014,titsias_doubly_2014}), so-called because, in addition to subsampling over $f_{i}$, stochastic estimates of each component \(f_i\) are used.




Previous studies have relied on strong assumptions to analyze doubly stochastic gradients.
For instance, \citet{kulunchakov_estimate_2020,bietti_stochastic_2017,zheng_lightweight_2018} have 
\begin{enumerate*}[label=\textbf{(\roman*)}]
    \item assumed that the variance of each component estimator is bounded by a constant, which contradicts componentwise strong convexity~\citep{nguyen_sgd_2018} when \(\mathcal{X} = \mathbb{R}^d\),
    \item or that the integrand \(\nabla f_i\left(\vx; \veta\right)\), is \(L\)-Lipschitz smooth ``uniformly'' over \(\veta\).
\end{enumerate*}
That is, for any fixed \(\veta\) and \(i\),
{%
\setlength{\abovedisplayskip}{1ex} \setlength{\abovedisplayshortskip}{1ex}
\setlength{\belowdisplayskip}{1ex} \setlength{\belowdisplayshortskip}{1ex}
\[
   \norm{ \nabla f_i\left(\vx; \veta\right) - \nabla f_i\left(\vy; \veta\right)} \leq L \norm{\vx - \vy}_2^2
\]
}%
holds for all \((\vx, \vy) \in \mathcal{X}^2\).
Unfortunately, this only holds for additive noise and is otherwise unrealizable when \(\rvveta\) has an unbounded support.
Therefore, analyses relying on uniform smoothness obscure a lot of interesting behavior.
Meanwhile, weaker assumptions such as expected smoothness (ES;~\citealp{moulines_nonasymptotic_2011,gower_stochastic_2021}) have shown to be realizable even for complex gradient estimators~\citep{domke_provable_2019,kim_convergence_2023}.
Therefore, a key question is how these ES-type assumptions propagate to doubly stochastic estimators.
Among these, we focus on the expected residual (ER; \citealp{gower_sgd_2019}) condition.

Furthermore, in practice, certain applications of doubly SGD share the randomness \(\rvveta\) across the batch \(\rvB\).
(See \cref{section:dsgd} for examples.)
This introduces dependence between the gradient estimate for each component such that \(\nabla f_i\left(\vx; \rvveta\right) \not{\indep} \nabla f_j\left(\vx; \rvveta\right)\) for \(i, j \in \rvB\).
Little is known about the effect of this practice apart from some empirical results~\cite {kingma_variational_2015}.
For instance, when \(m\) Monte Carlo samples of \(\rvveta\) and a minibatch of size \(b\) are used, what is the trade-off between \(m\) and \(b\)?
To answer this question, we provide a theoretical analysis of doubly SGD that encompasses dependent gradient estimators.

\vspace{2ex}

\begin{mdframed}[
    frametitle      = \textcolor{cherry1}{Technical Contributions},
    backgroundcolor = cherry1!5,
    linewidth       = 2pt,  
    linecolor       = cherry1,  
    topline         = false,  
    rightline       = false,  
    bottomline      = false,  
    leftline        = true,  
]
{\hypersetup{linkbordercolor=black,linkcolor=black}
\begin{itemize}[leftmargin=1ex]
    \setlength\itemsep{0ex}
    \item \textbf{\cref{thm:doubly_stochastic_variance}}: For doubly stochastic estimators, we establish a general variance bound of the form of
{%
\setlength{\abovedisplayskip}{1ex} \setlength{\abovedisplayshortskip}{1ex}
\setlength{\belowdisplayskip}{.5ex} \setlength{\belowdisplayshortskip}{.5ex}
    \[
       \mathcal{O}\,\Bigg(\, \frac{\textstyle \frac{1}{n} \sum^n_{i=1} \sigma^2_{i} }{m b} + \rho \frac{ {\textstyle\left(\frac{1}{n} \sum^n_{i=1} \sigma_{i} \right)}^2  }{m} + \frac{\tau^2}{b} \,\Bigg),
    \]
}%
    where \(\sigma_{i}^2\) is the variance of the estimator of \(\nabla f_i\), \(\rho \in [0, 1]\) is the correlation between the estimators, and \(\tau^2\) is the variance of subsampling.
    
    \item \textbf{\cref{thm:general_conditions,thm:bounded_variance}}: 
    Using the general variance bound, we show that a doubly stochastic estimator subsampling over correlated estimators satisfying the ER condition and the bounded variance (BV; \cref{assumption:bounded_variance}; bounded only on the solution set) condition equally satisfies the ER and BV conditions as well.
    This is sufficient to guarantee the convergence of doubly SGD on convex, quasar convex, and non-convex smooth objectives. 
    
    \item \textbf{\cref{thm:strongly_convex_reshuffling_sgd_convergence}}: Under similar assumptions, we also prove the convergence of doubly SGD with random reshuffling (doubly SGD-RR), instead of independent subsampling, on a strongly convex objective with strongly convex components.
\end{itemize}
}%
\end{mdframed}

\newpage
\begin{mdframed}[
    frametitle      = \textcolor{cherry2}{Practical Insights},
    backgroundcolor = cherry2!5,
    linewidth       = 2pt,  
    linecolor       = cherry2,  
    topline         = false,  
    rightline       = false,  
    bottomline      = false,  
    leftline        = true,  
]
{\hypersetup{linkbordercolor=black,linkcolor=black}
\begin{itemize}[leftmargin=1ex]
    \setlength\itemsep{0ex}
    \item \textbf{Should I invest in (increase) \(m\) or \(b\)?} 
    When dependent gradient estimators are used, increasing \(m\) or \(b\) does not have the same impact on the gradient variance as the subsampling strategy also affects the resulting correlation between the estimators.
    Through \cref{thm:expected_variance_lemma}, our analysis provides insight into this effect. 
    In particular, we reveal that reducing subsampling variance also reduces Monte Carlo variances.
    Therefore, for a fixed budget \(m \times b\), increasing \(b\) should always be preferred over increasing \(m\).
    
    \item \textbf{Random Reshuffling Improves Complexity.} 
    Our analysis of doubly SGD-RR reveals that, for strongly convex objectives, random reshuffling improves the iteration complexity of doubly SGD from \(\mathcal{O}\left( \frac{1}{\epsilon} \sigma^2_{\mathrm{mc}} + \frac{1}{\epsilon} \sigma^2_{\mathrm{sub}} \right)\) to \(\mathcal{O}\left( \frac{1}{\epsilon} \sigma^2_{\mathrm{mc}} + \frac{1}{\sqrt{\epsilon}} \sigma_{\mathrm{sub}} \right)\).
    Furthermore, for dependent gradient estimators, doubly SGD-RR is ``super-efficient'':
    for a batch taking \(\Theta(m b)\) samples to compute, it achieves a \(n/b\) tighter asymptotic sample complexity compared to full-batch SGD.
\end{itemize}
}
\end{mdframed}

%% file: section_background.tex
\paragraph{Notation}
We denote random variables (RVs) in serif (\textit{e.g.}, \(\rvx\), \(\rvvx\), \(\rvmX\), \(\rvB\)), vectors and matrices in bold (\textit{e.g.}, \(\vx\), \(\rvvx\), \(\mA\), \(\rvmA\)).
For a vector \(\vx\), we denote the \(\ell_2\)-norm as \(\norm{\vx}_2 \triangleq \sqrt{\inner{\vx}{\vx}} = \sqrt{\vx^{\top}\vx}\), where \(\inner{\vx}{\vx} = \vx^{\top}\vx\) is the inner product.
Lastly, \(\rvX \indep \rvY\) denotes independence of \(\rvX\) and \(\rvY\).



\begin{table}[h]
    \setlength{\tabcolsep}{.3em}
    \vspace{-3ex}
    \centering
    \caption{Nomenclature}
    \vspace{1ex}
    \begin{tabular}{clc}
      \textbf{Symb.} & \multicolumn{1}{c}{\textbf{Description}} & \multicolumn{1}{c}{\textbf{Ref.}} \\ \midrule
      \( F\left(\vx\right) \) & Objective function & \cref{eq:objective} \\
      \( f_i\left(\vx\right) \) & \(i\)th component of \(F\) & \cref{eq:objective} \\
      \(\nabla f_{\rvB}\left(\vx\right) \) & Minibatch subsampling estimator of \(\nabla F\) & \cref{eq:minibatch_gradient} \\
      \(\rvB\) & Minibatch of component indices & \cref{eq:finite_sum} \\
      \(\pi\) & Minibatch subsampling strategy & \cref{eq:finite_sum} \\
      \(b_{\mathrm{eff}}\) & Effective sample size of \(\pi\) & \cref{eq:ess} \\
      \(\rvvg_i\left(\vx\right) \) & Unbiased stochastic estimator of \(\nabla f_i\) & \cref{eq:monte_carlo_gradient} \\
      \(\vg_i\left(\vx; \rvveta\right) \) & Integrand of estimator \(\rvvg_i\left(\vx\right)\) & \cref{eq:monte_carlo_gradient} \\
      \(\rvvg_{\rvB}\left(\vx\right) \) & Doubly stochastic estimator of \(\nabla F\) & \cref{eq:doubly_stochastic_gradient} 
      \\
      \(\mathcal{L}_{\mathrm{sub}}\) & ER constant (\cref{assumption:expected_residual}) of \(\pi\) & 
      \hyperref[assumption:subsampling_er]{Assu.~\labelcref*{assumption:subsampling_er}} 
      \\
      \(\mathcal{L}_i \) & ER constant (\cref{assumption:expected_residual}) of \(\rvvg_i\) & \hyperref[assumption:montecarlo_er]{Assu.~\labelcref*{assumption:montecarlo_er}} 
      \\
      \(\tau^2\) & BV constant (\cref{assumption:bounded_variance}) of \(\pi\) & 
      \hyperref[assumption:bounded_variance_both]{Assu.~\labelcref*{assumption:bounded_variance_both}} 
      \\
      \(\sigma_i^2 \) & BV constant (\cref{assumption:bounded_variance}) of \(\rvvg_i\) & 
      \hyperref[assumption:bounded_variance_both]{Assu.~\labelcref*{assumption:bounded_variance_both}} 
    \end{tabular}
    \vspace{-3ex}
\end{table}

\subsection{Stochastic Gradient Descent on Finite-Sums}
Stochastic gradient descent (SGD) is an optimization algorithm that repeats the steps
{%
\setlength{\abovedisplayskip}{1ex} \setlength{\abovedisplayshortskip}{1ex}
\setlength{\belowdisplayskip}{.5ex} \setlength{\belowdisplayshortskip}{.5ex}
\[
  \vx_{t+1} = \Pi_{\mathcal{X}}\left( \vx_t - \gamma_t \rvvg\left(\vx_t\right) \right),
\]
}%
where, \(\Pi_{\mathcal{X}}\) is a projection operator onto \(\mathcal{X}\), \({(\gamma_t)}_{i=0}^{T-1}\) is some stepsize schedule, \(\rvvg\left(\vx\right)\) is an unbiased estimate of \(\nabla F\left(\vx\right)\).

\vspace{-1.5ex}
\paragraph{Finite-Sum Problems.}
When the objective can be represented as a ``finite sum'' it is typical to approximate the gradients of the objective as
{%
\setlength{\abovedisplayskip}{1ex} \setlength{\abovedisplayshortskip}{1ex}
\setlength{\belowdisplayskip}{1ex} \setlength{\belowdisplayshortskip}{1ex}
\begin{align}
  \nabla F\left(\vx\right) 
  = \frac{1}{n} {\textstyle\sum^n_{i=1}} \nabla f_i\left(\vx\right)
  = \Esub{\rvB \sim \pi}{ \frac{1}{b} {\textstyle\sum_{i \in \rvB}} \nabla f_i\left(\vx\right) },
  \label{eq:finite_sum}
\end{align}
}%
where \(\rvB \sim \pi\) is an index set of cardinality \(\abs{\rvB} = b\), or ``minibatch,'' formed by subsampling over the datapoint indices \(\{1, \ldots, n\}\).
More formally, we are approximating \(\nabla F\) using the (minibatch) subsampling estimator
{%
\setlength{\abovedisplayskip}{.5ex} \setlength{\abovedisplayshortskip}{.5ex}
\setlength{\belowdisplayskip}{.5ex} \setlength{\belowdisplayshortskip}{.5ex}
\begin{equation}
  \nabla f_{\rvB}\left(\vx\right)
  \triangleq
  \frac{1}{b} \sum_{i \in \rvB} \nabla f_i\left(\vx\right),
  \label{eq:minibatch_gradient}
\end{equation}
}%
where the performance of this estimator, or equivalently, of the subsampling strategy \(\pi\), can be quantified by its variance
{%
\setlength{\abovedisplayskip}{0ex} \setlength{\abovedisplayshortskip}{0ex}
\setlength{\belowdisplayskip}{1ex} \setlength{\belowdisplayshortskip}{1ex}
\begin{equation}
  \mathrm{tr}\V{ \nabla f_{\rvB}\left(\vx\right) } = \frac{1}{b_{\mathrm{eff}}} 
  \underbrace{\textstyle
  \frac{1}{n} \sum^n_{i=1} \norm{ \nabla f_i\left(\vx\right) - \nabla F\left(\vx\right) }_2^2,
  }_{\text{(unit) subsampling variance}}
  \label{eq:ess}
\end{equation}
}%
where we say \(b_{\mathrm{eff}}\) is the ``effective sample size'' of \(\pi\).
For instance, independent subsampling achieves \(b_{\mathrm{eff}} = b\), and sampling without replacement, also known as ``\(b\)-nice sampling''~\citep{gower_sgd_2019,richtarik_parallel_2016,csiba_importance_2018}, achieves \(b_{\mathrm{eff}} = \nicefrac{(n-1) b}{n - b}\) (\cref{thm:variancewithoutreplacement}).

\vspace{-0.5ex}
\subsection{Doubly Stochastic Gradients}\label{section:dsgd}
\vspace{-0.5ex}
For problems where the components are defined as intractable expectations as in \cref{eq:expectation_subcomponent}, we have to rely on an additional Monte Carlo approximation step such as
{%
\setlength{\abovedisplayskip}{1ex} \setlength{\abovedisplayshortskip}{1ex}
\setlength{\belowdisplayskip}{1ex} \setlength{\belowdisplayshortskip}{1ex}
\begin{alignat}{2}
  \nabla F\left(\vx\right) 
  &= \frac{1}{n} \sum^n_{i=1} \nabla f_i\left(\vx\right) 
  = \Esub{\rvB \sim \pi}{ \frac{1}{b} \sum_{i \in \rvB} \Esub{\rvveta \sim \varphi}{ \nabla f_i\left(\vx; \rvveta\right) } }
  \nonumber
  \\
  &= \Esub{\rvB \sim \pi,\;\; \rvveta_{j} \sim \varphi}{
    \frac{1}{m b} \sum_{i \in \rvB} \sum_{j=1}^m \nabla f_i\left(\vx; \rvveta_{j}\right)
  },
  \label{eq:doubly_stochastic_full}
\end{alignat}
}%
where \(\rvveta_{j} \sim \varphi\) are \(m\) independently and identically distributed (\textit{i.i.d.}) Monte Carlo samples from \(\varphi\).

\vspace{-1.5ex}
\paragraph{Doubly Stochastic Gradient}
Consider an unbiased estimator of the component gradient \(\nabla f_i\) such that
{%
\setlength{\abovedisplayskip}{1ex} \setlength{\abovedisplayshortskip}{1ex}
\setlength{\belowdisplayskip}{1ex} \setlength{\belowdisplayshortskip}{1ex}
\begin{equation}
  \mathbb{E} \rvvg_{i}\left(\vx\right) 
  = 
  \mathbb{E}_{\rvveta \sim \varphi} \vg_i\left(\vx; \rvveta\right) 
  = \nabla f_i\left(\vx\right),
  \label{eq:monte_carlo_gradient}
\end{equation}
}%
where \(\vg_i\left(\vx; \rvveta\right)\) is the measurable integrand.
Using these, we can estimate \(\nabla F\) through the \textit{doubly stochastic} gradient estimator
{
\setlength{\abovedisplayskip}{1ex} \setlength{\abovedisplayshortskip}{1ex}
\setlength{\belowdisplayskip}{1ex} \setlength{\belowdisplayshortskip}{1ex}
\begin{equation}
  \rvvg_{\rvB}\left(\vx\right) 
  \triangleq
  \frac{1}{b} \sum_{i \in \rvB} \rvvg_{i}\left(\vx\right), 
  \label{eq:doubly_stochastic_gradient}
\end{equation}
}%
We separately define the integrand \(\vg\left(\vx; \veta\right)\) since, in practice, a variety of unbiased estimators of \(\nabla f_i\) can be obtained by appropriately defining the integrand \(\vg_i\).
For example, one can form the \(m\)-sample ``naive'' Monte Carlo estimator by setting
{%
\setlength{\abovedisplayskip}{.5ex} \setlength{\abovedisplayshortskip}{.5ex}
\setlength{\belowdisplayskip}{.5ex} \setlength{\belowdisplayshortskip}{.5ex}
\[
  \vg_i\left(\vx; \rvveta\right)
  =
  {\textstyle\frac{1}{m} \sum^m_{j=1}} \nabla f_i\left(\vx; \rvveta_{j}\right),
\]
}%
where \(\rvveta = [\rvveta_1, \ldots, \rvveta_m] \sim \varphi^{\otimes m}\). 

\vspace{-1.5ex}
\paragraph{Dependent Component Gradient Estimators.}
Notice that, in \cref{eq:doubly_stochastic_full}, the subcomponents in the batch share the Monte Carlo samples, which may occur in practice.
This means \(\rvvg_i\) and \(\rvvg_j\) in the same batch are dependent and, in the worst case, positively correlated, which complicates the analysis.
While it is possible to make the estimators independent by sampling \(m\) unique Monte Carlo samples for each component (\(m b\) Monte Carlo samples in total) as highlighted by~\citet{kingma_variational_2015}, it is common to use dependent estimators for various practical reasons:
\begin{enumerate}
    \vspace{-2ex}
    \setlength\itemsep{0ex}
    
    \item \textbf{ERM with Randomized Smoothing}: 
    In the ERM context, recent works have studied the generalization benefits of randomly perturbing the model weights before computing the gradient~\citep{orvieto_explicit_2023,liu_noisy_2021a}.
    When subsampling is used, perturbing the weights independently for each datapoint is computationally inefficient.
    Therefore, the perturbation is shared across the batch, creating dependence.
    
    \item \textbf{Black-Box Variational inference}~\citep{titsias_doubly_2014,kucukelbir_automatic_2017}: Here, each component can be decomposed as 
{%
\setlength{\abovedisplayskip}{1ex} \setlength{\abovedisplayshortskip}{1ex}
\setlength{\belowdisplayskip}{1ex} \setlength{\belowdisplayshortskip}{1ex}
    \[
      f_i\left(\vx; \rvveta\right) = \ell_i\left(\vx; \rvveta\right) + r\left(\vx; \rvveta\right),
    \]
}%
    where \(\ell_i\) is the log likelihood and \(r\) is the log-density of the prior.
    By sharing \({(\rvveta_j)}_{j=1}^m\), \(r\) only needs to be evaluated \(m\) times.
    To create independent estimators, it needs to be evaluated \(m b\) times instead, but \(r\) can be expensive to compute.
    
    \item \textbf{Random feature kernel regression} with doubly SGD~\citep{dai_scalable_2014}: The features are shared across the batch \footnote{See the implementation at \url{https://github.com/zixu1986/Doubly_Stochastic_Gradients}}.
    This reduces the peak memory requirement from \(b m d_{\veta}\), where \(d_{\veta}\) is the size of the random features, to \(m d_{\veta}\).
    \vspace{-2ex}
\end{enumerate}
One of the analysis goals of this work is to characterize the effect of dependence in the context of SGD. 

\subsection{Technical Assumptions on Gradient Estimators}\label{section:conditions}
To establish convergence of SGD, contemporary analyses use the ``variance transfer'' strategy~\citep{moulines_nonasymptotic_2011,johnson_accelerating_2013,nguyen_sgd_2018,gower_sgd_2019,gower_stochastic_2021}.
That is, by assuming the gradient noise satisfies some condition resembling smoothness, it is possible to bound the gradient noise on some arbitrary point \(\vx\) by the gradient variance on the solution set \(\vx_* \in \argmin_{\vx \in \mathcal{X}} F\left(\vx\right)\).

\paragraph{ER Condition.}
In this work, we will use the \textit{expected residual} (ER) condition by \citet{gower_sgd_2021}:
\vspace{-1ex}
\begin{definition}[\textbf{Expected Residual; ER}]\label{assumption:expected_residual}
A gradient estimator \(\rvvg\) of \(F : \mathcal{X} \to \mathbb{R}\) is said to satisfy \(\mathrm{ER}\left(\mathcal{L}\right)\) if  
{%
\setlength{\abovedisplayskip}{1ex} \setlength{\abovedisplayshortskip}{1ex}
\setlength{\belowdisplayskip}{1ex} \setlength{\belowdisplayshortskip}{1ex}
\[
  \mathrm{tr}\V{ 
    \rvvg\left(\vx\right) - \rvvg\left(\vx_*\right) 
  }
  \leq
  2 \mathcal{L} \left( F(\vx) - F(\vx_*) \right),
\]
}%
for some \(0 < \mathcal{L} < \infty\) and all \(\vx\in \mathcal{X}\) and all \(\vx_* \in \argmin_{\vx \in \mathcal{X}} F\left(\vx\right)\).
\end{definition}
\vspace{-1ex}
When \(f\) is convex, a weaker form can be used:
We will also consider the \textit{convex} variant of the ER condition that uses the Bregman divergence defined as
{%
\setlength{\abovedisplayskip}{1.ex} \setlength{\abovedisplayshortskip}{1.ex}
\setlength{\belowdisplayskip}{1.ex} \setlength{\belowdisplayshortskip}{1.ex}
\[
{\textstyle
\mathrm{D}_{\phi}\left(\vy, \vx\right)
\triangleq 
\phi\left(\vy\right) - \phi\left(\vx\right) - \inner{\nabla \phi\left(\vx\right)}{\vy - \vx},
}%
\]
}%
\(\forall(\vx, \vy) \in \mathcal{X}^2\), where \(\phi : \mathcal{X} \to \mathbb{R}\) is a convex function.

 \vspace{-1.ex}
\paragraph{Why the ER condition?}
A way to think about the ER condition is that it corresponds to the ``variance form'' equivalent of the expected smoothness (ES) condition by \citet{gower_stochastic_2021} defined as
{%
\setlength{\abovedisplayskip}{1.5ex} \setlength{\abovedisplayshortskip}{1.5ex}
\setlength{\belowdisplayskip}{1.5ex} \setlength{\belowdisplayshortskip}{1.5ex}
\begin{align*}
  &\mathbb{E}\norm{\rvvg\left(\vx\right) - \rvvg\left(\vx_*\right) }_2^2 \leq 2 \mathcal{L} \left( F\left(\vx\right) - F\left(\vx_*\right) \right),
  &\quad
  \text{(ES)}
\end{align*}
}%
but is slightly weaker, as shown by \citet{gower_sgd_2021}.
The main advantage of the ER condition is that, due to the properties of the variance, it composes more easily:
\vspace{-1.5ex}
\begin{proposition}
    Let \(\rvvg\) satisfy \(\mathrm{ER}\left(\mathcal{L}\right)\). Then, the \(m\)-sample i.i.d. average of \(\rvvg\) satisfy \(\mathrm{ER}\left(\nicefrac{\mathcal{L}}{m}\right)\).
\end{proposition}
\vspace{-1.5ex}

\paragraph{BV Condition.}
From the ER property, the gradient variance on any point \(\vx \in \mathcal{X}\) can be bounded by the variance on the solution set as long as the following holds:
\begin{definition}[\textbf{Bounded Gradient Variance}]\label{assumption:bounded_variance}
A gradient estimator \(\rvvg\) of \(F : \mathcal{X} \to \mathbb{R}\) satisfies \(\mathrm{BV}\left(\sigma^2\right)\) if
{%
\setlength{\abovedisplayskip}{1ex} \setlength{\abovedisplayshortskip}{1ex}
\setlength{\belowdisplayskip}{1ex} \setlength{\belowdisplayshortskip}{1ex}
\[
  \mathrm{tr} \V{
    \rvvg\left(\vx_*\right) 
  }
  \leq
  \sigma^2
\]
}%
for some \(\sigma^2 < \infty\) and all \(\vx_* \in \argmax_{\vx \in \mathcal{X}} F\left(\vx\right)\).
\end{definition}
%

\subsection{Convergence Guarantees for SGD}\label{section:sgd_convergence}
\paragraph{Sufficiency of ER and BV.}
From the ER and BV conditions, other popular conditions such as ES~\citep{gower_stochastic_2021} and ABC~\citep{khaled_better_2023} can be established with minimal additional assumptions.
As a result, we retrieve the previous convergence results on SGD established for various objective function classes:
\begin{itemize}
\vspace{-2ex}
    \setlength\itemsep{0ex}
    \item[\ding{228}] strongly convex~\citep{gower_sgd_2019},
    \item[\ding{228}] quasar convex (+PL)~\citep{gower_sgd_2021}, 
    \item[\ding{228}] smooth (+PL)~\citep{khaled_better_2023}.
\vspace{-2ex}
\end{itemize}
(Note: quasar convexity is strictly weaker than convexity~\citealp{guminov_accelerated_2023}; PL: Polyak-{\L}ojasiewicz.)
(See also the comprehensive treatment by \citealp{garrigos_handbook_2023}.)
Therefore, ER and BV are sufficient conditions for SGD to converge on problem classes typically considered in SGD convergence analysis.

In this work, we will specifically focus on smooth and strongly convex objectives:
\begin{assumption}\label{assumption:objective}
    There exists some \(\mu, L\) satisfying \(0 < \mu \leq L < \infty\) suc that the objective function \(F : \mathcal{X} \to \mathbb{R}\) is \(\mu\)-strongly convex and \(L\)-smooth as
   {%
\setlength{\abovedisplayskip}{0.5ex} \setlength{\abovedisplayshortskip}{0.5ex}
\setlength{\belowdisplayskip}{0.5ex} \setlength{\belowdisplayshortskip}{0.5ex}
   \begin{align*}
      F\left(\vy\right) - F\left(\vx\right) &\geq  \inner{\nabla F\left(\vx\right)}{\vy - \vx} + \frac{\mu}{2} \norm{\vx - \vy}_2^2 
      \\
      F\left(\vy\right) - F\left(\vx\right)  &\leq  \inner{\nabla F\left(\vx\right)}{\vy - \vx}  + \frac{L}{2} \norm{\vx - \vy}_2^2
   \end{align*}
   }%
   hold for all \((\vx, \vy) \in \mathcal{X}^2\).
\end{assumption}
Also, we will occasionally assume that \(F\) is comprised of a finite sum of convex and smooth components:
\begin{assumption}\label{assumption:components}
   The objective function \(F : \mathcal{X} \to \mathbb{R}\) is a finite sum as \(F = \frac{1}{n} \left(f_1 + \ldots + f_n\right)\), where each component is \(L_i\)-smooth and convex such that
   {%
\setlength{\abovedisplayskip}{1ex} \setlength{\abovedisplayshortskip}{1ex}
\setlength{\belowdisplayskip}{1ex} \setlength{\belowdisplayshortskip}{1ex}
   \[
      \norm{\nabla f_i\left(\vx\right) - \nabla f_i\left(\vy\right)}_2^2 \leq 2 L_i 
 \, \mathrm{D}_{f_i}\left(\vx, \vy\right)
   \]
   }%
   holds for all \((\vx, \vy) \in \mathcal{X}^2\).
\end{assumption}
Note that \cref{assumption:components} alone already implies that \(F\) is convex and \(L_{\mathrm{max}}\)-smooth with \(L_{\mathrm{max}} = \max\left\{L_1, \ldots, L_n\right\}\).

\vspace{-1.5ex}
\paragraph{Why focus on strongly convex functions?}
We focus on strongly convex objectives as the effect of stochasticity is the most detrimental: in the deterministic setting, one only needs \(\mathcal{O}\left(\log\left(\nicefrac{1}{\epsilon}\right)\right)\) iterations to achieve an \(\epsilon\)-accurate solution.
But with SGD, one actually needs \(\mathcal{O}\left(\nicefrac{1}{\epsilon}\right)\) iterations due to noise.
As such, we can observe a clear contrast between the effect of optimization and noise in this setting.

With that said, for completeness, we provide full proof of convergence on strongly convex-smooth objectives:

\input{thm_strongly_convex_sgd_convergence}

\input{thm_strongly_convex_sgd_complexity}

Note that our complexity guarantee is only \(\mathcal{O}(\nicefrac{1}{\epsilon} \log\left(\nicefrac{1}{\epsilon}\right))\) due to the use of a fixed stepsize.
It is also possible to establish a \(\mathcal{O}(\nicefrac{1}{\epsilon})\) guarantee using decreasing stepsize schedules proposed by \citet{gower_sgd_2019,stich_unified_2019}.
In practice, these schedules are rarely used, and the resulting complexity guarantees are less clear than with fixed stepsizes.
Therefore, we will stay on fixed stepsizes.

%% file: thm_strongly_convex_sgd_convergence.tex
\begin{theoremEnd}[all end, category=expectedresidualgradientvariancebound]{lemma}\label{thm:expected_residual_gradient_variance_bound}
  Let \(F:\mathcal{X} \to \mathbb{R}\) be \(L\)-smooth function.
  Then, the expected squared norm of a gradient estimator \(\rvvg\) satisfying both \(\mathrm{ER}\left(\mathcal{L}\right)\) and \(\mathrm{BV}\left(\sigma^2\right)\) is bounded as
  \[
     \mathbb{E} \norm{\rvvg\left(\vx\right)}_2^2
     \leq
     4 \left(\mathcal{L} + L\right) \left(F\left(\vx\right) - F\left(\vx_*\right)\right) + 2 \sigma^2,
  \]
  for any \(\vx \in \mathcal{X}\) and \(\vx_* \in \argmax_{\vx \in \mathcal{X}} F\left(\vx\right)\).
\end{theoremEnd}
\vspace{-1.5ex}
\begin{proofEnd}
    The proof is a minor modification of Lemma 2.4 by \citet{gower_sgd_2019} and Lemma 3.2 by \citet{gower_sgd_2021}.

    By applying the bound \({(a + b)}^2 \leq 2 a^2 + 2 b^2\),  we can ``transfer'' the variance on \(\vx\) to the variance of \(\vx_*\).
    That is,
    \begin{align*}
      \mathbb{E} \norm{\rvvg\left(\vx\right)}_2^2
      &=
      \mathbb{E} \norm{\rvvg\left(\vx\right) - \rvvg\left(\vx_*\right) + \rvvg\left(\vx_*\right)}_2^2
      \\
      &\;\leq
      2 \, 
      \underbrace{
      \mathbb{E} \norm{\rvvg\left(\vx\right) - \rvvg\left(\vx_*\right)}_2^2
      }_{V_1}
      + 
      2 \, \underbrace{ \mathbb{E} \norm{\rvvg\left(\vx_*\right)}_2^2  }_{V_2}
    \end{align*}

    The key is to bound \(V_1\).
    It is typical to do this using expected-smoothness-type assumptions such as the ER assumption.
    That is,
    \begin{align*}
      V_1
      &=
      \mathbb{E} \norm{\rvvg\left(\vx\right) - \rvvg\left(\vx_*\right)}_2^2
      \\
      &=
      \mathrm{tr}\V{ \rvvg\left(\vx\right) - \rvvg\left(\vx_*\right)}
      +
      \left( \nabla F \left(\vx\right) - \nabla F\left(\vx_*\right) \right),
\shortintertext{from the \(L\)-smoothness of \(F\),}
      &\leq
      \mathrm{tr}\V{ \rvvg\left(\vx\right) - \rvvg\left(\vx_*\right)}
      +
      2 L \left(F\left(\vx\right) - F\left(\vx_*\right)\right),
\shortintertext{and the ER condition,}
      &\leq
      2 \mathcal{L} \left( F\left(\vx\right) - F\left(\vx_*\right) \right)
      +
      2 L \left(F\left(\vx\right) - F\left(\vx_*\right)\right)
      \\
      &=
      2 \left(L + \mathcal{L}\right)
      \big( 
        F\left(\vx\right) - F\left(\vx_*\right)
      \big).
    \end{align*}
    Finally, \(V_2\) immediately follows from the BV condition as
    \[
      V_2 = \mathbb{E} \norm{\rvvg\left(\vx_*\right)}_2^2 \leq \sigma^2.
    \]
\end{proofEnd}

\begin{theoremEnd}[all end, category=stronglyconvexsgdconvergence]{lemma}\label{thm:strongly_convex_sgd_convergence}
Let the objective function \(F\) satisfy \cref{assumption:objective} and the gradient estimator \(\rvvg\) be unbiased and satisfy both \(\mathrm{ER}\left(\mathcal{L}\right)\) and \(\mathrm{BV}\left(\sigma^2\right)\).
Then, the last iterate of SGD guarantees
{%
\setlength{\abovedisplayskip}{1ex} \setlength{\abovedisplayshortskip}{1ex}
\setlength{\belowdisplayskip}{1ex} \setlength{\belowdisplayshortskip}{1ex}
\[
  \E{ 
    \norm{\vx_T - \vx^*}_2^2
  }
  \leq
  {\left(1 - \mu \gamma\right)}^{T} \norm{\vx_0 - \vx_*}_2^2
  +
  \frac{2 \sigma^2}{\mu} \gamma
\]
}%
where \(\vx_* = \argmin_{\vx \in \mathcal{X}} F\left(\vx\right)\) is the global optimum.
\end{theoremEnd}
\begin{proofEnd}
    Firstly, we have
    \begin{align*}
        &\norm{\vx_{t+1} - \vx_*}_2^2
        \\
        &\;=
        \norm{\Pi_{\mathcal{X}}\left( \vx_{t} - \gamma \rvvg\left(\vx_t\right) \right) - \Pi\left(\vx_*\right)}_2^2,
\shortintertext{and since the projection onto a convex set under a Euclidean metric is non-expansive,}
        &\;\leq
        \norm{\vx_{t} - \gamma \rvvg\left(\vx_t\right) - \vx_*}_2^2
        \\
        &\;=
        \norm{\vx_{t} - \vx_*}_2^2
        - 2 \gamma \inner{\rvvg\left(\vx_t\right) }{\vx_{t} - \vx_*}
        + \gamma^2 \norm{\rvvg\left(\vx_t\right)}_2^2.
    \end{align*}
    Denoting the \(\sigma\)-algebra formed by the randomness and the iterates up to the \(t\)th iteration as \(\mathcal{F}_t\) such that \({(\mathcal{F}_t)}_{t \geq 1}\) forms a filtration, the conditional expectation is 
    \begin{align*}
        &\E{ \norm{\vx_{t+1} - \vx_*}_2^2 \mid \mathcal{F}_t }
        \\
        &\;=
        \norm{\vx_{t} - \vx_*}_2^2
        - 2 \gamma \inner{\E{ \rvvg\left(\vx_t\right) \mid \mathcal{F}_t }}{\vx_{t} - \vx_*}
        \\
        &\qquad+ \gamma^2 \E{\norm{\rvvg\left(\vx_t\right)}_2^2 \mid \mathcal{F}_t }.
        \\
        &\;=
        \norm{\vx_{t} - \vx_*}_2^2
        - 2 \gamma \inner{\nabla F \left(\vx_t\right)}{\vx_{t} - \vx_*}
        \\
        &\qquad+ \gamma^2 \E{\norm{\rvvg\left(\vx_t\right)}_2^2 \mid \mathcal{F}_t },
\shortintertext{applying the \(\mu\)-strong convexity of \(F\),}
        &\;\leq
        \norm{\vx_{t} - \vx_*}_2^2
        - 
        2 \gamma \left(F\left(\vx_t\right) - F\left(\vx_*\right) + \frac{\mu}{2} \norm{\vx_t - \vx_*}_2^2\right)
        \\
        &\qquad+ \gamma^2 \E{\norm{\rvvg\left(\vx_t\right)}_2^2 \mid \mathcal{F}_t }
        \\
        &\;=
        \left(1 - \gamma \mu\right) \norm{\vx_{t} - \vx_*}_2^2
        - 
        2 \gamma \left(F\left(\vx_t\right) - F\left(\vx_*\right)\right)
        \\
        &\qquad+ \gamma^2 \E{\norm{\rvvg\left(\vx_t\right)}_2^2 \mid \mathcal{F}_t }
    \end{align*}

    From \cref{thm:expected_residual_gradient_variance_bound}, we have
    \begin{align*}
        \E{\norm{\rvvg\left(\vx_t\right)}_2^2 \mid \mathcal{F}_t }
        \leq
        \left(4 \left(\mathcal{L}+L\right) \left(F\left(\vx_t\right) - F\left(\vx_*\right) \right) + 2 \sigma^2 \right).
    \end{align*}
    Therefore, 
    \begin{align*}
        &\E{ \norm{\vx_{t+1} - \vx_*}_2^2 \mid \mathcal{F}_t }
        \\
        &\leq
        \left(1 - \gamma \mu\right) \norm{\vx_{t} - \vx_*}_2^2
        - 
        2 \gamma \left(F\left(\vx_t\right) - F\left(\vx_*\right)\right)
        \\
        &\quad+ \gamma^2 \left(4 \left(\mathcal{L}+L\right) \left(F\left(\vx_t\right) - F\left(\vx_*\right) \right) + 2 \sigma^2 \right)
        \\
        &=
        \left(1 - \gamma \mu\right) \norm{\vx_{t} - \vx_*}_2^2
        \\
        &\quad- 
        2 \gamma \left(1 - 2 \gamma \left(\mathcal{L} + L\right)\right) \left(F\left(\vx_t\right) - F\left(\vx_*\right)\right)
        + 2 \gamma^2 \sigma^2,
\shortintertext{and with a small-enough stepsize satisfying \(\gamma < \frac{1}{2 \left(\mathcal{L} + L\right)}\), we can guarantee a partial contraction as}
        &\leq
        \left(1 - \gamma \mu\right) \norm{\vx_{t} - \vx_*}_2^2 + 2 \gamma^2 \sigma^2.
    \end{align*}
    Note that the coefficient \(1 - \gamma \mu\) is guaranteed to be strictly smaller than 1 since \(\mu \leq L\), which means that we indeed have a partial contraction.
    
    Now, taking full expectation, we have
    \[
      \mathbb{E}\norm{\vx_{t+1} - \vx_*}_2^2
      \leq
      \left(1 - \gamma \mu\right) \mathbb{E}\norm{\vx_{t} - \vx_*}_2^2 + 2 \gamma^2 \sigma^2.
    \]
    Unrolling the recursion from \(0\) to \(T-1\), we have
    \begin{align*}
      \mathbb{E}\norm{\vx_{T} - \vx_*}_2^2
      &\leq
      {\left(1 - \gamma \mu\right)}^T \mathbb{E}\norm{\vx_{0} - \vx_*}_2^2 
      \\
      &\qquad+ 2 \gamma^2 \sigma^2 \sum^{T-1}_{t=0} {\left(1 - \gamma \mu\right)}^{t}.
      \\
      &\leq
      {\left(1 - \gamma \mu\right)}^T \mathbb{E}\norm{\vx_{0} - \vx_*}_2^2 
      + \frac{2 \sigma^2}{\mu} \gamma.
    \end{align*}
    where the last inequality follows from the asymptotic bound on geometric sums.
\end{proofEnd}

%% file: thm_strongly_convex_sgd_complexity.tex
\begin{theoremEnd}[category=stronglyconvexsgdcomplexity]{lemma}\label{thm:strongly_convex_sgd_complexity}
Let the objective \(F\) satisfy \cref{assumption:objective} and the gradient estimator \(\rvvg\) satisfy \(\mathrm{ER}\left(\mathcal{L}\right)\) and \(\mathrm{BV}\left(\sigma^2\right)\).
Then, the last iterate of SGD is \(\epsilon\)-close to the global optimum \(\vx_* = \argmin_{\vx \in \mathcal{X}} F\left(\vx\right)\) such that \(\mathbb{E}\norm{\vx_T - \vx_*}_2^2 \leq \epsilon\) if
{%
\setlength{\abovedisplayskip}{.5ex} \setlength{\abovedisplayshortskip}{.5ex}
\setlength{\belowdisplayskip}{.5ex} \setlength{\belowdisplayshortskip}{.5ex}
\[
    T \geq 2 \max\left(\frac{\sigma^2}{\mu^2} \frac{1}{\epsilon}, \frac{\mathcal{L} + L}{\mu}\right) \log\left(2 \norm{\vx_0 - \vx_*}_2^2 \frac{1}{\epsilon} \right)
\]
}%
and the fixed stepsize
\[
    \gamma = \min\left( \frac{\epsilon \mu}{2 \sigma^2}, \frac{1}{2 \left( \mathcal{L} + L \right)} \right).
\]
\end{theoremEnd}
\vspace{-1.5ex}
\begin{proofEnd}
We can apply \cref{thm:geometric_complexity} to the result of \cref{thm:strongly_convex_sgd_convergence} with the constants
\[
  r_0 = \norm{\vx_0 - \vx_*}_2^2,
  \quad
  B = \frac{2 \sigma^2}{\mu},
  \;\text{and}\;\;
  C = 2 \left( \mathcal{L} + L \right).
\]
Then, we can guarantee an \(\epsilon\)-accurate solution with the stepsize 
\[
  \gamma = \min\left( \frac{\epsilon \mu}{2 \sigma^2}, \frac{1}{2 \left( \mathcal{L} + L \right)} \right)
\]
and a number of iterations of at least
\begin{align*}
  T 
  &\geq 
  \frac{1}{\mu} \max\left( \frac{2 \sigma^2}{\mu}, 2 \left(\mathcal{L} + L\right) \right) \log\left(2 \norm{\vx_0 - \vx_*}_2^2 \frac{1}{\epsilon} \right)
  \\
  &=
  2 \max\left( \frac{\sigma^2}{\mu^2}, \frac{\mathcal{L} + L}{\mu} \right) \log\left(2 \norm{\vx_0 - \vx_*}_2^2 \frac{1}{\epsilon} \right).
\end{align*}

\end{proofEnd}

%% file: section_main.tex
\subsection{Doubly Stochastic Gradients}
\begin{wraptable}{r}{0.35\columnwidth}
    \setlength{\tabcolsep}{.3em}
    \vspace{-14ex}
    \centering
    \caption{Rosetta Stone}\label{table:rosetta}
    \vspace{1ex}
    \hspace{-1.5em}
    \begin{tabular}{ccc}
      \S 3.1.1 & & \S 3.1.2 \\ \midrule
      \(\rvvx_i \) & \(\leftrightarrow\) & \(\rvvg_i\) \\
      \(\rvvx_{\rvB}\) & \(\leftrightarrow\) & \(\rvvg_{\rvB}\) \\
      \(\bar{\vx}_i\) & \(\leftrightarrow\) & \(\nabla f_i\) \\
      \(\bar{\vx}\) & \(\leftrightarrow\) & \(\nabla F\) \\
    \end{tabular}
    \vspace{-3ex}
\end{wraptable}
First, while taming notational complexity, we will prove a general result that holds for combining unbiased but potentially correlated estimators through subsampling.
All of the later results on SGD will fall out as special cases following the correspondence in \cref{table:rosetta}.

\subsubsection{General Variance Bound}
\vspace{-.5ex}
\paragraph{Theoretical Setup.}
Consider the problem of estimating the population mean \(\bar{\vx} = \frac{1}{n} \sum_{i=1}^n \bar{\vx}_i\) with a collection of RVs \(\rvvx_{1}, \ldots, \rvvx_{n}\), each an unbiased estimator of the component \(\bar{\vx}_i = \mathbb{E}\rvvx_{i}\).
Then, any subsampled ensemble
{
\setlength{\abovedisplayskip}{0.5ex} \setlength{\abovedisplayshortskip}{0.5ex}
\setlength{\belowdisplayskip}{0.5ex} \setlength{\belowdisplayshortskip}{0.5ex}
\[
    \rvvx_{\rvB} \triangleq \frac{1}{b} \sum_{i \in \rvB} \rvvx_i
    \quad\text{with}\quad
    \rvB \sim \pi,
\]
}%
where \(\pi\) is an unbiased subsampling strategy with an effective sample size of \(b_{\mathrm{eff}}\), is also an unbiased estimator of \(\bar{\vx}\).
The goal is to analyze how the variance of the component estimators \(\mathrm{tr}\mathbb{V}\rvvx_i\) for \(i = 1, \ldots, n\) and the variance of \(\pi\) affect the variance of \(\rvvx_{\rvB}\).
The following condition characterizes the correlation between the component estimators:
\vspace{-2ex}
\begin{assumption}
\label{assumption:correlation}
    The component estimators \(\rvvx_1, \ldots, \rvvx_n\) have finite variance \(\mathrm{tr}\mathbb{V}\rvvx_i < \infty\) for all \(i = 1, \ldots, n\) and, there exists some \(\rho \in [0, 1]\) for all \(i \neq j\) such that
    {%
    \setlength{\abovedisplayskip}{1.ex} \setlength{\abovedisplayshortskip}{1.ex}
    \setlength{\belowdisplayskip}{1.ex} \setlength{\belowdisplayshortskip}{1.ex}
    \[
        \mathrm{tr}\,
        \Cov{\rvvx_i, \rvvx_j}
        \leq
        \rho \, \sqrt{\mathrm{tr}\mathbb{V}\rvvx_i} \sqrt{\mathrm{tr}\mathbb{V}\rvvx_j}.
    \]
    }%
\end{assumption}
\begin{remark}\label{remark:rho1}
    \cref{assumption:correlation} always holds with \(\rho = 1\) as a basic consequence of the Cauchy-Schwarz inequality.
\end{remark}
\vspace{.3ex}
\begin{remark}\label{remark:independence_rho0}
    For a collection of mutually independent estimators \(\rvvx_1, \ldots, \rvvx_n\) such that \(\rvvx_i \indep \rvvx_j\) for all \(i \neq j\), \cref{assumption:correlation} holds with \(\rho = 0\).
\end{remark}
\vspace{.3ex}
\begin{remark}
    The equality in \cref{assumption:correlation} holds with \(\rho = 0\) for independent estimators, while it holds with \(\rho = 1\) when they are perfectly positively correlated such that, for all \(i \neq j\), there exists some constant \(\alpha_{ij} \geq 0\) such that \(\rvvx_i = \alpha_{i,j} \rvvx_j\)
\end{remark}
\vspace{-1ex}


\input{thm_auxiliary}

\input{thm_doubly_stochastic_variance}
\vspace{-2ex}
In \cref{thm:doubly_stochastic_variance}, \(V_{\mathrm{com}}\) is the contribution of the variance of the component estimators, while \(V_{\mathrm{cor}}\) is the contribution of the correlation between component estimators
, and \(V_{\mathrm{sub}}\) is the subsampling variance.

\vspace{-1ex}
\paragraph{Monte Carlo with Subampling Without Replacement.}
\cref{thm:doubly_stochastic_variance} is very general: it encompasses both the correlated and uncorrelated cases and matches the constants of all of the important special cases.
We will demonstrate this in the following corollary along with variance reduction by Monte Carlo averaging of \(m\) i.i.d. samples.
That is, we subsample over \(\rvvx_1^m, \ldots, \rvvx_n^m\), where each estimator is an \(m\)-sample Monte Carlo estimator:
{%
\setlength{\abovedisplayskip}{1.ex} \setlength{\abovedisplayshortskip}{1.ex}
\setlength{\belowdisplayskip}{1.ex} \setlength{\belowdisplayshortskip}{1.ex}
\[
    \rvvx_i^{m} \triangleq {\textstyle\frac{1}{m} \sum^m_{j=1}} \rvvx_i^{(j)},
\]
}%
where \(\rvvx_i^{(1)}, \ldots, \rvvx_i^{(m)}\) are i.i.d replications with mean \(\bar{\vx}_i = \mathbb{E}\rvvx^{(j)}_i\).
Then, the variance of the doubly stochastic estimator \(\rvvx_{\rvB}\) of the mean \(\bar{\vx} = \frac{1}{n} \sum^n_{i=1} \bar{\vx}_i\) defined as 
{
\setlength{\abovedisplayskip}{1ex} \setlength{\abovedisplayshortskip}{1ex}
\setlength{\belowdisplayskip}{1ex} \setlength{\belowdisplayshortskip}{1ex}
\[
    \rvvx_{\rvB}^m \triangleq {\textstyle\frac{1}{b} \sum_{i \in \rvB}} \rvvx_i^m
    \quad\text{with}\quad
    \rvB \sim \pi,
\]
}%
can be bounded as follows:
\begin{corollary}\label{thm:general_variance_corollary}
    For each \(j=1, \ldots, m\), let \(\rvvx_1^{(j)}, \ldots, \rvvx_n^{(j)}\) satisfy \cref{assumption:correlation}.
    Then, the variance of the doubly stochastic estimator \(\rvvx_{\rvB}^m\) of the mean \(\bar{\vx} = \frac{1}{n} \sum^n_{i=1} \bar{\vx}_i\), where \(\pi\) is \(b\)-minibatch sampling without replacement, satisfy the following corollaries:
\begin{enumerate}[label=\textbf{(\roman*)}]
\vspace{-2ex}
    \item \(\rho=1\) and \(1 < b < n\): \label{item:doubly_stochastic_without_replacement_general}
{
\setlength{\abovedisplayskip}{1ex} \setlength{\abovedisplayshortskip}{1ex}
\setlength{\belowdisplayskip}{1ex} \setlength{\belowdisplayshortskip}{1ex}
\begin{align*}
\hspace{-2em}
    \mathrm{tr}\V{\rvvx_{\rvB}^m}
    &\leq
    \frac{n-b}{(n-1) m b} {\textstyle\left(\frac{1}{n} \sum^n_{i=1}  \sigma^2_{i}\right)
    }
    \\
    &\;\;+
    \frac{n \left(b-1\right)}{(n-1) m b} {\textstyle\left(\frac{1}{n} \sum^n_{i=1}  \sigma_{i}\right)}^2
    +
    \frac{n-b}{(n-1) b} \tau^2
\end{align*}
}

    \item \(\rho = 1\) and \(b = 1\):
{
\setlength{\abovedisplayskip}{0ex} \setlength{\abovedisplayshortskip}{0ex}
\setlength{\belowdisplayskip}{1ex} \setlength{\belowdisplayshortskip}{1ex}
\begin{align*}
    \mathrm{tr}\V{\rvvx_{\rvB}^m}
    &\leq
    \frac{1}{m} {\textstyle\left(\frac{1}{n} \sum^n_{i=1} \sigma_i^2 \right)} 
    +
    \tau^2
\end{align*}
}%

    \item \(\rho=1\) and \(b = n\):
{
\setlength{\abovedisplayskip}{0ex} \setlength{\abovedisplayshortskip}{0ex}
\setlength{\belowdisplayskip}{1ex} \setlength{\belowdisplayshortskip}{1ex}
\begin{align*}
    \mathrm{tr}\V{\rvvx_{\rvB}^m}
    \leq
    \frac{1}{m} {\textstyle\left(\frac{1}{n} \sum^n_{i=1} \sigma_i \right)}^2
\end{align*}
}%

    \item \(\sigma_{i} = 0\) for all \(i = 1, \ldots, n\):
{
\setlength{\abovedisplayskip}{1ex} \setlength{\abovedisplayshortskip}{1ex}
\setlength{\belowdisplayskip}{0ex} \setlength{\belowdisplayshortskip}{0ex}
\begin{align*}
    \mathrm{tr}\V{\rvvx_{\rvB}^m}
    \leq
    \frac{n-b}{(n-1) b} 
    \tau^2,
\end{align*}%
}%
\vspace{-3ex}
    \item \(\rho = 0\):
{
\setlength{\abovedisplayskip}{0ex} \setlength{\abovedisplayshortskip}{0ex}
\setlength{\belowdisplayskip}{0ex} \setlength{\belowdisplayshortskip}{0ex}
\begin{align*}
    \mathrm{tr}\V{\rvvx_{\rvB}^m}
    \leq
    \frac{1}{m b} {\textstyle\left(\frac{1}{n} \sum^n_{i=1} \sigma_i^2 \right)} 
    +
    \frac{n - b}{(n-1) b}
    \tau^2
\end{align*}%
}%
\end{enumerate}%
where, for all \(i = 1, \ldots, n\) and any \(j = 1, \ldots, m\), 
{%
\setlength{\abovedisplayskip}{0.5ex} \setlength{\abovedisplayshortskip}{0.5ex}
\setlength{\belowdisplayskip}{1ex} \setlength{\belowdisplayshortskip}{1ex}
\begin{alignat*}{2}
     \sigma_i^2 &= \mathrm{tr}\mathbb{V}\,\rvvx_i^{(j)}
     &&\;\;\text{is invidual variance and} \\
    \tau^2  &= {\textstyle \frac{1}{n}\sum^{n}_{i=1}} \norm{ \bar{\vx}_i - \bar{\vx} }_2^2
     &&\;\;\text{is the subsampling variance.}
\end{alignat*}
}%
\end{corollary}
\vspace{0.5ex}
\begin{remark}[\textbf{For dependent estimators, increasing \(b\) also reduces component variance.}]\label{eq:monte_carlo_recommendation}
    Notice that, for case of \(\rho=1\),  \cref{thm:general_variance_corollary} (i), the term with \(\frac{1}{n} \sum^n_{i=1} \sigma_{i}^2\) is reduced in a rate of \(\mathcal{O}\left(\nicefrac{1}{m b}\right)\).
    This means reducing the subsampling noise by increasing \(b\) also reduces the noise of estimating each component.
    Furthermore, the first term dominates the second term as
{%
\setlength{\abovedisplayskip}{0ex} \setlength{\abovedisplayshortskip}{0ex}
\setlength{\belowdisplayskip}{1ex} \setlength{\belowdisplayshortskip}{1ex}
    \[
    {\textstyle
      {\left( \frac{1}{n} \sum^n_{i=1} \sigma_{i} \right)}^2
      \leq
      \frac{1}{n} \sum^n_{i=1} \sigma_{i}^2,
    }
    \]
}%
    as stated by Jensen's inequality.
    Therefore, despite correlations, increasing \(b\) will have a more significant effect since it reduces both dominant terms \(\frac{1}{n}\sum_{i=1}^n\sigma_{i}^2\) and \(\tau^2\).
\end{remark}
\vspace{0.2ex}
\begin{remark}
    When independent estimators are used, \cref{thm:general_variance_corollary} (v) shows that increasing \(b\) reduces the full variance in a \(\mathcal{O}(1/b)\) rate, but increasing \(m\) does not.
\end{remark}
\vspace{0.2ex}
\begin{remark}
    \cref{thm:general_variance_corollary} achieves all known endpoints in the context of SGD:
    For \(b=n\) (full batch), doubly SGD reduces to SGD with a Monte Carlo estimator, where there is no subsampling noise (no \(\tau^2\)).
    When the Monte Carlo noise is 0, then doubly SGD reduces to SGD with a subsampling estimator (no \(\sigma_{i}\)), retrieving the result of \citet{gower_sgd_2019}.
\end{remark}

\vspace{-0.5ex}
\subsubsection{Gradient Variance Conditions for SGD}
\vspace{-0.5ex}
From \cref{thm:doubly_stochastic_variance}, we can establish the ER and BV conditions (\cref{section:conditions}) of the doubly stochastic gradient estimators.
Following the notation in \cref{section:dsgd}, we will denote the doubly stochastic gradient estimator as \(\rvvg_{\rvB}\), which combines the estimators \(\rvvg_1, \ldots, \rvvg_n\) according to the subsampling strategy \(\rvB \sim \pi\), which achieves an effective sample size of \(b_{\mathrm{eff}}\).
We will also use the corresponding minibatch subsampling estimator \(\nabla f_{\rvB}\) for the analysis.


\vspace{-.5ex}
\begin{assumption}\label{assumption:estimator_correlations}
  For all \(\vx \in \mathcal{X}\), the component gradient estimators \(\rvvg_1\left(\vx\right), \ldots, \rvvg_{n}\left(\vx\right)\) satisfy \cref{assumption:correlation} with some \(\rho \in [0, 1]\).
\end{assumption}
\vspace{-.5ex}
Again, this assumption is always met with \(\rho = 1\) and holds with \(\rho = 0\) if the estimators are independent.
\vspace{-.5ex}
\begin{assumption}\label{assumption:subsampling_er}
  The subsampling estimator \(\nabla f_{\rvB}\) satisfies the \(\mathrm{ER}\left(\mathcal{L}_{\mathrm{sub}}\right)\) condition in \cref{assumption:expected_residual}.
\end{assumption}
\vspace{-.5ex}
This is a classical assumption used to analyze SGD on finite sums and is automatically satisfied by \cref{assumption:components}. (See \cref{thm:expected_residual_without_replacement} in \cref{section:proof_sgd_complexity} for a proof.)
\vspace{-.5ex}
\begin{assumption}\label{assumption:montecarlo_er}
   For all \(i = 1, \ldots, n\) and \(\vx \in \mathcal{X}\) and global minimizers \(\vx_* \in \argmin_{\vx \in \mathcal{X}} F\left(\vx_*\right)\), the component gradient estimator \(\rvvg_i\) satisfies at least one of the following variants of the ER condition:
   \begin{itemize}[leftmargin=3.5em]%
       \vspace{-2ex}
       \setlength\itemsep{1ex}
       \item[(\(\mathrm{A}^{\mathrm{CVX}}\))]
       \(
         \mathrm{tr}\V{ \rvvg_i\left(\vx\right) - \rvvg_i\left(\vy\right) }
         \leq
         2 \mathcal{L}_i \mathrm{D}_{f_i}\left(\vx, \vy\right),
       \)
       \\
       where \(f_i\) is convex.
       
       \item[(\(\mathrm{A}^{\mathrm{ITP}}\))] \label{item:montecarlo_er2}
       \(
         \mathrm{tr}\V{ \rvvg_i\left(\vx\right) - \rvvg_i\left(\vy\right) }
         \leq
         2 \mathcal{L}_i \left( f_i\left(\vx\right) - f_i\left(\vx_*\right) \right)
       \)
       \\
       where \(f_i\left(\vx\right) \geq  f_i\left(\vx_*\right)\).
       
       \item[(B)] \label{item:montecarlo_er3}
       \(
         \mathrm{tr}\V{ \rvvg_i\left(\vx\right) - \rvvg_i\left(\vy\right) }
         \leq
         2 \mathcal{L}_i \left( F\left(\vx\right) - F\left(\vx_*\right) \right)
       \).
   \end{itemize}
\end{assumption}
\vspace{-.5ex}
Each of these assumptions holds under different assumptions and problem setups.
For instance, \(\mathrm{A}^{\mathrm{CVX}}\) holds only under componentwise convexity, while \(\mathrm{A}^{\mathrm{ITP}}\) requires majorization \(f_i\left(\vx\right) \geq  f_i\left(\vx_*\right)\), which is essentially assuming ``interpolation''~\citep{vaswani_fast_2019,ma_power_2018,gower_sgd_2021} in the ERM context.
Among these, (B) is the strongest since it directly relates the individual components \(f_1, \ldots, f_n\) with the full objective \(F\).

We now state our result establishing the ER condition:

\input{thm_general_conditions}

\vspace{.5ex}
\begin{remark}
    Assuming the conditions in \cref{assumption:montecarlo_er} hold with the same value of \(\mathcal{L}_i\), the inequality
{%
\setlength{\abovedisplayskip}{0.5ex} \setlength{\abovedisplayshortskip}{.5ex}
\setlength{\belowdisplayskip}{1ex} \setlength{\belowdisplayshortskip}{1ex}
    \[
    {\textstyle
      {\left( \frac{1}{n} \sum^n_{i=1} \sqrt{\mathcal{L}_{i} }\right)}^{2}
      \leq
      \frac{1}{n} \sum^n_{i=1} \mathcal{L}_{i}
      \leq
      \mathcal{L}_{\mathrm{max}},
    }
    \]
}%
implies \(\mathcal{L}_{\rm{B}} \leq \mathcal{L}_{\rm{A}}\).
\end{remark}

Meanwhile, The BV condition follows by assuming equivalent conditions on each component estimator:
\vspace{-0.5ex}
\begin{assumption}\label{assumption:bounded_variance_both}
Variance is bounded for all \(\vx_* \in \argmin_{\vx \in \mathcal{X}} F\left(\vx\right)\) such that the following hold:
\vspace{-0.5ex}
\begin{enumerate}[leftmargin=1.3em]
\vspace{-1ex}
    \item 
    \(
    {\textstyle
        \frac{1}{n} \sum^n_{i=1} \norm{ \nabla f_i\left(\vx_*\right)}_2^2 \leq \tau^2
    }
    \) for some \(\tau^2 < \infty\) and,
    \item 
    \(\mathrm{tr}\V{\rvvg_i\left(\vx_*\right)} \leq \sigma^2_{i}\) 
    for some \(\sigma^2_{i} < \infty\), for all \(i = 1, \ldots, n\).
\end{enumerate}
\end{assumption}
\vspace{-0.5ex}
Based on these, \cref{thm:doubly_stochastic_variance} immediately yields the result:

\vspace{-0.5ex}
\input{thm_bounded_variance}
\vspace{-0.5ex}

As discussed in \cref{section:sgd_convergence}, \cref{thm:general_conditions,thm:bounded_variance} are sufficient to guarantee convergence of doubly SGD.
For completeness, let us state a specific result for \(\rho = 1\):

\input{thm_doubly_stochastic_without_replacement}

\input{thm_strongly_convex_sgd_sampling_without_replacement_complexity}


\subsection{Random Reshuffling of Stochastic Gradients}
We now move to our analysis of SGD with random reshuffling (SGD-RR).
In the doubly stochastic setting, this corresponds to reshuffling over stochastic estimators instead of gradients, which we will denote as doubly SGD-RR.
In practice, doubly SGD-RR is often observed to converge faster than doubly SGD, even when dependent estimators are used.

\vspace{-.5ex}
\subsubsection{Algorithm}
\vspace{-0.5ex}
\paragraph{Doubly SGD-RR}
The algorithm is stated as follows:
\begin{enumerate}
\setlength\itemsep{-.5ex}
\vspace{-2ex}
    \item[\ding{182}] Reshuffle and partition the gradient estimators into minibatches of size \(b\) as \(\rvP = \{\rvP_1, \ldots, \rvP_p\}\), where \(p = n/b\) is the number of partitions or minibatches.
    \item[\ding{183}] Perform gradient descent for \(i = 1, \ldots, p\) steps as
{%
\setlength{\abovedisplayskip}{1ex} \setlength{\abovedisplayshortskip}{1ex}
\setlength{\belowdisplayskip}{1ex} \setlength{\belowdisplayshortskip}{1ex}
\[\vx_{k}^{i+1} = \Pi_{\mathcal{X}}\left( \vx_k^{i} - \gamma \rvvg_{\rvP_i}\,(\vx_k^i) \right)\]
}
    \item[\ding{184}] \(k \leftarrow k + 1\) and go back to step \ding{182}.
\vspace{-2ex}
\end{enumerate}
(We assume \(n\) is an integer multiple of \(b\) for clarity.)
Here, \(i = 1, \ldots, p\) denotes the step within the epoch, \(k = 1, \ldots, K\) denotes the epoch number.

\vspace{-.5ex}
\subsubsection{Proof Sketch}
\vspace{-.5ex}
\paragraph{Why SGD-RR is Faster}
A key aspect of random reshuffling in the finite sum setting (SGD-RR) is that it uses conditionally biased gradient estimates.
Because of this, on strongly convex finite sums, \citet{mishchenko_random_2020} show that the Lyapunov function for random reshuffling is not the usual \({\lVert \vx_k^i - \vx_* \rVert}_2^2\), but some \textit{biased} Lyapunov function \({\lVert \vx_k^i - \vx_*^k \rVert}_2^2\), where the reference point is 
{%
\setlength{\abovedisplayskip}{1ex} \setlength{\abovedisplayshortskip}{1ex}
\setlength{\belowdisplayskip}{1ex} \setlength{\belowdisplayshortskip}{1ex}
\begin{equation}
{\textstyle
  \vx_*^i \triangleq \Pi_{\mathcal{X}}\left( \vx_* - \gamma \sum^{i-1}_{j=0} \nabla f_{\rvP_i}\left(\vx_*\right) \right).
}
\label{eq:lyapunov_bias}
\end{equation}
}%
Under this definition, the convergence rate of SGD is not determined by the gradient variance anymore; it is determined by the squared error of the Lyapunov reference point, \({\lVert \vx_*^i - \vx_* \rVert}_2^2\).
There are two key properties of this quantity:
\begin{itemize}
\setlength\itemsep{0ex}
\vspace{-2ex}
    \item The peak mean-squared error decreases at a rate of \(\gamma^2\) with respect to the stepsize \(\gamma\).
    \item The squared error is 0 at the following two endpoints: beginning of the epoch and at the end of the epoch.
\vspace{-2ex}
\end{itemize}
For some stepsize achieving a \(\mathcal{O}(1/T)\) rate on SGD, these two properties combined result in SGD-RR attaining a \(\mathcal{O}(1/T^2)\) rate at exactly the end of each epoch.

\begin{figure*}
    \vspace{-1ex}
\centering
    \subfloat[Low Hetero. (\(s = 1\))]{
        \input{figures/tikz_isoquad_variance_low_hetero}
        \vspace{-1ex}
    }\hspace{2em}
    \subfloat[Mid. Hetero. (\(s = 2\))]{
        \input{figures/tikz_isoquad_variance_mid_hetero}
        \vspace{-1ex}
    }\hspace{2em}
    \subfloat[High Hetero. (\(s = 4\))]{
        \input{figures/tikz_isoquad_variance_high_hetero}
        \vspace{-1ex}
    }
    \vspace{-1ex}
    \caption{
      \textbf{Trade-off between \(b\) and \(m\) on the gradient variance \(\mathrm{tr}\mathbb{V}\rvvg\left(\vx_*\right)\) under varying budgets \(m  \times b\)}.
      The problem is a finite sum of \(d = 10\), \(n=1024\) isotropic quadratics with smoothness constants sampled as \(L_i \sim \text{Inv-Gamma}(1/2, 1/2)\) and stationary points sampled as \(\vx^*_i \sim \mathcal{N}\left(\boldupright{0}_d, s^2 \boldupright{I}_d\right)\), where the gradient has additive noise of \(\rvveta \sim \mathcal{N}\left(\boldupright{0}_d, \boldupright{I}_d\right)\).
      Larger \(s\) means more heterogeneous data.
    }\label{fig:simulation}
    \vspace{-2ex}
\end{figure*}

\vspace{-1.5ex}
\paragraph{Is doubly SGD-RR as Fast as SGD-RR?}
Unfortunately, doubly SGD-RR does not achieve the same rate as SGD-RR.
Since stochastic gradients are used in addition to reshuffling, doubly SGD-RR deviates from the path that minimizes the biased Lyapunov function.
Still, doubly SGD-RR does have provable benefits.

\subsubsection{Complexity Analysis}
We provide the general complexity guarantee for doubly SGD-RR on strongly convex objectives with \(\mu\)-strongly convex components and fully correlated component estimators (\(\rho=1\)):
\input{thm_reshuffle_variance}

\input{thm_strongly_convex_reshuffling_sgd_convergence}

\input{thm_strongly_convex_reshuffling_sgd_complexity}

\vspace{1ex}
\begin{remark}
  When \(\sigma_{i} = 0\) for all \(i = 1, \ldots, n\), the anytime convergence bound \cref{thm:strongly_convex_reshuffling_sgd_convergence} in the Appendix reduces exactly to Theorem 1 of \citet{mishchenko_random_2020}.
  Therefore, \cref{thm:strongly_convex_reshuffling_sgd_convergence} is a strict generalization of SGD-RR to the doubly stochastic setting.
\end{remark}

Using \(m\)-sample Monte Carlo improves the constants as follows:
\begin{corollary}
   Let the assumptions of \cref{thm:strongly_convex_reshuffling_sgd_complexity} hold.   
   Then, for \(1 < b < n\) and \(m\)-sample Monte Carlo, the same guarantees hold with the constant
{%
\setlength{\abovedisplayskip}{1ex} \setlength{\abovedisplayshortskip}{1ex}
\setlength{\belowdisplayskip}{1ex} \setlength{\belowdisplayshortskip}{1ex}
   \begin{align*}
     C_{\rm{var}}^{\mathrm{com}} = 
     \frac{2}{m b} {\left(\frac{1}{n} \sum^n_{i=1} \frac{\sigma_{i}^2}{\mu^2} \right)} + \frac{2}{m} \, {\left( \frac{1}{n}\sum^n_{i=1} \frac{\sigma_{i}}{\mu} \right)}^2.
   \end{align*}
}
\end{corollary}

\newpage
\begin{remark}
  Compared to doubly SGD, doubly SGD-RR improves the dependence on the subsampling noise \(\tau^2\) from \(\mathcal{O}\left(1/\epsilon\right)\) to \(\mathcal{O}(1/\sqrt{\epsilon})\).
  Therefore, random reshuffling does improve the complexity of doubly SGD.
  Unfortunately, it also means that it does not achieve a better asymptotic complexity as in the finite sum setting.
  However, non-asymptotically, if the subsampling noise dominates component estimation noise, doubly SGD-RR will behave closely to an \(\mathcal{O}(1/\sqrt{\epsilon})\) (or equivalently, \(\mathcal{O}(1/T)\)) algorithm.
\end{remark}
\vspace{0.5ex}
\begin{remark}
  As was the case with independent subsampling, increasing \(b\) also reduces component estimation noise for RR-SGD.
  However, the impact on the complexity is more subtle.
  Consider that the iteration complexity is
  {%
\setlength{\abovedisplayskip}{1ex} \setlength{\abovedisplayshortskip}{1ex}
\setlength{\belowdisplayskip}{1ex} \setlength{\belowdisplayshortskip}{1ex}
  \begin{equation}
  {\textstyle
    \mathcal{O}\left( \kappa_{\sigma}^2 \left(\frac{1}{mb} + \frac{1}{m}\right) \frac{1}{\epsilon} + \kappa \, \kappa_{\tau} \,
 \frac{\sqrt{n}}{b} \frac{1}{\sqrt{\epsilon}} \right),
  }
  \label{eq:doublysgdrr_complexity}
  \end{equation}
  }%
  where \(\kappa_{\sigma} = { \max_{i=1,\ldots, n} \sigma_{i}}/{\mu}\), \(\kappa_{\tau} = \tau/{\mu}\) and  \(\kappa = {\max_{i=1,\ldots, n} L_{i}}/{\mu}\).
  The \(1/\epsilon\) term decreases the fastest with \(m\).
  Therefore, it might seem that increasing \(m\) is advantageous.
  However, the \(1/\sqrt{\epsilon}\) term has a \(\mathcal{O}\left(\sqrt{n}\right)\) dependence on the dataset size, which would be non-negligible for large datasets.
  As a result, in the large \(n\), large \(\epsilon\) regime, increasing \(b\) over \(m\) should be more effective.
\end{remark}
\vspace{0.5ex}
\begin{remark}
\cref{eq:doublysgdrr_complexity} also implies that, for dependent estimators, doubly SGD-RR achieves an asymptotic speedup of \({n}/{b}\) compared to full-batch SGD with only component estimation noise.
Assume that the sample complexity of a single estimate is \(\Theta(m b)\) (\(\Theta(m n)\) for full-batch).
Then, the sample complexity of doubly SGD-RR is \(\mathcal{O}\left(b\nicefrac{1}{\epsilon}\right)\) and \(\mathcal{O}\left(n\nicefrac{1}{\epsilon}\right)\) for full-batch SGD.
However, the \(n/b\) seed-up comes from correlations.
Therefore, for independent estimators, the asymptotic complexity of the two is equal.
\end{remark}

%% file: thm_auxiliary.tex
\begin{theoremEnd}[all end, category=variancewithoutreplacement]{lemma}\label{thm:variancewithoutreplacement}
Consider a finite population of \(n\) vector-variate random variables \(\vx_1, \ldots, \vx_n\).
Then, the variance of the average of \(b\) samples chosen without replacement is 
\[
   \mathrm{tr}\V{ \frac{1}{b} \sum_{i=1}^b \rvvx_{\rvB_i} }
   =
   \frac{n - b}{b \left(n - 1\right)} \sigma^2,
\]
where \(\rvB = \{\rvB_1, \ldots, \rvB_b\}\) is the collection of random indices of the samples and \(\sigma^2\) is the variance of independently choosing a single sample.
\end{theoremEnd}
\begin{proofEnd}
From the variance of the sum of random variables, we have
\begin{alignat}{2}
    \mathrm{tr}\V{ \sum_{i=1}^b \rvvx_{\rvB_i} }
    &=
    \sum_{i=1}^b \mathrm{tr}\V{ \rvvx_{\rvB_i} }
    +
    \sum_{i=1}^b \sum_{i \neq j}^b \Cov{\rvvx_{\rvB_i}, \rvvx_{\rvB_j}},
    \nonumber
\shortintertext{and noticing that the covariance is independent of the index in the batch,}
    &=
    b \, \mathrm{tr}\V{ \rvvx_{\rvB_i} }
    +
    b (b-1) C,
    \label{eq:sampling_wihout_replacement_eq1}
\end{alignat}
where \(C = \mathrm{Cov}\left(\rvvx_{\rvB_i}, \rvvx_{\rvB_j}\right)\).
Using the fact that the variance is 0 for \(b = n\), we can solve for \(C\) such that
\[
   C = - \frac{1}{n-1}\, \mathrm{tr}\V{ \rvvx_{\rvB_i} },
\]
which is negative, and a negative correlation is always great.
Plugging this expression to \cref{eq:sampling_wihout_replacement_eq1}, we have
\begin{align*}
  \mathrm{tr}\V{ \sum_{i=1}^b \rvvx_{\rvB_i} }
  &=
  b \, \mathrm{tr}\V{ \rvvx_{\rvB_i} }
  -
  b (b-1) \frac{1}{n-1}\, \mathrm{tr}\V{ \rvvx_{\rvB_i} },
  \\
  &=
  b \left(1 - \frac{b-1}{n-1} \right) \mathrm{tr}\V{ \rvvx_{\rvB_i} }
  \\
  &=
  b\left(\frac{n - b}{n - 1}\right) \mathrm{tr}\V{ \rvvx_{\rvB_i} }.
\end{align*}
Dividing both sides by \(b^2\) yields the result.
\end{proofEnd}

\begin{theoremEnd}[all end, category=varianceofsum]{lemma}\label{thm:varianceofsum}
Let \(\rvvx_1, \ldots, \rvvx_n\) be vector-variate random variables.
Then, the variance of the sum is upper-bounded as 
\begin{align}
  \mathrm{tr}\V{ \sum_{i=1}^n \rvvx_i }
  &\leq {\left( {\textstyle\sum^n_{i=1} \sqrt{ \mathrm{tr}\V{ \rvvx_i } } } \right)}^2  
  \label{eq:varianceofsumeq1}
  \\
  &\leq n {\textstyle\sum^n_{i=1} \mathrm{tr}\V{ \rvvx_i }}.
  \label{eq:varianceofsumeq2}
\end{align}
The equality in \cref{eq:varianceofsumeq1} holds if and only if \(\rvvx_i\) and \(\rvvx_j\) are constant multiples such that there exists some \(\alpha_{ij} \geq 0\) such that
{%
\setlength{\abovedisplayskip}{1ex} \setlength{\abovedisplayshortskip}{1ex}
\setlength{\belowdisplayskip}{1ex} \setlength{\belowdisplayshortskip}{1ex}
\[ 
    \rvvx_i
    =
    \alpha_{ij} \rvvx_j
\]
}%
for all \(i, j\).
\end{theoremEnd}
\begin{proofEnd}
The variance of a sum is
\begin{alignat*}{2}
    \mathrm{tr}\V{ \sum_{i=1}^n \rvvx_i }
    =
    \sum_{i=1}^n \sum_{j=1}^n
    \mathrm{tr}\, \Cov{\rvvx_i, \rvvx_j}.
\end{alignat*}
From the Cauchy-Schwarz inequality for expectations,
\begin{align*}
\mathrm{tr}\,
\Cov{\rvvx_i, \rvvx_j}
&=
\mathbb{E}{\left(\rvvx_i - \mathbb{E}\rvvx_i\right)}^{\top} {\left(\rvvx_j - \mathbb{E}\rvvx_j\right)}
\\
&\leq
\mathbb{E}\norm{\rvvx_i - \mathbb{E}\rvvx_i}_2 \mathbb{E}{\lVert\rvvx_j - \mathbb{E}\rvvx_j \rVert}_2
\\
&=
\sqrt{\mathrm{tr}\V{\rvvx_i}} \sqrt{\mathrm{tr}\V{\rvvx_j}}.
\end{align*}
This implies 
\begin{alignat}{2}
    \mathrm{tr}\V{ \sum_{i=1}^n \rvvx_i }
    &=
    \sum_{i=1}^n \sum_{j=1}^n \mathrm{tr}\,\Cov{\rvvx_i, \rvvx_j}
    \nonumber
    \\
    &\leq
    \sum_{i=1}^n \sum_{j=1}^n 
    \sqrt{\mathrm{tr}\V{\rvvx_i}} \sqrt{\mathrm{tr}\V{\rvvx_j}}
    \nonumber
    \\
    &=
    {\left(
      \sum_{i=1}^n  
      \sqrt{\mathrm{tr}\V{\rvvx_i}} 
    \right)}^2.
    \label{eq:sumofvariance_eq3}
\end{alignat}
The equality statement comes from the property of the Cauchy-Schwarz inequality.
Lastly, \cref{eq:varianceofsumeq2} follows from additionally applying Jensen's inequality as
\begin{alignat*}{2}
    {\left(
      \sum_{i=1}^n  
      \sqrt{\mathrm{tr}\V{\rvvx_i}} 
    \right)}^2
    &=
    n^2
    {\left(
      \frac{1}{n}
      \sum_{i=1}^n  
      \sqrt{\mathrm{tr}\V{\rvvx_i}} 
    \right)}^2
    \\
    &\leq
    n^2
    \frac{1}{n}
    \sum_{i=1}^n  
    {\left(
      \sqrt{\mathrm{tr}\V{\rvvx_i}} 
    \right)}^2
    \\
    &=
    n
    \sum_{i=1}^n \mathrm{tr}\V{\rvvx_i}.
\end{alignat*}
An equivalent proof strategy is to expand the quadratic in \cref{eq:sumofvariance_eq3} and apply the arithmetic mean-geometric mean inequality to the cross terms.
\end{proofEnd}

\begin{theoremEnd}[all end, category=geometriccomplexity]{lemma}[Lemma A.2; \citealp{garrigos_handbook_2023}]
\label{thm:geometric_complexity}
  For a recurrence relation given as
  \[
    r_{T} \leq {\left(1 - \gamma \mu\right)}^{T} r_{0} + B \gamma,
  \]
  for some constant  \(0 < \gamma < 1/C\), \[r_T \leq \epsilon\] can be guaranteed by setting
  \begin{align*}
    \gamma &= \min\left(\frac{\epsilon}{2 B}, \frac{1}{C} \right) \text{ and } \\
    T      &\geq \frac{1}{\mu} \max\left(2 B \, \frac{1}{\epsilon}, C\right) \log\left(2 \frac{r_0}{\epsilon}\right), 
 \end{align*}
  where \(\mu, B > 0\) and \(0 < C < \mu\) are some finite constants.
\end{theoremEnd}
\begin{proofEnd}
  First, notice that the recurrence
  \[
    r_{T} \leq \underbrace{{\left(1 - \gamma \mu\right)}^{T} r_{0}}_{\text{bias}} 
    +
    \underbrace{
      B \gamma
    }_{\text{variance}}
    ,
  \]
  is a sum of monotonically increasing (variance) and decreasing (bias) terms with respect to \(\gamma\).
  Therefore, the bound is minimized when both terms are equal.
  This implies that \(r_t \leq \epsilon\) can be achieved by solving for 
  \[
    {\left(1 - \gamma \mu\right)}^{T} r_{0} \leq \frac{\epsilon}{2}
    \quad\text{and}\quad
    B \gamma \leq \frac{\epsilon}{2}
  \]
  First, for the variance term, 
  \begin{alignat*}{3}
    B \gamma &\leq \frac{\epsilon}{2}
    \quad\Leftrightarrow\quad
    \gamma \leq \frac{\epsilon}{2 B}.
  \end{alignat*}
  
  For the bias term, as long as \(\gamma < \frac{1}{\mu}\),
  \begin{alignat*}{3}
    & &
    {\left(1 - \gamma \mu\right)}^{T} r_{0} &\leq \frac{\epsilon}{2}
    \\
    &\Leftrightarrow&\qquad
    T  \log{\left(1 - \gamma \mu\right)} &\leq \log \frac{\epsilon}{2 r_{0}}
    \\
    &\Leftrightarrow&\qquad
     T  &\leq \frac{\log \frac{\epsilon}{2 r_{0}}}{\log{\left(1 - \gamma \mu\right)}}
    \\
    &\Leftrightarrow&\qquad
     T  &\geq \frac{\log \frac{2 r_{0}}{\epsilon}}{ \log{\left(1/\left(1 - \gamma \mu\right)\right)}}
  \end{alignat*}
  Furthermore, using the bound \(\log 1/x \geq 1 - x \) for \(0 < x < 1\), we can achieve the guarantee with 
  \[
     T  \geq  \frac{1}{\gamma \mu} \log\left( \frac{2 r_{0}}{\epsilon} \right).
  \]
  Therefore, \(1/\gamma\) determines the iteration complexity.
  Plugging in the minimum over the constraints on \(\gamma\) yields the iteration complexity.
\end{proofEnd}

\begin{theoremEnd}[all end, category=geometriccomplexitysquared]{lemma}\label{thm:geometric_complexity_squared}
  For a recurrence relation given as
  \[
    r_{T} \leq {\left(1 - \gamma \mu\right)}^{T} r_{0} + A \gamma^{2} + B \gamma,
  \]
  for some constant  \(0 < \gamma < 1/C\), \[r_T \leq \epsilon\] can be guaranteed by setting
  \begin{align*}
    \gamma &= \min\left(\frac{-B + \sqrt{B^2 + 2A\epsilon}}{2 A}, \frac{1}{C} \right) \text{ and } \\
    T      &\geq \frac{1}{\mu} \max\left(
       2 B \, \frac{1}{\epsilon} + \sqrt{2 A}\,\frac{1}{\sqrt{\epsilon}}, C
    \right) \log\left(2 \frac{r_0}{\epsilon}\right), 
 \end{align*}
  where \(\mu, A, B > 0\) and \(0 < C < \mu\) are some finite constants.
\end{theoremEnd}
\begin{proofEnd}
  This theorem is a generalization of Lemma A.2 by \citet{garrigos_handbook_2023}.
  First, notice that the recurrence
  \[
    r_{T} \leq \underbrace{{\left(1 - \gamma \mu\right)}^{T} r_{0}}_{\text{bias}} 
    +
    \underbrace{
      A \gamma^{2}
      +
      B \gamma
    }_{\text{variance}}
    ,
  \]
  is a sum of monotonically increasing (variance) and decreasing (bias) terms with respect to \(\gamma\).
  Therefore, the bound is minimized when both terms are equal.
  This implies that \(r_t \leq \epsilon\) can be achieved by solving for 
  \[
    {\left(1 - \gamma \mu\right)}^{T} r_{0} \leq \frac{\epsilon}{2}
    \quad\text{and}\quad
    A \gamma^{2} + B \gamma \leq \frac{\epsilon}{2}
  \]
  First, for the variance term, 
  \begin{alignat*}{3}
    & &
    A \gamma^{2} + B \gamma &\leq \frac{\epsilon}{2}
    \\
    &\Leftrightarrow&\qquad
    A \gamma^{2} + B \gamma - \frac{\epsilon}{2} &\leq 0
  \end{alignat*}
  The solution to this equation is given by the positive solution of the quadratic equation as
  \[
    0 < \gamma \leq \frac{- B + \sqrt{ B^2 + 2 A \epsilon }}{2 A}.
  \]
  
  For the bias term, as long as \(\gamma < \frac{1}{\mu}\), the solution is identical to \cref{thm:geometric_complexity}.
  Therefore,
  \begin{equation}
     T  \geq  \frac{1}{\gamma \mu} \log\left( \frac{2 r_{0}}{\epsilon} \right)
     \label{eq:reshuffle_recurrence_t_lowerbound}
  \end{equation}
  can guarantee the bias term to be smaller than \(\nicefrac{\epsilon}{2}\), while \(1/\gamma\) determines the iteration complexity.
  Plugging in the minimum over the constraints on \(\gamma\),
  \begin{equation}
     \gamma = \min\left( \frac{- B + \sqrt{ B^2 + 2 A \epsilon }}{2 A}, \frac{1}{C}\right)
     \label{eq:reshuffling_reccurrence_stepsize}
  \end{equation}
  yields the iteration complexity.

  Now, since the quadratic formula is not very interpretable, let us simplify the expression for \(1/\gamma\) using the bound
  \[
     \frac{a}{2 \sqrt{b^2 + a}} \leq - b + \sqrt{ b^2 + a},
  \]
  which holds for any \(a,b > 0\) and is tight for \(\epsilon \to 0\).
  With our constants, this reads
  \begin{align*}
     \frac{A \epsilon}{\sqrt{B^2 + 2 A \epsilon}} \leq - B + \sqrt{ B^2 + 2 A \epsilon },
  \end{align*}
  and therefore
  \begin{align*}
    \frac{2 A}{- B + \sqrt{ B^2 + 2 A \epsilon }}     
    &\leq
    \frac{2 \sqrt{B^2 + 2 A \epsilon}}{\epsilon}
    \\
    &\leq
    \frac{2 B + \sqrt{2 A \epsilon}}{\epsilon}
    \\
    &=
    2 B \frac{1}{\epsilon} + \sqrt{2 A} \frac{1}{\sqrt{\epsilon}}.
  \end{align*}
  Therefore, for the stepsize choice of \cref{eq:reshuffling_reccurrence_stepsize}, 
  \[
     \frac{1}{\gamma} \leq \min\left(2 B \frac{1}{\epsilon} + \sqrt{2 A} \frac{1}{\sqrt{\epsilon}}, \frac{1}{C}\right).
  \]
  Plugging this into \cref{eq:reshuffle_recurrence_t_lowerbound} yields the statement.
\end{proofEnd}

\begin{theoremEnd}[all end, category=averagebregman]{lemma}\label{thm:average_bregman}
  Let \(F : \mathcal{X} \to \mathbb{R}\) be a finite sum of convex functions as \(F = \frac{1}{n}\left(f_1 + \ldots + f_n\right)\), where \(f_i : \mathcal{X} \to \mathbb{R}\).
  Then,
  \[
     \frac{1}{n} \sum^n_{i=1} \mathrm{D}_{f_i}\left(\vx, \vx'\right)
     =
     \mathrm{D}_{F}\left(\vx, \vx'\right),
  \]
  for any \(\vx, \vx' \in \mathcal{X}\).
\end{theoremEnd}
\begin{proofEnd}
The result immediately follows from the definition of Bregman divergences as
  \begin{align*}
    &\frac{1}{n} \sum^n_{i=1} \mathrm{D}_{f_i}\left(\vx, \vx'\right)     
    \\
    &\;=
    \frac{1}{n} \sum^n_{i=1} 
    \left(
    f_i\left(\vx\right) - f_i\left(\vx'\right) - \inner{\nabla f_i\left(\vx'\right)}{\vx - \vx'}
    \right)
    \\
    &\;=
    \left(\frac{1}{n} \sum^n_{i=1} 
    f_i\left(\vx\right) \right)
    - 
    \left(\frac{1}{n} \sum^n_{i=1} f_i\left(\vx'\right) \right) 
    \\
    &\qquad- 
    \inner{\frac{1}{n} \sum^n_{i=1} \nabla f_i\left(\vx'\right)}{\vx - \vx'}
    \\
    &\;=
    F\left(\vx\right)
    - 
    F\left(\vx'\right)
    - 
    \inner{\nabla F\left(\vx'\right)}{\vx - \vx'}
    \\
    &\;=
    \mathrm{D}_{F}\left(\vx, \vx'\right).
  \end{align*}
\end{proofEnd}

%% file: thm_doubly_stochastic_variance.tex
\pgfkeys{/prAtEnd/local custom defaults/.style={
    no link to proof
}
}

\begin{theoremEnd}[category=doublystochasticvariance]{theorem}\label{thm:doubly_stochastic_variance}
Let the component estimators \(\rvvx_1, \ldots, \rvvx_n\) satisfy \cref{assumption:correlation}.
Then, the variance of the doubly stochastic estimator \(\rvvx_{\rvB}\) is bounded as
{%
\setlength{\abovedisplayskip}{1ex} \setlength{\abovedisplayshortskip}{1ex}
\setlength{\belowdisplayskip}{-.5ex} \setlength{\belowdisplayshortskip}{-.5ex}
\begin{align*}
  \mathrm{tr}\V{\rvvx_{\rvB}} 
  \leq
  V_{\mathrm{com}}
  +
  V_{\mathrm{cor}}
  +
  V_{\mathrm{sub}},
\end{align*}
}%
where
{%
\setlength{\abovedisplayskip}{1ex} \setlength{\abovedisplayshortskip}{1ex}
\setlength{\belowdisplayskip}{1ex} \setlength{\belowdisplayshortskip}{1ex}
\begin{align*}
   V_{\mathrm{com}} 
   &=
   {\textstyle
   \left( \frac{\rho}{b_{\mathrm{eff}}} + \frac{1 - \rho}{b} \right)  
   {\left(\frac{1}{n} \sum^n_{i=1} \mathrm{tr}\V{\rvvx_i} \right)},
   }
   \\
   V_{\mathrm{cor}}
   &=
   {\textstyle
   \rho \left(1 - \frac{1}{b_{\mathrm{eff}}}\right) {\textstyle\left(\frac{1}{n} \sum^n_{i=1} \sqrt{\mathrm{tr}\V{\rvvx_i}} \right)}^2, \text{ and}
   }
   \\
   V_{\mathrm{sub}} 
   &=
   {\textstyle
   \frac{1}{b_{\mathrm{eff}}}
   \frac{1}{n}\sum^{n}_{i=1} \norm{ \bar{\vx}_i - \bar{\vx} }_2^2.
   }
\end{align*}
Equality holds when the equality in \cref{assumption:correlation} holds.
}%
\end{theoremEnd}
\begin{proof}
We start from the law of total (co)variance,
{%
\setlength{\abovedisplayskip}{1ex} \setlength{\abovedisplayshortskip}{1ex}
\setlength{\belowdisplayskip}{1ex} \setlength{\belowdisplayshortskip}{1ex}
\[
    \mathrm{tr}\V{\rvvx_{\rvB}}
    =
    \underbrace{
    \Esub{\pi}{
    \mathrm{tr}\V{
      \rvvx_{\rvB}
      \mid
      \rvB
    }
    }
    }_{\text{Variance of ensemble}}
    +
    \underbrace{
    \mathrm{tr}\Vsub{\pi}{
      \E{
      \rvvx_{\rvB}
      \mid 
      \rvB
      }
    }.
    }_{\text{Variance of subsampling}}
\]
}%
This splits the variance into the variance of the specific ensemble of \(\rvB\) and subsampling variance.
The main challenge is to relate the variance of the ensemble of \(\rvB\) with the variance of the individual estimators in the sum
{%
\setlength{\abovedisplayskip}{1ex} \setlength{\abovedisplayshortskip}{1ex}
\setlength{\belowdisplayskip}{1ex} \setlength{\belowdisplayshortskip}{1ex}
\begin{align}
    \Esub{\pi}{\textstyle
    \mathrm{tr}\V{
      \rvvx_{\rvB}
      \mid
      \rvB
    }
    }
    &=
    \Esub{\pi}{\textstyle
    \mathrm{tr}\V{
      \frac{1}{b} \sum_{i \in \rvB} \rvvx_{i}
    }
    }.
\end{align}
}%
Since the individual estimators may not be independent, analyzing the variance of the sum can be tricky.
However, the following lemma holds generally:
\vspace{-1ex}
\begin{theoremEndRestateBefore}{lemma}[]{expectedvariancelemma}\label{thm:expected_variance_lemma}
Let \(\rvvx_1, \ldots, \rvvx_b\) be a collection of vector-variate RVs dependent on some random variable \(\rvB\) satisfying \cref{assumption:correlation}.
Then, the expected variance of the sum of \(\rvvx_1, \ldots, \rvvx_b\) conditioned on \(\rvB\) is bounded as
{%
\setlength{\abovedisplayskip}{1.5ex} \setlength{\abovedisplayshortskip}{1.5ex}
\setlength{\belowdisplayskip}{1.5ex} \setlength{\belowdisplayshortskip}{1.5ex}
\begin{align*}
   \E{ {\textstyle \mathrm{tr}\V{ \sum_{i=1}^b \rvvx_i \mid \rvB }} }
   \leq
   \rho \V{\rvS} + \rho {\left(\mathbb{E}\rvS\right)}^2 + \left(1 - \rho\right) \E{\rvV},
\end{align*}
}%
where
{%
\setlength{\abovedisplayskip}{1ex} \setlength{\abovedisplayshortskip}{1ex}
\setlength{\belowdisplayskip}{1ex} \setlength{\belowdisplayshortskip}{1ex}
\[
{\textstyle
  \rvS = \sum^{b}_{i=1} \sqrt{\mathrm{tr}\V{\rvvx_i \mid \rvB}}
  \;\;\text{and}\;\;
  \rvV = \sum^{b}_{i=1} \mathrm{tr}\V{\rvvx_i \mid \rvB}.
}
\]
}%
Equality holds when the equality in \cref{assumption:correlation} holds.
\end{theoremEndRestateBefore}
\vspace{-1ex}
Here, \(\rvS\) is the sum of conditional standard deviations, while \(\rvV\) is the sum of conditional variances.
Notice that the ``variance of the variances'' is playing a role: if we reduce the subsampling variance, then the variance of the ensemble, \(V_{\mathrm{com}}\), also decreases.

The rest of the proof, along with the proof of \cref{thm:expected_variance_lemma}, can be found in \cref{section:doubly_stochastic_variance} page \pageref{section:doubly_stochastic_variance}.
\end{proof}
\begin{proofEnd}
Starting from the law of total covariance, we have
{%
\setlength{\abovedisplayskip}{1.5ex} \setlength{\abovedisplayshortskip}{1.5ex}
\setlength{\belowdisplayskip}{1.5ex} \setlength{\belowdisplayshortskip}{1.5ex}
\begin{align}
    \V{\rvvx_{\rvB}}
    &=
    \underbrace{
    \Esub{\rvB \sim \pi}{
    \mathrm{tr}\V{
      \rvvx_{\rvB}
      \mid
      \rvB
    }
    }
    }_{\text{Ensemble Variance}}
    +
    \underbrace{
    \mathrm{tr}\Vsub{\rvB \sim \pi}{
      \E{
      \rvvx_{\rvB}
      \mid 
      \rvB
      }
    }.
    }_{\text{Subsampling Variance}}
    \label{eq:doubly_stochastic_variance_main}
\end{align}
}

\paragraph{Ensemble Variance}
Bounding the variance of each ensemble is key.
From \cref{thm:expected_variance_lemma}, we have
\begin{align}
    \E{
    \mathrm{tr}\V{
      \rvvx_{\rvB}
      \mid
      \rvB
    }
    }
    &=
    \E{
    \mathrm{tr}\V{
      \frac{1}{b} \sum_{i \in \rvB} \rvvx_{i}
      \,\middle|\,
      \rvB
    }
    }
    \nonumber
    \\
    &=
    \E{
    \mathrm{tr}\V{
      \sum_{i \in \rvB} \left( \frac{1}{b} \rvvx_{i} \right)
      \,\middle|\,
      \rvB
    }
    }
    \nonumber
    \\
    &\leq
    \rho \mathbb{V}{\rvS} + \rho {\left(\mathbb{E}{\rvS}\right)}^2
    +
    \left(1 - \rho\right) \mathbb{E}{\rvV}
    ,
    \label{eq:doubly_stochastic_variance_eq1}
\end{align}
where 
\begin{align*}
  \rvS  
  &\triangleq \sum_{i \in \rvB} \sqrt{\mathrm{tr} \V{ \frac{1}{b} \rvvx_{i}}}
  = \frac{1}{b} \sum_{i \in \rvB} \sqrt{\mathrm{tr} \V{ \rvvx_{i}}},
  \\
  \rvV &\triangleq 
  \sum_{i \in \rvB} \mathrm{tr} \V{ \frac{1}{b} \rvvx_{i} }
  =
  \frac{1}{b^2} \sum_{i \in \rvB} \mathrm{tr} \V{ \rvvx_{i} }.
\end{align*}
In our context, \(\rvS\) is the batch average of the standard deviations, and \(\rvV\) is the batch average of the variance (scaled with a factor of \(1/b\)).

Notice that \(\rvS\) is an \(b\)-sample average of the standard deviations.
Therefore, if \(\pi\) is an unbiased subsampling strategy, we retrieve the population average standard deviation as 
\begin{align}
   \Esub{\rvB \sim \pi}{\rvS} = 
   \Esub{\rvB \sim \pi}{ \frac{1}{b} \sum_{i \in \rvB} \sqrt{\mathrm{tr} \V{ \rvvx_{i} }} }
   =
   \frac{1}{n} \sum_{i = 1}^n \sqrt{\mathrm{tr} \V{ \rvvx_{i} }}.
    \label{eq:doubly_stochastic_variance_eq2}
\end{align}
Under a similar reasoning, the variance of the standard deviations follows as
\begin{align}
    &\Vsub{\rvB \sim \pi}{\rvS}
    \nonumber
    \\
    &= 
    \Vsub{\rvB \sim \pi}{ \frac{1}{b} \sum_{i \in \rvB} \sqrt{\mathrm{tr} \V{ \rvvx_{i} }}   }
    \nonumber
    \\
    &\;=
    \frac{1}{b_{\mathrm{eff}}} 
    \Vsub{i \sim \mathrm{Uniform}\left\{1, \ldots, n\right\}}{  
    \sqrt{\mathrm{tr} \V{ \rvvx_{i} }} 
    }
    \nonumber
    \\
    &\;=
    \frac{1}{b_{\mathrm{eff}}} \left( 
        \frac{1}{n} \sum_{i=1}^n \mathrm{tr} \V{ \rvvx_{i} }
        -
        {\left( \frac{1}{n} \sum_{i=1}^n \sqrt{\mathrm{tr} \V{ \rvvx_{i} }}  \right)}^2
    \right),
    \label{eq:doubly_stochastic_variance_eq3}
\end{align}
where the last identity is the well-known formula for the variance: \(\mathbb{V}\rvX = \mathbb{E}\rvX^2 - {(\mathbb{E}\rvX)}^2\).
Likewise, the average variance follows as
\begin{align}
  \mathbb{E}_{\rvB \sim \pi} \rvV
  &=
  \frac{1}{b^2}
  \Esub{\rvB \sim \pi}{ \sum_{i \in \rvB } \mathrm{tr} \V{ \rvvx_{i} } }
  \nonumber
  \\
  &=
  \frac{1}{b}
  \Esub{\rvB \sim \pi}{ \frac{1}{b} \sum_{i \in \rvB } \mathrm{tr} \V{ \rvvx_{i} } }
  \nonumber
  \\
  &=
  \frac{1}{b} 
  \left(
  \frac{1}{n} \sum_{i=1}^n \mathrm{tr} \V{ \rvvx_{i} }
  \right)
    \label{eq:doubly_stochastic_variance_eq4}
\end{align}
Plugging 
\cref{eq:doubly_stochastic_variance_eq2,eq:doubly_stochastic_variance_eq3,eq:doubly_stochastic_variance_eq4} into \cref{eq:doubly_stochastic_variance_eq1}, we have
\begin{align}
    &\Esub{\rvB \sim \pi}{
    \mathrm{tr}\V{
      \rvvx_{\rvB} \mid \rvB
    }}
    \nonumber
    \\
    &\;\leq
    \rho \mathbb{V}{\rvS} + \rho {\left(\mathbb{E}{\rvS}\right)}^2
    +
    \left(1 - \rho\right) \mathbb{E}{\rvV}
    \nonumber
    \\
    &\;=
    \frac{\rho}{b_{\mathrm{eff}}} \left( 
        \frac{1}{n} \sum_{i=1}^n \mathrm{tr} \V{ \rvvx_{i} }
        -
        {\left( \frac{1}{n} \sum_{i=1}^n \sqrt{\mathrm{tr} \V{ \rvvx_{i} }}  \right)}^2
    \right)
    \nonumber
    \\
    &\quad+
    \rho
    {\left( \frac{1}{n} \sum_{i = 1}^n \sqrt{\mathrm{tr} \V{ \rvvx_{i} }} \right)}^2
    \nonumber
    \\
    &\quad+
    \frac{1 - \rho}{b} 
    {\left(
    \frac{1}{n} \sum_{i=1}^n \mathrm{tr} \V{ \rvvx_{i} }
    \right)}
    \nonumber
    \\
    &\;=
    \left(
      \frac{\rho}{b_{\mathrm{eff}}} 
      +
      \frac{1 - \rho}{b}
    \right)
    \left( 
        \frac{1}{n} \sum_{i=1}^n \mathrm{tr} \V{ \rvvx_{i} }
    \right)   
    \nonumber
    \\
    &\quad+
    \rho \left(1 - \frac{1}{b_{\mathrm{eff}}}\right)
    {\left( \frac{1}{n} \sum_{i=1}^n \sqrt{\mathrm{tr} \V{ \rvvx_{i} }}  \right)}^2.
    \label{eq:doubly_stochastic_variance_ensemble_variance}
\end{align}

\paragraph{Subsampling Variance}
The subsampling noise is straightforward.
For this, we will denote the minibatch subsampling estimator of the component means as
\[
    \bar{\vx}_{\rvB} \triangleq \frac{1}{b} {\sum_{i \in \rvB}} \bar{\vx}_{i}.
\]
Since each component estimator \(\rvvx_i\) is unbiased, the expectation conditional on the minibatch \(\rvB\) is 
\begin{align*}
    \E{
      \bar{\vx}_{\rvB}
      \mid 
      \rvB
    }
    =
    \frac{1}{b} {\sum_{i \in \rvB}} \bar{\vx}_{i}.
\end{align*}
Therefore, 
\begin{align}
    \mathrm{tr}\Vsub{\rvB \sim \pi}{
      \E{
      \rvvx_{\rvB}
      \mid 
      \rvB
      }
    }
    &=
    \mathrm{tr}\Vsub{\rvB \sim \pi}{
    \bar{\vx}_{\rvB}
    }
    \nonumber
    \\
    &=
    \frac{1}{b_{\mathrm{eff}}} 
    \left(\frac{1}{n} {\sum^n_{i=1}} \norm{ \bar{\vx}_{i} - \bar{\vx}}_2^2 \right).
    \label{eq:doubly_stochastic_variance_subsampling_variance}
\end{align}

Combining \cref{eq:doubly_stochastic_variance_ensemble_variance,eq:doubly_stochastic_variance_subsampling_variance} into \cref{eq:doubly_stochastic_variance_main} yields the result.
Notice that the only inequality we used is \cref{eq:doubly_stochastic_variance_eq1}, \cref{thm:expected_variance_lemma}, in which equality holds if the equality in \cref{assumption:correlation} holds.
\end{proofEnd}

\pgfkeys{/prAtEnd/local custom defaults/.style={
    link to proof
}
}

%% file: thm_general_conditions.tex
\begin{theoremEnd}[category=generalconditions]{theorem}\label{thm:general_conditions}
  Let \cref{assumption:subsampling_er,assumption:estimator_correlations,assumption:montecarlo_er} hold. 
  Then, we have:
    \begin{enumerate}[label=\textbf{(\roman*)}]
      \vspace{-2ex}
      \setlength\itemsep{0ex}
      \item If \((\mathrm{A}^{\mathrm{CVX}})\) or \((\mathrm{A}^{\mathrm{ITP}})\) hold, \(\rvvg_{\rvB}\) satisfies \(\mathrm{ER}\left(\mathcal{L}_{\rm{A}}\right)\).
     \label{thm:general_conditions_expected_residual1}
      
      \item If \(\rm{(B)}\) holds, \(\rvvg_{\rvB}\) satisfies \(\mathrm{ER}\left(\mathcal{L}_{\rm{B}}\right)\).
     \label{thm:general_conditions_expected_residual3}
     \vspace{-2ex}
  \end{enumerate}
  where \(\mathcal{L}_{\mathrm{max}} = \max\big\{ \mathcal{L}_{1}, \ldots, \mathcal{L}_{n} \big\}\),
{\small%
\setlength{\abovedisplayskip}{1ex} \setlength{\abovedisplayshortskip}{1ex}
\setlength{\belowdisplayskip}{1ex} \setlength{\belowdisplayshortskip}{1ex}
  \begin{align*}
    \mathcal{L}_{\rm{A}}
    &=
    {\textstyle
    \left( \frac{\rho}{b_{\mathrm{eff}}} + \frac{1 - \rho}{b} \right) \mathcal{L}_{\mathrm{max}}
    +
    \rho \left(1 - \frac{1}{b_{\mathrm{eff}}}\right)
    {\textstyle{\left( \frac{1}{n} \sum^{n}_{i=1} \mathcal{L}_i \right)}}
    +
    \frac{\mathcal{L}_{\mathrm{sub}}}{b_{\mathrm{eff}}}
    }
    \\
    \mathcal{L}_{\rm{B}}
    &=
    {\textstyle
    \left( \frac{\rho}{b_{\mathrm{eff}}} + \frac{1 - \rho}{b} \right)
    \left( \frac{1}{n} \sum^{n}_{i=1} \mathcal{L}_{i} \right)
    }
    \\
    &\qquad+
    {\textstyle
    \rho \left(1 - \frac{1}{b_{\mathrm{eff}}}\right)
    {\left( \frac{1}{n} \sum^{n}_{i=1} \sqrt{\mathcal{L}_{i}} \right) }^2
    +
    \frac{\mathcal{L}_{\mathrm{sub}}}{b_{\mathrm{eff}}}.
    }
  \end{align*}
}%
\end{theoremEnd}
\vspace{-1.5ex}
\begin{proofEnd}
From \cref{thm:doubly_stochastic_variance}, we have
\begin{align}
  &\mathrm{tr}\V{\rvvg_{\rvB}\left(\vx\right) - \rvvg_{\rvB}\left(\vx_*\right)} 
  \nonumber
  \\
  &\;\leq 
  \left( \frac{\rho}{b_{\mathrm{eff}}} + \frac{1 - \rho}{b} \right)
  \left( \frac{1}{n} \sum^n_{i=1}  
  \mathrm{tr}\V{\rvvg_i\left(\vx\right) - \rvvg_i\left(\vx_*\right)}
  \right)
  \nonumber
  \\
  &\;\qquad+
  \rho \left(1 - \frac{1}{b_{\mathrm{eff}}}\right) 
  {\left( \frac{1}{n} \sum^n_{i=1}
  \sqrt{\mathrm{tr}\V{\rvvg_i\left(\vx\right) - \rvvg_i\left(\vx_*\right)}}
    \right)}^2 
  \nonumber
  \\
  &\;\qquad
  +
  \frac{1}{b_{\mathrm{eff}}}
  \mathrm{tr}\V{\nabla f_{\rvB}\left(\vx\right) - \nabla F\left(\vx\right)},
  \nonumber
  \shortintertext{where \cref{assumption:subsampling_er} yields}
  &\;\leq 
  \left( \frac{\rho}{b_{\mathrm{eff}}} + \frac{1 - \rho}{b} \right)
  \underbrace{
  \left( \frac{1}{n} \sum^n_{i=1}  
  \mathrm{tr}\V{\rvvg_i\left(\vx\right) - \rvvg_i\left(\vx_*\right)}
  \right)
  }_{\triangleq \mathrm{T_{\mathrm{var}}}}
  \nonumber
  \\
  &\;\qquad+
  \rho \left(1 - \frac{1}{b_{\mathrm{eff}}}\right) 
  \underbrace{
  {\left( \frac{1}{n} \sum^n_{i=1}
  \sqrt{\mathrm{tr}\V{\rvvg_i\left(\vx\right) - \rvvg_i\left(\vx_*\right)}}
    \right)}^2 
  }_{\triangleq T_{\mathrm{cov}}}
  \nonumber
  \\
  &\;\qquad
  +
  \frac{2 \mathcal{L}_{\mathrm{sub}} }{b_{\mathrm{eff}}} \left( F\left(\vx\right) - F\left(\vx_*\right) \right)
  \nonumber
  \\
  &\;=
  \left( \frac{\rho}{b_{\mathrm{eff}}} + \frac{1 - \rho}{b} \right)
  T_{\mathrm{var}}
  \nonumber
  +
  \rho \left(1 - \frac{1}{b_{\mathrm{eff}}}\right) 
  T_{\mathrm{cov}}
  \\
  &\qquad+
  \frac{2 \mathcal{L}_{\mathrm{sub}} }{b_{\mathrm{eff}}} \left( F\left(\vx\right) - F\left(\vx_*\right) \right).
  \label{eq:master_equation_er}
\end{align}

\paragraph{Proof of \labelcref{thm:general_conditions_expected_residual1} with \((\rm{A}^{\rm{CVX}})\)}
Since \cref{assumption:montecarlo_er}~\((\rm{A}^{\rm{CVX}})\) requires \(f_1, \ldots, f_n\) to be convex, \(F\) is also convex.
Therefore, we can use the identity in \cref{thm:average_bregman} and 
\[
  \mathrm{D}_{F}\left(\vx, \vx_*\right) = F\left(\vx\right) - F\left(\vx_*\right).
\]
With that said, under \((\rm{A}^{\rm{CVX}})\), we have
\begin{align}
  T_{\text{var}}
  &\leq
  \frac{1}{n} \sum^n_{i=1}  
  \mathrm{tr}\V{\rvvg_i\left(\vx\right) - \rvvg_i\left(\vx_*\right)}
  \nonumber
  \\
  &=
  \frac{1}{n} \sum^n_{i=1}  2 \mathcal{L}_i \mathrm{D}_{f_i}\left(\vx, \vx_*\right),
  \nonumber
\shortintertext{applying \(\mathcal{L}_{\mathrm{max}} \geq \mathcal{L}_i\) for all \(i = 1, \ldots, n\),} 
  &\leq
  2 \mathcal{L}_{\mathrm{max}}
  \frac{1}{n} \sum^n_{i=1} 
  \mathrm{D}_{f_i}\left(\vx, \vx_*\right)
  \nonumber
\shortintertext{and \cref{thm:average_bregman},} 
  &=
  2 \mathcal{L}_{\mathrm{max}} \mathrm{D}_{F}\left(\vx, \vx_*\right).
  \label{eq:ercondition_item1_tvar}
\end{align}

For \(T_{\text{cov}}\), since
\begin{align}
  T_{\text{cov}}
  &=
  {\left( \frac{1}{n} \sum^n_{i=1}
  \sqrt{\mathrm{tr}\V{\rvvg_i\left(\vx\right) - \rvvg_i\left(\vx_*\right)}}
    \right)}^2 
  \nonumber
\shortintertext{
  is monotonic w.r.t. the variance, we can apply \((\rm{A}^{\rm{CVX}})\) as 
}    
  &\leq
  \frac{2}{n^2} 
  {\left( 
  \sum^n_{i=1}
  \sqrt{ 
    \mathcal{L}_i \mathrm{D}_{f_i}\left(\vx, \vx_*\right)
  }
  \right)}^2.
  \nonumber
\shortintertext{Now, applying the Cauchy-Schwarz inequality yields}  
  &\leq
  \frac{2}{n^2} 
  \left( 
  \sum^n_{i=1} \mathcal{L}_i
  \right)
  \left( 
  \sum^n_{i=1} \mathrm{D}_{f_i}\left(\vx, \vx_*\right)
  \right)
  \nonumber
  \\
  &=
  2
  \left( 
  \frac{1}{n}
  \sum^n_{i=1} \mathcal{L}_i
  \right)
  \left( 
  \frac{1}{n}
  \sum^n_{i=1} \mathrm{D}_{f_i}\left(\vx, \vx_*\right)
  \right)
  \nonumber
\shortintertext{and by \cref{thm:average_bregman},}  
  &=
  2
  \left( 
  \frac{1}{n}
  \sum^n_{i=1} \mathcal{L}_i
  \right)
  \mathrm{D}_{F}\left(\vx, \vx_*\right).
  \label{eq:ercondition_item1_tcov}
\end{align}
Plugging \cref{eq:ercondition_item1_tvar,eq:ercondition_item1_tcov} into \cref{eq:master_equation_er}, we have
\begin{align*}
  &\mathrm{tr}\V{\rvvg\left(\vx\right) - \rvvg\left(\vx_*\right)} 
  \\
  &\;\leq 
  \left( \frac{\rho}{b_{\mathrm{eff}}} + \frac{1 - \rho}{b} \right)
  T_{\mathrm{var}}
  \nonumber
  +
  \rho \left(1 - \frac{1}{b_{\mathrm{eff}}}\right) 
  T_{\mathrm{cov}}
  \\
  &\qquad+
  \frac{2 \mathcal{L}_{\mathrm{sub}} }{b_{\mathrm{eff}}} \left( F\left(\vx\right) - F\left(\vx_*\right) \right)
  \\
  &\;\leq
  \left( \frac{\rho}{b_{\mathrm{eff}}} + \frac{1 - \rho}{b} \right)
  2 \mathcal{L}_{\mathrm{max}} \mathrm{D}_{F}\left(\vx, \vx_*\right)
  \\
  &\;\qquad+
  \rho \left(1 - \frac{1}{b_{\mathrm{eff}}}\right) 
  2
  \left( 
  \frac{1}{n}
  \sum^n_{i=1} \mathcal{L}_i
  \right)
  \mathrm{D}_{F}\left(\vx, \vx_*\right)
  \\
  &\;\qquad+
  \frac{1}{b_{\mathrm{eff}}} 
  2 \mathcal{L}_{\mathrm{sub}} \mathrm{D}_{F}\left(\vx, \vx_*\right).
  \\
  &=
  2
  \,\Bigg(\,
    \left( \frac{\rho}{b_{\mathrm{eff}}} + \frac{1 - \rho}{b} \right)
    \mathcal{L}_{\mathrm{max}}
    +
    \rho \left(1 - \frac{1}{b_{\mathrm{eff}}}\right) 
    \left( 
    \frac{1}{n}
    \sum^n_{i=1} \mathcal{L}_i
    \right)
    \\
    &\qquad\qquad+
    \frac{1}{b_{\mathrm{eff}}} 
    \mathcal{L}_{\mathrm{sub}} 
  \,\Bigg)\,
  \mathrm{D}_{F}\left(\vx, \vx_*\right)
  \\
  &=
  2
  \,\Bigg(\,
    \left( \frac{\rho}{b_{\mathrm{eff}}} + \frac{1 - \rho}{b} \right)
    \mathcal{L}_{\mathrm{max}} 
    +
    \rho \left(1 - \frac{1}{b_{\mathrm{eff}}}\right) 
    \left( 
    \frac{1}{n}
    \sum^n_{i=1} \mathcal{L}_i
    \right)
    \\
    &\qquad\qquad+
    \frac{1}{b_{\mathrm{eff}}} 
    \mathcal{L}_{\mathrm{sub}} 
  \,\Bigg)\,
  \left( F\left(\vx\right) - F\left(\vx_*\right) \right).
\end{align*}

\paragraph{Proof of \labelcref{thm:general_conditions_expected_residual1} with \((\rm{A}^{\rm{ITP}})\)}
From \cref{assumption:montecarlo_er}~\((\rm{A}^{\rm{ITP}})\), we have
\begin{align}
  T_{\text{var}}
  &=
  \frac{1}{n} \sum^n_{i=1}  
  \mathrm{tr}\V{\rvvg_i\left(\vx\right) - \rvvg_i\left(\vx_*\right)}
  \nonumber
  \\
  &\leq
  \frac{1}{n} \sum^n_{i=1}  2 \mathcal{L}_i \left(f_i\left(\vx\right) - f_i\left(\vx_*\right)\right),
  \nonumber
\shortintertext{applying \(\mathcal{L}_{\mathrm{max}} \geq \mathcal{L}_i\) for all \(i = 1, \ldots, n\),} 
  &\leq
  2 \mathcal{L}_{\mathrm{max}}
  \frac{1}{n} \sum^n_{i=1} 
  \left(f_i\left(\vx\right) - f_i\left(\vx_*\right)\right)
  \nonumber
  \\
  &=
  2 \mathcal{L}_{\mathrm{max}} \left( F\left(\vx\right) - F\left(\vx_*\right) \right).
  \label{eq:ercondition_item2_tvar}
\end{align}

Similarly,
\begin{align}
  T_{\text{cov}}
  &=
  {\left( \frac{1}{n} \sum^n_{i=1}
  \sqrt{\mathrm{tr}\V{\rvvg_i\left(\vx\right) - \rvvg_i\left(\vx_*\right)}}
    \right)}^2,
  \nonumber
\shortintertext{applying \((\rm{A}^{\rm{ITP}})\),} 
  &\leq
  \frac{2}{n^2} 
  {\left( 
  \sum^n_{i=1}
  \sqrt{ 
    \mathcal{L}_i \left( f_i\left(\vx\right) - f_i\left(\vx_*\right) \right)
  }
  \right)}^2,
  \nonumber
\shortintertext{and applying the Cauchy-Schwarz inequality,}  
  &\leq
  \frac{2}{n^2} 
  \left( 
  \sum^n_{i=1} \mathcal{L}_i
  \right)
  \left( 
  \sum^n_{i=1}
  f_i\left(\vx\right) - f_i\left(\vx_*\right)
  \right)
  \nonumber
  \\
  &=
  2
  \left( 
  \frac{1}{n}
  \sum^n_{i=1} \mathcal{L}_i
  \right)
  \left( 
  \frac{1}{n}
  \sum^n_{i=1}
  f_i\left(\vx\right) - f_i\left(\vx_*\right)
  \right)
  \nonumber
  \\
  &=
  2
  \left( 
  \frac{1}{n}
  \sum^n_{i=1} \mathcal{L}_i
  \right)
  \left( F\left(\vx\right) - F\left(\vx_*\right) \right).
  \label{eq:ercondition_item2_tcov}
\end{align}

Plugging \cref{eq:ercondition_item2_tvar,eq:ercondition_item2_tcov} into \cref{eq:master_equation_er}, we have
\begin{align*}
  &\mathrm{tr}\V{\rvvg\left(\vx\right) - \rvvg\left(\vx_*\right)} 
  \\
  &\;\leq 
  \left( \frac{\rho}{b_{\mathrm{eff}}} + \frac{1 - \rho}{b} \right)
  T_{\mathrm{var}}
  +
  \rho \left(1 - \frac{1}{b_{\mathrm{eff}}}\right) 
  T_{\mathrm{cov}}
  \\
  &\qquad+
  \frac{2 \mathcal{L}_{\mathrm{sub}} }{b_{\mathrm{eff}}} \left( F\left(\vx\right) - F\left(\vx_*\right) \right)
  \\
  &\;\leq
  \left( \frac{\rho}{b_{\mathrm{eff}}} + \frac{1 - \rho}{b} \right)
  2 \mathcal{L}_{\mathrm{max}} \left( F\left(\vx\right) F\left(\vx_*\right) \right)
  \\
  &\;\qquad+
  \rho \left(1 - \frac{1}{b_{\mathrm{eff}}}\right) 
  2
  \left( 
  \frac{1}{n}
  \sum^n_{i=1} \mathcal{L}_i
  \right)
  \left(F\left(\vx\right) - F\left(\vx_*\right)\right)
  \\
  &\;\qquad+
  \frac{1}{b_{\mathrm{eff}}} 
  2 \mathcal{L}_{\mathrm{sub}} \left(F\left(\vx\right) - F\left(\vx_*\right)\right).
  \\
  &=
  2
  \,\Bigg(\,
    \left( \frac{\rho}{b_{\mathrm{eff}}} + \frac{1 - \rho}{b} \right)
    \mathcal{L}_{\mathrm{max}} 
    +
    \rho \left(1 - \frac{1}{b_{\mathrm{eff}}}\right) 
    \left( 
    \frac{1}{n}
    \sum^n_{i=1} \mathcal{L}_i
    \right)
    \\
    &\qquad\qquad+
    \frac{1}{b_{\mathrm{eff}}} 
    \mathcal{L}_{\mathrm{sub}} 
  \,\Bigg)\,
  \left( F\left(\vx\right) - F\left(\vx_*\right)\right).
\end{align*}

\paragraph{Proof of \labelcref{thm:general_conditions_expected_residual3}}
From \cref{assumption:montecarlo_er} (B), we have
\begin{align}
  T_{\text{var}}
  &=
  \frac{1}{n} \sum^n_{i=1}  
  \mathrm{tr}\V{\rvvg_i\left(\vx\right) - \rvvg_i\left(\vx_*\right)}
  \nonumber
  \\
  &\leq
  \frac{1}{n} \sum^n_{i=1}  2 \mathcal{L}_i \left(F\left(\vx\right) - F\left(\vx_*\right)\right)
  \nonumber
  \\
  &=
  2 
  \left( \frac{1}{n} \sum^n_{i=1} \mathcal{L}_{i} \right)
  \left(F\left(\vx\right) - F\left(\vx_*\right)\right).
  \label{eq:ercondition_item3_tvar}
\end{align}

And,
\begin{align}
  T_{\text{cov}}
  &=
  {\left( \frac{1}{n} \sum^n_{i=1}
  \sqrt{\mathrm{tr}\V{\rvvg_i\left(\vx\right) - \rvvg_i\left(\vx_*\right)}}
    \right)}^2 
    \nonumber
  \\
  &\leq
  \frac{2}{n^2} 
  {\left( 
  \sum^n_{i=1}
  \sqrt{ 
    \mathcal{L}_i \left( F\left(\vx\right) - F\left(\vx_*\right) \right)
  }
  \right)}^2
    \nonumber
  \\
  &=
  2
  {\left( 
  \frac{1}{n} 
  \sum^n_{i=1}
  \sqrt{ 
    \mathcal{L}_i 
  }
  \right)}^2
  \left( F\left(\vx\right) - F\left(\vx_*\right) \right).
  \label{eq:ercondition_item3_tcov}
\end{align}

Plugging \cref{eq:ercondition_item3_tcov,eq:ercondition_item3_tvar} into \cref{eq:master_equation_er}, we have
\begin{align*}
  &\mathrm{tr}\V{\rvvg\left(\vx\right) - \rvvg\left(\vx_*\right)} 
  \\
  &\;\leq 
  \left( \frac{\rho}{b_{\mathrm{eff}}} + \frac{1 - \rho}{b} \right)
  2 \left( \frac{1}{n} \sum^n_{i=1}  
  \mathcal{L}_i \right) \left(F\left(\vx\right) - F\left(\vx_*\right) \right)
  \\
  &\;\qquad+
  \rho \left(1 - \frac{1}{b_{\mathrm{eff}}}\right) 
  {2 \left( \frac{1}{n} \sum^n_{i=1}
  \sqrt{ 
    \mathcal{L}_i 
  }
  \right)}^2 
  \left(F\left(\vx\right) - F\left(\vx_*\right) \right)
  \\
  &\;\qquad+
  \frac{1}{b_{\mathrm{eff}}} 
  2 \mathcal{L}_{\mathrm{sub}} \left(F\left(\vx\right) - F\left(\vx_*\right) \right),
  \\
  &\;=
  2 \,
  \Bigg(
  \left( \frac{\rho}{b_{\mathrm{eff}}} + \frac{1 - \rho}{b} \right)
  \left( \frac{1}{n} \sum^n_{i=1}  
  \mathcal{L}_i 
  \right)
  \\
  &\quad\qquad+
  \rho \left(1 - \frac{1}{b_{\mathrm{eff}}}\right) 
  {\left( \frac{1}{n} \sum^n_{i=1}
  \sqrt{ 
    \mathcal{L}_i 
  }
  \right)}^2 
  \\
  &\quad\qquad+
  \frac{1}{b_{\mathrm{eff}}} 
  \mathcal{L}_{\mathrm{sub}} 
  \Bigg) \,
  \left(F\left(\vx\right) - F\left(\vx_*\right) \right).
\end{align*}    

\end{proofEnd}

%% file: thm_bounded_variance.tex
\begin{theoremEnd}[category=boundedvariance]{theorem}\label{thm:bounded_variance}
Let \cref{assumption:estimator_correlations,assumption:bounded_variance_both} hold.
Then, \(\rvvg_{\rvB}\) satisfies \(\mathrm{BV}\left(\sigma^2\right)\), where
{%
\setlength{\abovedisplayskip}{1ex} \setlength{\abovedisplayshortskip}{1ex}
\setlength{\belowdisplayskip}{1ex} \setlength{\belowdisplayshortskip}{1ex}
  \begin{align*}
    \sigma^2
    &=
    \left( \frac{\rho}{b_{\mathrm{eff}}} + \frac{1 - \rho}{b} \right)
    \left( \textstyle\frac{1}{n} \sum^n_{i=1} \sigma_{i}^2 \right)
    \\
    &\qquad+
    \rho \left(1 - \frac{1}{b_{\mathrm{eff}}}\right) {\textstyle\left( \frac{1}{n} \sum^n_{i=1} \sigma_{i} \right)}^2
    +
    \frac{\tau^2}{b_{\mathrm{eff}}}.
  \end{align*}
}%
Equality in \cref{assumption:bounded_variance} holds if equality in \cref{assumption:estimator_correlations} holds.
\end{theoremEnd}
\vspace{-1.5ex}
\begin{proofEnd}
For any element of the solution set \(\vx_* = \argmin_{\vx \in \mathcal{X}} F\left(\vx\right)\), by \cref{thm:doubly_stochastic_variance}, we have
\begin{align*}
  \mathrm{tr}\V{\rvvg_{\rvB}\left(\vx_*\right)} 
  &\leq 
  \left( \frac{\rho}{b_{\mathrm{eff}}} + \frac{1 - \rho}{b}  \right)
  \left( \frac{1}{n} \sum^n_{i=1}  
  \mathrm{tr}\V{\rvvg_i\left(\vx_*\right)}
  \right)
  \\
  &\qquad+
  \rho \left(1 - \frac{1}{b_{\mathrm{eff}}}\right) 
  {\left( \frac{1}{n} \sum^n_{i=1}
  \sqrt{\mathrm{tr}\V{\rvvg_i\left(\vx_*\right)}}
    \right)}^2 
    \\
  &\qquad+
  \frac{1}{b_{\mathrm{eff}}} 
  \mathrm{tr}\V{\nabla f_{\rvB}\left(\vx_*\right)}.
\shortintertext{Applying \cref{assumption:bounded_variance_both}, we have}  
  &\leq 
  \left( \frac{\rho}{b_{\mathrm{eff}}} + \frac{1 - \rho}{b}  \right)
  \left( \frac{1}{n} \sum^n_{i=1} \sigma_{i}^{2} \right)
  \\
  &\qquad+
  \left(1 - \frac{1}{b_{\mathrm{eff}}}\right) 
  {\left( \frac{1}{n} \sum^n_{i=1}
  \sqrt{ 
    \sigma^2_{i}
  }
  \right)}^2 
  \\
  &\qquad+
  \frac{1}{b_{\mathrm{eff}}} \tau^2
  \\
  &=
  \left(1 - \frac{1}{b_{\mathrm{eff}}}\right) 
  \left( \frac{1}{n} \sum^n_{i=1} \sigma_{i}^{2} \right)
  \\
  &\qquad+
  \rho \left(1 - \frac{1}{b_{\mathrm{eff}}}\right) 
  {\left( \frac{1}{n} \sum^n_{i=1} \sigma_{i} \right)}^2 
  \\
  &\qquad+
  \frac{1}{b_{\mathrm{eff}}} \tau^2.
\end{align*}    
\end{proofEnd}

%% file: thm_doubly_stochastic_without_replacement.tex
\begin{theoremEnd}[all end, category=expectedresidualwihtoutreplacement]{lemma}\label{thm:expected_residual_without_replacement}
    Let the objective function \(F\) satisfy \cref{assumption:components}, \(\pi\) be sampling \(b\) samples without replacement, and all elements of the solution set \(\vx_* \in \argmin_{\vx \in \mathcal{X}} F\left(\vx\right)\) be stationary points of \(F\).
    Then, the subsampling estimator \(\nabla f_{\rvB}\) satisfies the \(\mathrm{ER}\) condition as
    \begin{align*}
      &\mathrm{tr}\Vsub{\rvB \sim \pi}{ \nabla f_{\rvB}\left(\vx\right) - \nabla f_{\rvB}\left(\vx_*\right)  }
      \\
      &\qquad\leq
      2 \frac{n - b}{b \left(n - 1\right)} L_{\mathrm{max}} \left(F\left(\vx\right) - F\left(\vx_*\right)\right),
    \end{align*}
    where \(L_{\mathrm{max}} = \max\left\{L_1, \ldots, L_n\right\}\).
\end{theoremEnd}
\begin{proofEnd}
Consider that, for any random vector \(\rvvx\), 
\[
  \mathrm{tr}\V{\rvvx^2} \leq \mathbb{E}\norm{\rvvx}_2^2
\]
holds.
Also, sampling without replacement achieves \(b_{\mathrm{eff}} = \frac{(n-1) b}{n-b}\).
Therefore, we have
\begin{align*}
      &\mathrm{tr}\Vsub{\rvB \sim \pi}{ \nabla f_{\rvB}\left(\vx\right) - \nabla f_{\rvB}\left(\vx_*\right)  }
      \\
      &\;=
      \frac{n - b}{b \left(n - 1\right)}  
      \mathrm{tr}\V{
        \nabla f_{i}\left(\vx\right) - \nabla f_{i}\left(\vx_*\right) 
      }
      \\
      &\;\leq
      \frac{n - b}{b \left(n - 1\right)} 
      \left(
        \frac{1}{n} \sum_{i=1}^n \norm{ \nabla f_{i}\left(\vx\right) - \nabla f_{i}\left(\vx_*\right) }_2^2
      \right),
\shortintertext{and from \cref{assumption:components},}
      &\;=
      \frac{n - b}{b \left(n - 1\right)} 
      \left(
        \frac{1}{n} \sum_{i=1}^n 2 L_i \mathrm{D}_{f_i}\left(\vx, \vx_* \right)
      \right).
\shortintertext{Using the bound \(L_{\mathrm{max}} \geq L_i\) for all \(i = 1, \ldots, n\),} 
      &\;\leq
      2 L_{\mathrm{max}} \frac{n - b}{b \left(n - 1\right)} \left( \frac{1}{n} \sum_{i=1}^n  \mathrm{D}_{f_i}\left(\vx, \vx_* \right) \right),
\shortintertext{applying \cref{thm:average_bregman},} 
      &\;=
      2 L_{\mathrm{max}} \frac{n - b}{b \left(n - 1\right)}  \mathrm{D}_{F}\left(\vx, \vx_* \right),
\shortintertext{and since \(\vx_*\) is a stationary point of \(F\),} 
      &\;=
      2 \frac{n - b}{b \left(n - 1\right)} L_{\mathrm{max}} \left(F\left(\vx\right) - F\left(\vx_*\right) \right).
\end{align*}
\end{proofEnd}

%% file: thm_strongly_convex_sgd_sampling_without_replacement_complexity.tex
\begin{theoremEnd}[category=stronglyconvexsgdsamplingwithoutreplacement]{corollary}\label{thm:strongly_convex_sgd_sampling_without_replacement}
Let the objective \(F\) satisfy \cref{assumption:objective,assumption:components}, the global optimum \(\vx_* = \argmin_{\vx \in \mathcal{X}} F\left(\vx\right)\) be a stationary point of \(F\), the component gradient estimators \(\rvvg_1, \ldots, \rvvg_n\) satisfy \cref{assumption:montecarlo_er} (B) and \labelcref{assumption:bounded_variance_both}, and \(\pi\) be \(b\)-minibatch sampling without replacement.
Then the last iterate of SGD with \(\rvvg_{\rvB}\) is \(\epsilon\)-close to \(\vx_*\) as \(\mathbb{E}\norm{\vx_T - \vx_*}_2^2 \leq \epsilon\) if
{%
\setlength{\abovedisplayskip}{.5ex} \setlength{\abovedisplayshortskip}{.5ex}
\setlength{\belowdisplayskip}{.5ex} \setlength{\belowdisplayshortskip}{.5ex}
\begin{align*}
    T \geq 2 \max\left(C_{\mathrm{var}} \frac{1}{\epsilon},\; C_{\mathrm{bias}}\right)
    \log\left(2 {\lVert \vx_0 - \vx_* \rVert}_2^2 \frac{1}{\epsilon} \right)
\end{align*}
}%
for some fixed stepsize where
{%
\setlength{\abovedisplayskip}{.5ex} \setlength{\abovedisplayshortskip}{.5ex}
\setlength{\belowdisplayskip}{.5ex} \setlength{\belowdisplayshortskip}{.5ex}
\begin{align*}
  C_{\mathrm{var}} &=
  \frac{2}{b} \left( {\frac{1}{n} \sum^n_{i=1} \frac{\sigma_{i}^2}{\mu^2} }\right) 
  + 
  { 2\left( \frac{1}{n} \sum^n_{i=1} \frac{\sigma_{i}}{\mu} \right)}^2 
  + \frac{2}{b} \frac{\tau^2}{\mu^2},
  \\
  C_{\mathrm{bias}} &=
  \frac{2}{b} \left( {\frac{1}{n} \sum^n_{i=1} \frac{\mathcal{L}_i}{\mu} }\right) 
  + 
  {\textstyle 2 \left( \frac{1}{n} \sum^n_{i=1} \sqrt{\frac{\mathcal{L}_i}{\mu}} \right)}^2 + \frac{2}{b} \frac{L}{\mu}.
\end{align*}
}%
\end{theoremEnd}
\vspace{-1.5ex}
\begin{proofEnd}
From \cref{assumption:components} and the assumption that \(\vx_*\) is a stationary point, \cref{thm:expected_residual_without_replacement} establishes that \(\nabla f_{\rvB}\) satisfies the \(\mathrm{ER}\left(\mathcal{L}_{\mathrm{sub}}\right)\) holds with 
\[
  \mathcal{L}_{\mathrm{sub}} = \frac{n-b}{(n-1) b} L_{\mathrm{max}}.
\]
Therefore, \cref{assumption:subsampling_er} holds.
Furthermore, since the component gradient estimators satisfy \cref{assumption:montecarlo_er} (B) and \cref{assumption:correlation} always hold with \(\rho = 1\), we can apply \cref{thm:general_conditions} which estblishes that \(\rvvg_{\rvB}\) satisfies \(\mathrm{ER}\left(\mathcal{L}\right)\) with 
\begin{align*}
  \mathcal{L}
  &=
  \frac{n-b}{(n-1) b} \left( \frac{1}{n} \sum^n_{i=1} \mathcal{L}_i \right) 
  + 
   \frac{n (b-1)}{(n-1) b} {\left( \frac{1}{n} \sum^n_{i=1} \sqrt{\mathcal{L}_i} \right)}^2 
  \\
  &\quad+ \frac{n-b}{(n-1) b} L_{\mathrm{max}}.
\end{align*}
Furthermore, under \cref{assumption:bounded_variance_both}, \cref{thm:bounded_variance} shows that \(\mathrm{BV}\left(\sigma^2\right)\) holds with 
\begin{align*}
  \sigma^2
  &=
  \frac{n-b}{(n-1) b} \left(\frac{1}{n} \sum^n_{i=1} \sigma^2_{i}\right) 
  + \frac{n (b-1)}{(n-1) b} {\left(\frac{1}{n} \sum^n_{i=1} \sigma_{i}\right)}^2
  \\
  &\quad+ \frac{n-b}{(n-1) b} \tau^2.
\end{align*}
Since both \(\mathrm{ER}\left(\mathcal{L}\right)\) and \(\mathrm{BV}\left(\sigma^2\right)\) hold and \(F\) satisfies \cref{assumption:objective}, we can now invoke \cref{thm:strongly_convex_sgd_complexity}, which guarantees that we can obtain an \(\epsilon\)-accurate solution after
\[
    T \geq 2 \max\Bigg( \underbrace{\frac{\sigma^2}{\mu^2}}_{C_{\mathrm{var}}} \frac{1}{\epsilon}, \underbrace{\frac{\mathcal{L} + L}{\mu}}_{C_{\mathrm{bias}}} \Bigg) \log\left(2 \norm{\vx_0 - \vx_*}_2^2 \frac{1}{\epsilon} \right)
\]
iterations and fixed stepsize of
\[
    \gamma = \min\left( \frac{\epsilon \mu}{2 \sigma^2}, \frac{1}{2 \left(\mathcal{L} + L\right)} \right).
\]

The constants in the lower bound on the number of required iterations can be made more precise as
{%
\setlength{\abovedisplayskip}{1ex} \setlength{\abovedisplayshortskip}{1ex}
\setlength{\belowdisplayskip}{1ex} \setlength{\belowdisplayshortskip}{1ex}
\begin{align*}
  C_{\mathrm{var}} &= 
  {
  \frac{n-b}{(n-1) b} \left( {\frac{1}{n} \sum^n_{i=1} \frac{\sigma_{i}^2}{\mu^2} }\right) 
  }
  \\
  &\quad+ 
  {
  \frac{n (b-1)}{(n-1) b} {\left( \frac{1}{n} \sum^n_{i=1} \frac{\sigma_{i}}{\mu} \right)}^2 + \frac{n-b}{(n-1)b} \frac{\tau^2}{\mu^2}
  }
  \\
  C_{\mathrm{bias}} &= 
  {
  \frac{n-b}{(n-1) b} \left( {\frac{1}{n} \sum^n_{i=1} \frac{\mathcal{L}_i}{\mu} }\right) 
  }
  \\
  &\quad+ 
  {
  \frac{n (b-1)}{(n-1) b} {\left( \frac{1}{n} \sum^n_{i=1} \sqrt{\frac{\mathcal{L}_i}{\mu}} \right)}^2 + \frac{n-b}{(n-1) b} \frac{L}{\mu}.
  }
\end{align*}
Using the fact that \((n-b)/n \leq (n-1)/n \leq 2\) for all \(n \geq 2\) yields the simplified constants in the statement.
}%
\end{proofEnd}

%% file: figures/tikz_isoquad_variance_low_hetero.tex
\begin{tikzpicture}
\begin{axis}[
    tuftelike,
    width=0.26\textwidth,
    xmode=log,
    log basis x={2},
    ymax   = 2.0,
    ymin   = 0.0,
    xlabel = \(b\),
    ylabel = gradient variance,
    xmin={1},
    xmax={1024},
    xtick={1,4,16,64,256,1024},
    scaled x ticks=false,
    axis line style = thick,
    axis x line shift=2ex,
    axis y line shift=2ex,
    every tick/.style={black,thick},
    tick label style={font=\small},  
    legend style={
        draw=none,
        anchor=north east,
        at={(1.05,1.05)},
    },
    legend cell align=left,
]

	\addplot[color=cherry6, very thick] coordinates {
(1,1.8859556618621411)
(2,0.9436282183769147)
(4,0.47342045009396055)
(8,0.24022847287180202)
(16,0.12745629809935952)
(32,0.07871783839041192)
(64,0.06964386389048541)
(128,0.09569738734961672)
(256,0.1699051704973715)
(512,0.32937110490762717)
(1024,0.6538281577855116)
    };
    \addlegendentry{\footnotesize\(mb=1024\)}
    
	\addplot[color=cherry4, very thick] coordinates {
(1,1.899030949947676)
(2,0.9611646226075259)
(4,0.4998790866147244)
(8,0.28453157397287104)
(16,0.20744832836103888)
(32,0.23008772697331195)
(64,0.3637694691158268)
(128,0.6753344258598407)
    };
    \addlegendentry{\footnotesize\(mb=128\)}
    
	\addplot[color=cherry2, very thick] coordinates {
(1,2.0036332546319557)
(2,1.1014558564524162)
(4,0.7115481787808353)
(8,0.6389563827814232)
(16,0.8473845704544738)
    };
    \addlegendentry{\footnotesize\(mb=16\)}
\end{axis}
\end{tikzpicture}

%% file: figures/tikz_isoquad_variance_mid_hetero.tex
\begin{tikzpicture}
\begin{axis}[
    tuftelike,
    width=0.26\textwidth,
    xmode=log,
    log basis x={2},
    ymax   = 6.0,
    ymin   = 0.0,
    xlabel = \(b\),
    xmin={1},
    xmax={1024},
    xtick={1,4,16,64,256,1024},
    scaled x ticks=false,
    axis line style = thick,
    axis x line shift=2ex,
    axis y line shift=2ex,
    every tick/.style={black,thick},
    tick label style={font=\small},  
]

	\addplot[color=cherry6, very thick] coordinates {
(1,7.538218952554764)
(2,3.7669972716945392)
(4,1.8823423847240863)
(8,0.9419268481581784)
(16,0.4755428937138612)
(32,0.24999854416897627)
(64,0.15252162475108105)
(128,0.13437367575122802)
(256,0.18648072266949062)
(512,0.3348962889650002)
(1024,0.6538281577855116)
    };

	\addplot[color=cherry4, very thick] coordinates {
(1,7.551294240640298)
(2,3.7845336759251507)
(4,1.9088010212448503)
(8,0.9862299492592475)
(16,0.5555349239755406)
(32,0.4013684327518763)
(64,0.4466472299764224)
(128,0.714010714261452)
    };
    
	\addplot[color=cherry2, very thick] coordinates {
(1,7.655896545324579)
(2,3.9248249097700407)
(4,2.120470113410961)
(8,1.3406547580677997)
(16,1.1954711660689754)
    };
\end{axis}
\end{tikzpicture}

%% file: figures/tikz_isoquad_variance_high_hetero.tex
\begin{tikzpicture}
\begin{axis}[
    tuftelike,
    width=0.26\textwidth,
    xmode=log,
    log basis x={2},
    ymax   = 30.0,
    ymin   = 0.0,
    xlabel = \(b\),
    xmin={1},
    xmax={1024},
    xtick={1,4,16,64,256,1024},
    scaled x ticks=false,
    axis line style = thick,
    axis x line shift=2ex,
    axis y line shift=2ex,
    every tick/.style={black,thick},
    tick label style={font=\small},  
]

	\addplot[color=cherry6, very thick] coordinates {
(1,30.147272115325254)
(2,15.060473484965039)
(4,7.518030123244589)
(8,3.7487203493036843)
(16,1.867889276171868)
(32,0.9351213672832336)
(64,0.48403266819346363)
(128,0.2890788293576732)
(256,0.25278293135796714)
(512,0.3569970251944924)
(1024,0.6538281577855116)
    };
    
	\addplot[color=cherry4, very thick] coordinates {
(1,30.16034740341079)
(2,15.07800988919565)
(4,7.544488759765353)
(8,3.7930234504047533)
(16,1.9478813064335474)
(32,1.0864912558661337)
(64,0.7781582734188051)
(128,0.8687158678678972)
    };
    
	\addplot[color=cherry2, very thick] coordinates {
(1,30.264949708095067)
(2,15.21830112304054)
(4,7.756157851931464)
(8,4.147448259213306)
(16,2.5878175485269823)
    };
\end{axis}
\end{tikzpicture}

%% file: thm_reshuffle_variance.tex
\begin{theoremEnd}[all end, category=reshufflevariance]{lemma}\label{thm:reshuffle_variance}
For any \(b\)-minibatch reshuffling strategy, the squared error of the reference point of the Lyapunov function (\cref{eq:lyapunov_bias}) under reshuffling is bounded as
\begin{align*}
    \mathbb{E} {\lVert\vx_{*}^{i} - \vx_*\rVert}_2^2
    \leq
    \frac{\gamma^2 n }{4 b^2} \, \tau^2
\end{align*}
for all \(i = 1, \ldots, p\), where \(\vx_* \in \argmin_{\vx \in \mathcal{X}} F\left(\vx\right)\).
\end{theoremEnd}
\begin{proofEnd}
The proof is a generalization of \citet[Proposition 1]{mishchenko_random_2020}, where we sample \(b\)-minibatches instead of single datapoints.
Recall that \(\rvP\) denotes the (possibly random) partitioning of the \(n\) datapoints into \(b\)-minibatches \(\rvP_1, \ldots, \rvP_p\).
From the definition of the squared error of the Lyapunov function in \cref{eq:lyapunov_bias}, we have
\begin{align*}
    &\E{ {\lVert\vx_{*}^{i} - \vx_*\rVert}_2^2 }
    \\
    &\;=
    \E{ \norm{ \Pi_{\mathcal{X}}\left( \vx_* - \sum_{k=0}^{i-1} \gamma \nabla f_{\rvP_i}\left(\vx_*\right) \right) - \Pi_{\mathcal{X}}\left(\vx_*\right) }_2^2 },
\shortintertext{and since the projection onto a convex set under a Euclidean metric is non-expansive,}
    &\;\leq
    \E{ \norm{ \vx_* - \sum_{k=0}^{i-1} \gamma \nabla f_{\rvP_i}\left(\vx_*\right) - \vx_* }_2^2 }
    \\
    &\;=
    \E{ 
    \norm{
      \sum_{k=0}^{i-1} \gamma \nabla f_{\rvP_i}\left(\vx_*\right)
    }_2^2
    },
\shortintertext{introducing a factor of \(i\) in and out of the squared norm,}
    &\;=
    \frac{i^2}{2} \, 
    \E{ 
      \norm{
      \frac{1}{i}
      \sum_{k=0}^{i-1} \gamma \nabla f_{\rvP_i}\left(\vx_*\right)
      }_2^2
    }
    \\
    &\;=
    \frac{\gamma^2  i^2}{2} \, 
    \E{ \norm{
      \frac{1}{i}
      \sum_{k=0}^{i-1} \nabla f_{\rvP_i}\left(\vx_*\right)
    }_2^2 }.
\end{align*}
Now notice that \(\frac{1}{i} \sum^{i-1}_{j=0} \nabla f_{\rvP_i} \left(\vx_*\right)\) is a sample average of \(i b\) samples drawn without replacement.
Therefore, it is an unbiased estimate of \(\nabla F\left(\vx_*\right)\).
This implies 
\begin{align*}
    \E{ {\lVert\vx_{*}^{i} - \vx_*\rVert}_2^2 }
    &=
    \frac{\gamma^2  i^2}{2} \, 
    \E{ \norm{
      \frac{1}{i}
      \sum_{k=0}^{i-1} \nabla f_{\rvP_i}\left(\vx_*\right)
    }_2^2 }
    \\
    &=
    \frac{\gamma^2  i^2}{2} \, 
    \mathrm{tr}\V{ 
      \frac{1}{i}
      \sum_{k=0}^{i-1} \nabla f_{\rvP_i}\left(\vx_*\right)
    },
\shortintertext{and from \cref{thm:variancewithoutreplacement} with a sample size of \(ib\), } 
    &=
    \frac{\gamma^2  i^2}{2} \, 
    \frac{n - ib}{\left(n-1\right) ib}
    \frac{1}{n}  \sum^n_{i=1} \norm{\nabla f_{i}\left(\vx_*\right)}_2^2
    \\
    &=
    \frac{\gamma^2 i \left(\frac{n}{b} - i\right)}{2 \left(n-1\right)} \tau^2.
\end{align*}
Notice that this is a quadratic with respect to \(i\), where the maximum is obtained by \(i = \nicefrac{n}{2b}\).
Then,
\begin{align*}
    \E{ {\lVert\vx_{*}^{i} - \vx_*\rVert}_2^2 }
    &\leq
    \frac{\gamma^2 {\left(\frac{n}{2b}\right)}^2 }{2 \left(n-1\right)} \tau^2
    \\
    &=
    \frac{\gamma^2 n^2 }{8 b^2 \left(n-1\right)} \tau^2,
\shortintertext{and using the bound \(n/(n-1) \leq 2\) for all \(n \geq 2\),} 
    &\leq
    \frac{\gamma^2 n }{4 b^2} \, \tau^2.
\end{align*}

\end{proofEnd}

%% file: thm_strongly_convex_reshuffling_sgd_convergence.tex
\begin{theoremEnd}[all end, category=expectedsmoothnessreshuffling]{lemma}\label{thm:reshuffling_expected_smoothness}
Let the objective function satisfy \cref{assumption:components}, \(B\) be any \(b\)-minibatch of indices such that \(B \subseteq \{1, \ldots, n\}\) and the component gradient estimators \(\rvvg_1, \ldots, \rvvg_n\) satisfy \cref{assumption:montecarlo_er} (\(\rm{A}^{\rm{CVX}}\)).
Then, \(\rvvg_{B}\) is convex-smooth in expectation such that
\begin{align*}
  \mathbb{E}_{\varphi}\norm{\rvvg_{B}\left(\vx\right) - \rvvg_{B}\left(\vx_*\right)}_2^2
  \leq
  2 \left( \mathcal{L}_{\rm{max}} + L_{\mathrm{max}} \right) \mathrm{D}_{f_{B}} \left(\vx, \vx_*\right),
\end{align*}
for any \(\vx \in \mathcal{X}\), where
\begin{alignat*}{3}
   \vx_* &= \argmin_{\vx \in \mathcal{X}} F\left(\vx\right),  \\
   \mathcal{L}_{\rm{max}} &= \max\left\{\mathcal{L}_{1}, \ldots, \mathcal{L}_{n} \right\}, \\
   L_{\mathrm{max}} &= \max\left\{L_1, \ldots, L_n\right\}.
\end{alignat*}
\end{theoremEnd}
\begin{proofEnd}
Notice that, for this Lemma, we do not assume that the minibatch \(B\) is a random variable.
Therefore, the only randomness is the stochasticity of the component gradient estimators \(\rvvg_1, \ldots, \rvvg_n\).

Now, from the property of the variance, we can decompose the expected squared norm as
\begin{align*}
  &\mathbb{E} {\lVert \rvvg_B\left(\vx\right) - \rvvg_B\left(\vx_*\right)\rVert}_2^2
  \\
  &\;=
  \underbrace{
  \mathrm{tr}\Vsub{\varphi}{
    \rvvg_B\left(\vx\right) - \rvvg_B\left(\vx_*\right)
  }
  }_{V_{\rm{com}}}
  +
  \underbrace{
  {\lVert 
  \nabla f_{B} \left(\vx\right) - \nabla f_{B}\left(\vx_*\right)
  \rVert}_2^2
  }_{V_{\mathrm{sub}}}.
\end{align*}

First, the contribution of the variances of the component gradient estimators follows as
\begin{align}
  V_{\rm{com}}
  &=
  \mathrm{tr}\Vsub{\varphi}{
    \rvvg\left(\vx\right) - \rvvg\left(\vx_*\right)
  }
  \nonumber
  \\
  &=
  \mathrm{tr}\Vsub{\varphi}{
    \frac{1}{b} \sum_{i \in B} \rvvg_{i}\left(\vx\right) - \rvvg_{i}\left(\vx_*\right)
  },
  \nonumber
\shortintertext{applying \cref{eq:varianceofsumeq2} of \cref{thm:varianceofsum},} 
  &\leq
  \frac{1}{b} \sum_{i \in B}
  \mathrm{tr}\Vsub{\varphi}{
      \rvvg_{i}\left(\vx\right) - \rvvg_{i}\left(\vx_*\right)
  },
  \label{eq:loose_bound_for_reshuffling}
\shortintertext{and then \cref{assumption:montecarlo_er} (\(\rm{A}^{\rm{CVX}}\)),}
  &\;\leq
  \frac{1}{b} \sum_{i \in B} 2 \mathcal{L}_{i} \, \mathrm{D}_{f_i}\left(\vx, \vx_*\right).
  \nonumber
\shortintertext{Now, since \(\mathcal{L}_{\mathrm{max}} \geq \mathcal{L}_i\) for all \(i= 1, \ldots, n\),}
  &\;\leq
  2 \mathcal{L}_{\rm{max}} \frac{1}{b} \sum_{i \in B} \mathrm{D}_{f_i}\left(\vx, \vx_*\right)
  \nonumber
  \\
  &\;=
  2 \mathcal{L}_{\rm{max}} \mathrm{D}_{f_{B}}\left(\vx, \vx_*\right).
  \nonumber
\end{align}

On the other hand, the squared error of subsampling (it is not the variance since we do not take expectation over the batches) follows as
\begin{align*}
  V_{\mathrm{sub}}
  &=
  {\lVert 
  \nabla f_{\rvB} \left(\vx\right) - \nabla f_{\rvB}\left(\vx_*\right)
  \rVert}_2^2
  \\
  &=
  \norm{ 
    \frac{1}{b} \sum_{i \in B} \nabla f_{i} \left(\vx\right) - \nabla f_{i}\left(\vx_*\right)
  }_2^2,
\shortintertext{by Jensen's inequality,}
  &\leq
  \frac{1}{b} \sum_{i \in B} 
  {\lVert 
    \nabla f_{i} \left(\vx\right) - \nabla f_{i}\left(\vx_*\right)
  \rVert}_2^2,
\shortintertext{from \cref{assumption:components},}
  &\leq
  \frac{1}{b} \sum_{i \in B} 
  2 L_{i} \mathrm{D}_{f_i}\left(\vx, \vx_*\right)
\shortintertext{and since \(L_{\mathrm{max}} \geq L_i\) for all \(i= 1, \ldots, n\),}
  &\leq
  2 L_{\mathrm{max}} \frac{1}{b} \sum_{i \in B} \mathrm{D}_{f_i}\left(\vx, \vx_*\right)
  \\
  &=
  2 L_{\mathrm{max}} \mathrm{D}_{f_B}\left(\vx, \vx_*\right).
\end{align*}
Combining the bound on \(V_{\rm{com}}\) and \(V_{\rm{sub}}\) immediately yields the result.
\end{proofEnd}

\begin{theoremEnd}[all end, category=stronglyconvexreshufflingsgdconvergence]{theorem}\label{thm:strongly_convex_reshuffling_sgd_convergence}
  Let the objective \(F\) satisfy \cref{assumption:objective,assumption:components}, where, each component \(f_i\) is additionally \(\mu\)-strongly convex and 
  \cref{assumption:montecarlo_er} (\(\rm{A}^{\rm{CVX}}\)), \labelcref{assumption:bounded_variance_both} hold.
  Then, the last iterate \(\vx_T\) of doubly SGD-RR with a stepsize satisfying \(\gamma < 1/\left(\mathcal{L}_{\mathrm{max}} + L_{\mathrm{max}}\right)\) guarantees 
  \begin{align*}
    &\mathbb{E}{\lVert \vx_{K+1}^{0} - \vx_* \rVert}_2^2
    \leq
    r^{K p}
    {\lVert \vx_1^{0} - \vx_* \rVert}_2^2
    +
    C_{\mathrm{var}}^{\rm{sub}} \, \gamma^2 
    +
    C_{\mathrm{var}}^{\rm{com}} \, \gamma
  \end{align*}
  where \(p = n/b\) is the number of epochs, \(\vx_* = \argmin_{\vx \in \mathcal{X}} F\left(\vx\right)\), \(r =  1 - \gamma \mu\) is the contraction coefficient,
\begin{alignat*}{3}
    C_{\mathrm{var}}^{\rm{com}} &= 
    {
    \frac{4}{\mu b} {\left( \frac{1}{n} \sum^n_{i=1} \sigma^2_{i}  \right)}
    + 
    \frac{4}{\mu}
    \left(\frac{1}{n} \sum^n_{i=1} \sigma_{i}\right)}^2, \text{ and}
    \\
    C_{\mathrm{var}}^{\rm{sub}} &= \frac{1}{4} \frac{L_{\mathrm{max}}}{\mu} \frac{n}{b^2} \left( \frac{1}{n} \sum^n_{i=1} \norm{\nabla f_i\left(\vx_*\right)}_2^2 \right).
\end{alignat*}
\end{theoremEnd}
\begin{proofEnd}
The key element of the analysis of random reshuffling is that the Lyapunov function that achieves a fast convergence is \({\lVert \vx_k^{i+1} - \vx^{i+1}_{*} \rVert}_2^2\) not \({\lVert \vx_{k}^{i+1} - \vx_* \rVert}_2^2\).
This stems from the well-known fact that random reshuffling results in a conditionally biased gradient estimator.

Recall that \(\rvP\) denotes the partitioning of the \(n\) datapoints into \(b\)-minibatches \(\rvP_1, \ldots, \rvP_p\).
As usual, we first expand the Lyapunov function as
\begin{align*}
    &{\lVert \vx^{i+1}_k - \vx_*^{i+1}\rVert}_2^2
    \\
    &\;=
    {\lVert \Pi_{\mathcal{X}}(\vx^{i}_k - \gamma \, \rvvg_{\rvP_i}(\vx^i_k)) - \Pi_{\mathcal{X}}(\vx_*^i - \gamma \, \nabla f_{\rvP_i}\left(\vx_*\right) ) \rVert}_2^2
\shortintertext{and since the projection onto a convex set under a Euclidean metric is non-expansive,}
    &\;\leq
    {\lVert (\vx^{i}_k - \gamma \, \rvvg_{\rvP_i}(\vx^i_k)) - (\vx_*^i - \gamma \, \nabla f_{\rvP_i}\left(\vx_*\right) ) \rVert}_2^2
    \\
    &\;=
    {\lVert \vx^{i}_k -  \vx_*\rVert}_2^2
    -2 \gamma \inner{
      \vx^{i}_k -  \vx_*^i \;
    }{
       \rvvg_{\rvP_i}(\vx^i_k) - \nabla f_{\rvP_i} \left(\vx_*\right)
    }
    \\
    &\qquad+
    \gamma^2 {\lVert \rvvg_{\rvP_i}(\vx^i_k) - \nabla f_{\rvP_i}\left(\vx_*\right)\rVert}_2^2.
\end{align*}

Taking expectation over the Monte Carlo noise conditional on the partitioning \(\rvP\),
\begin{align*}
    &\mathbb{E}_{\varphi}{\lVert \vx^{i+1}_k - \vx_*^{i+1}\rVert}_2^2
    \\
    &\;=
    {\lVert \vx^{i}_k -  \vx_*^i\rVert}_2^2
    -2 \gamma \inner{
      \vx^{i}_k -  \vx_*^i \;
    }{
    \mathbb{E}_{\varphi}[\rvvg_{\rvP_i}(\vx^i_k)] - \nabla f_{\rvP_i} \left(\vx_*\right)
    }
    \\
    &\qquad+
    \gamma^2 
    \mathbb{E}_{\varphi}
    {\lVert \rvvg_{\rvP_i}(\vx^i_k) - \nabla f_{\rvP_i} \left(\vx_*\right) \rVert}
    \\
    &\;=
    {\lVert \vx^{i}_k -  \vx_*^i \rVert}_2^2
    -2 \gamma 
    \inner{
      \vx^{i}_k -  \vx_*^i \;
    }{
      \nabla f_{\rvP_i}(\vx_k^i) - \nabla f_{\rvP_i}\left(\vx_*\right)
    }
    \\
    &\qquad+
    \gamma^2 \,
    \mathbb{E}_{\varphi} {\lVert \rvvg_{\rvP_i}(\vx^i_k) - \nabla f_{\rvP_i} \left(\vx_*\right) \rVert}_2^2.
\end{align*}

From the three-point identity, we can more precisely characterize the effect of the conditional bias such that
\begin{align*}
    &\inner{
      \vx^{i}_k -  \vx_*^i \;
    }{
      \nabla f_{\rvP_i}(\vx_k^i) - \nabla f_{\rvP_i}\left(\vx_*\right)
    }
    \\
    &=
    \mathrm{D}_{f_{\rvP_i}}(\vx_*^i, \vx_k^i)
    +
    \mathrm{D}_{f_{\rvP_i}}(\vx_k^i, \vx_*)
    -
    \mathrm{D}_{f_{\rvP_i}}(\vx_*^i, \vx_*).
\end{align*}

For the gradient noise, 
\begin{align*}
  &\mathbb{E}_{\varphi} {\lVert \rvvg_{\rvP_i}(\vx^i_k) - \nabla f_{\rvP_i} \left(\vx_*\right) \rVert}_2^2
  \\
  &\;=
  \mathbb{E}_{\varphi} {\lVert 
    \rvvg_{\rvP_i}(\vx^i_k) - \rvvg_{\rvP_i} \left(\vx_*\right) + \rvvg_{\rvP_i}\left(\vx_*\right) - \nabla f_{\rvP_i} \left(\vx_*\right) 
  \rVert}_2^2
  \\
  &\;\leq
  2 \mathbb{E}_{\varphi} {\lVert 
    \rvvg_{\rvP_i}(\vx^i_k) - \rvvg_{\rvP_i} \left(\vx_*\right) 
  \rVert}_2^2
  +
  2 
  \mathbb{E}_{\varphi} {\lVert 
    \rvvg_{\rvP_i}\left(\vx_*\right) - \nabla f_{\rvP_i} \left(\vx_*\right) 
  \rVert}_2^2
  \\
  &\;=
  2 \,
  \mathbb{E}_{\varphi} 
  {\lVert 
    \rvvg_{\rvP_i}(\vx^i_k) - \rvvg_{\rvP_i} \left(\vx_*\right) 
  \rVert}_2^2
  +
  2 \,
  \mathrm{tr} \Vsub{\varphi}{ \rvvg_{\rvP_i}\left(\vx_*\right) },
\shortintertext{and from \cref{thm:reshuffling_expected_smoothness},}
  &\;\leq
  4 \left( \mathcal{L}_{\mathrm{max}} + L_{\mathrm{max}} \right) \mathrm{D}_{f_{\rvP_i}}(\vx_k^i, \vx_*)
  +
  2 \mathrm{tr} \Vsub{\varphi}{ \rvvg_{\rvP_i}\left(\vx_*\right) }
\end{align*}
Notice the variance term \(\mathrm{tr} \Vsub{\varphi}{ \rvvg_{\rvP_i}\left(\vx_*\right) }\).
This quantifies the amount of deviation from the trajectory of singly stochastic random reshuffling.
As such, it quantifies how slower we will be compared to its fast rate.

Now, we will denote the \(\sigma\)-algebra formed by the randomness and the iterates up to the \(i\)th step of the \(k\)th epoch as \(\mathcal{F}_k^i\) such that \((\mathcal{F}_k^i)_{k \geq 1, i \geq 1}\) is a filtration.
Then,
\begin{align*}
    &\Esub{\rvveta_{k}^{i} \sim \varphi}{{\lVert \vx^{i+1}_k - \vx_*^{i+1}\rVert}_2^2 \,\middle|\, \mathcal{F}_k^i}
    \\
    &\leq
    {\lVert \vx^{i}_k -  \vx_*^i\rVert}_2^2
    \\
    &\quad-2 \gamma
    \left(
    \mathrm{D}_{f_{\rvP_i}}(\vx_*^i, \vx_k^i)
    +
    \mathrm{D}_{f_{\rvP_i}}(\vx_k^i, \vx_*)
    -
    \mathrm{D}_{f_{\rvP_i}}(\vx_*^i, \vx_*)
    \right)
    \\
    &\quad+
    4 \gamma^2 \left(\mathcal{L}_{\mathrm{max}} + L_{\mathrm{max}}\right)  
    \mathrm{D}_{f_{\rvP_i}}(\vx_k^i, \vx_*) 
    \\
    &\quad+
    2 \gamma^2  \mathrm{tr} \Vsub{\varphi}{ \rvvg_{\rvP_i}\left(\vx_*\right) }.
\shortintertext{Now, the \(\mu\)-strong convexity of the component functions imply \(\mathrm{D}_{f_{\rvP_i}}\left(\vx_*^i, \vx_k^i\right) \leq \frac{\mu}{2} {\lVert \vx_k^i - \vx_*^i \rVert}_2^2\). Therefore,}
    &\leq
    {\lVert \vx^{i}_k -  \vx_*^i\rVert}_2^2
    \\
    &\quad-
    2 \gamma
    \left(
    \frac{\mu}{2} {\lVert \vx_k^i - \vx_*^i \rVert}_2^2
    +
    \mathrm{D}_{f_{\rvP_i}}(\vx_k^i, \vx_*)
    -
    \mathrm{D}_{f_{\rvP_i}}(\vx_*^i, \vx_*)
    \right)
    \\
    &\quad+
    4 \gamma^2 \left( \mathcal{L}_{\mathrm{max}} + L_{\mathrm{max}} \right) \mathrm{D}_{f_{\rvP_i}}(\vx_k^i, \vx_*)
    \\
    &\quad+
    2 \gamma^2  \mathrm{tr} \Vsub{\varphi}{ \rvvg_{\rvP_i}\left(\vx_*\right) }
    ,
\shortintertext{and reorganizing the terms,}
    &=
    \left(1 - \gamma \mu \right)
    {\lVert \vx^{i}_k -  \vx_*^i\rVert}_2^2
    \\
    &\quad-2 \gamma
    \left(1 - 2 \gamma \left(\mathcal{L}_{\mathrm{max}} + L_{\mathrm{max}} \right)\right) 
    \mathrm{D}_{f_{\rvP_i}}(\vx_k^i, \vx_*)
    \\
    &\quad+ 2 \gamma \, \mathrm{D}_{f_{\rvP_i}}(\vx_*^i, \vx_*)
    \\
    &\quad+
    \gamma^2 2 \mathrm{tr} \Vsub{\varphi}{ \rvvg_{\rvP_i}\left(\vx_*\right) }.
\end{align*}

Taking full expectation, 
\begin{align}
    &\mathbb{E}{\lVert \vx^{i+1}_k - \vx_*^{i+1}\rVert}_2^2
    \nonumber
    \\
    &\;\leq
    \left(1 - \gamma \mu \right)
    \mathbb{E}{\lVert \vx^{i}_k -  \vx_*^i\rVert}_2^2
    \nonumber
    \\
    &\quad-
    2 \gamma
    \left(1 - 2 \gamma \left(\mathcal{L}_{\mathrm{max}} + L_{\mathrm{max}} \right)\right) 
    \E{ \mathrm{D}_{f_{\rvP_i}}\left(\vx_k^n, \vx_*\right) }
    \nonumber
    \\
    &\quad+ 
    2 \gamma \, \E{ \mathrm{D}_{f_{\rvP_i}}\left(\vx_*^i, \vx_*\right)}
    \nonumber
    \\
    &\quad+
    2 \gamma^2 
    \E{ \mathrm{tr} \Vsub{\varphi}{ \rvvg_{\rvP_i}\left(\vx_*\right) } },
    \nonumber
\shortintertext{and as long as \(\gamma < {1}/{\left( 2 \left(\mathcal{L}_{\mathrm{max}} + L_{\mathrm{max}}\right) \right)}\)}
    &\;\leq
    \left(1 - \gamma \mu \right)
    \mathbb{E}{\lVert \vx^{i}_k -  \vx_*^i\rVert}_2^2
    + 
    2 \gamma \underbrace{ \E{ \mathrm{D}_{f_{\rvP_i}}\left(\vx_*^i, \vx_*\right)} }_{T_{\text{err}}}
    \nonumber
    \\
    &\quad+
    2 \gamma^2  
    \underbrace{
    \E{ \mathrm{tr} \Vsub{\varphi}{ \rvvg_{\rvP_i}\left(\vx_*\right) } }
    }_{T_{\text{var}}}
    .
    \label{eq:reshuffling_convergence_main_bound}
\end{align}

\paragraph{Bounding \(T_{\mathrm{err}}\)}
From the definition of the Bregman divergence and \(L\)-smoothness, for all \(j = 1, \ldots, n\), notice that we have
\begin{align}
    \mathrm{D}_{f_j}\left( \vy, \vx \right)
    &=
    f_j\left(\vy\right)
    -
    f_j\left(\vx\right)
    -
    \inner{\nabla f_j\left(\vx\right)}{ \vy - \vx}
    \nonumber
    \\
    &\leq
    \frac{L}{2} \norm{\vy - \vx}_2^2.
    \label{eq:bregman_smooth_upperbound}
\end{align}
for all \((\vx, \vx') \in \mathcal{X}^2\).
Given this, the Lyapunov error term
\begin{align}
    \E{ \mathrm{D}_{f_{\rvP_i}}\left(\vx_k^i, \vx_*\right) }
    &=
    \E{
    \frac{1}{b}
    \sum_{j \in \rvP_i}
    \mathrm{D}_{f_{j}}\left(\vx_k^i, \vx_*\right)
    }
    \nonumber
\shortintertext{can be bounded using \(L\)-smoothness by \cref{eq:bregman_smooth_upperbound},} 
    &\leq
    \E{
    \frac{1}{b}
    \sum_{j \in \rvP_i}
    \frac{L_{j}}{2} {\lVert \vx_k^i -  \vx_*\rVert}_2^2
    }
    \nonumber
\shortintertext{and \(L_{\rm{max}} \geq L_i\) for all \(i = 1, \ldots, n\),}
    &\leq
    \frac{L_{\mathrm{max}}}{2}
    \E{
    \frac{1}{b}
    \sum_{j \in \rvP_i}
    {\lVert \vx_k^i -  \vx_*\rVert}_2^2
    }
    \nonumber
    \\
    &=
    \frac{L_{\mathrm{max}}}{2}
    \mathbb{E} {\lVert \vx_k^i -  \vx_*\rVert}_2^2.
    \label{eq:reshuffling_convergence_Terr}
\end{align}

The squared error \({\lVert \vx_k^i -  \vx_*\rVert}_2^2\) is bounded in \cref{thm:reshuffle_variance} as
\begin{equation}
  \mathbb{E} {\lVert \vx_k^i -  \vx_*\rVert}_2^2 
  \leq
  \epsilon_{\mathrm{sfl}}^2
  \triangleq
  \frac{\gamma^2 n }{4 b^2} \, \tau^2
  <
  \infty.
  \label{eq:reshuffling_epsilon}
\end{equation}

\paragraph{Bounding \(T_{\mathrm{var}}\)}
Now, let's take a look at the variance term.
First, notice that, by the Law of Total Expectation,
\begin{align*}
  \E{ \mathrm{tr} \Vsub{\varphi}{ \rvvg_{\rvP_i}\left(\vx_*\right) } }
  =
  \E{ 
  \E{ 
  \mathrm{tr} \Vsub{\varphi}{ \rvvg_{\rvP_i}\left(\vx_*\right) }  \mid \rvP
  }
  }.
\end{align*}
Here, 
\[
  \E{ 
    \mathrm{tr} \Vsub{\varphi}{ \rvvg_{\rvP_i}\left(\vx_*\right) } \mid \rvP
  }
\]
is the variance from selecting \(b\) samples without replacement.
We can thus apply \cref{thm:expected_variance_lemma} with \(b_{\mathrm{eff}} = \frac{(n-1) b}{n-b}\) such that
\begin{align}
  &\E{ 
    \mathrm{tr} \Vsub{\varphi}{ \rvvg_{\rvP_i}\left(\vx_*\right) } \mid \rvP
  }
  \nonumber
  \\
  &\;\leq
  \frac{n - b}{\left(n - 1\right) b} 
  \left(\frac{1}{n} \sum^n_{j=1} \sigma^2_{j}\right)
  +
  \frac{n \left(b - 1\right)}{\left(n - 1\right) b} 
  {\left(\frac{1}{n} \sum^n_{j=1} \sigma_{j}\right)}^2,
  \nonumber
\shortintertext{which we will denote as}  
  &\;= \sigma^2
  \label{eq:reshuffling_convergence_variance}
\end{align}
for clarity.
Also, notice that \(\sigma^2\) no longer depends on the partitioning.

\paragraph{Per-step Recurrence Equation}
Applying \cref{eq:reshuffling_convergence_Terr,eq:reshuffling_convergence_variance} to \cref{eq:reshuffling_convergence_main_bound}, we now have the recurrence equation
\begin{align*}
    \mathbb{E}{\lVert \vx^{i+1}_k - \vx_*^{i+1}\rVert}_2^2
    &\leq
    \left(1 - \gamma \mu \right)
    \mathbb{E}{\lVert \vx^{i}_k -  \vx_*^i\rVert}_2^2
    \\
    &\qquad+
    L_{\mathrm{max}} \epsilon_{\mathrm{sfl}}^2 \, \gamma
    +
    2 \sigma^2 \, \gamma^2.
\end{align*}

Now that we have a contraction of the Lyapunov function \(\mathbb{E}{\lVert \vx_k^{i+1} - \vx_*^{i+1} \rVert}_2^2\), it remains to convert this that the Lyapunov function bounds our objective \(\mathbb{E}{\lVert \vx_k^{i+1} - \vx_* \rVert}_2^2\).
This can be achieved by noticing that, at the end of each epoch, we have \(\vx_{k+1} - \vx_* = \vx^{p}_{k} - \vx^p_*\), and equivalently, we have \(\vx_k - \vx_* = \vx^{0}_k - \vx^{0}_*\) at the beginning of the epoch.
The fact that the relationship with the original objective is only guaranteed at the endpoints (beginning and end of the epoch) is related to the fact that the bias of random reshuffling starts increasing at the beginning of the epoch and starts decreasing near the end.

\paragraph{Per-Epoch Recurrence Equation}
Nevertheless, this implies that by simply unrolling the recursion as in the analysis of regular SGD, we obtain a per-epoch contraction of
\begin{align*}
  \mathbb{E}{\lVert \vx_{k+1}^0 - \vx_* \rVert}_2^2
  &\leq
  {\left(1 - \gamma \mu\right)}^{p}
  \mathbb{E} {\lVert \vx_{k}^0 - \vx_* \rVert}_2^2
  \\
  &\quad+
  \left(
  L_{\mathrm{max}} \epsilon^2_{\mathrm{sfl}}  \gamma
  +
  2 \sigma^2 \gamma^2 
  \right)
  \left(\sum^{p-1}_{i=0} {\left(1 - \mu \gamma\right)}^{i} \right).
\end{align*}
And after \(K\) epochs, 
\begin{align*}
  &\mathbb{E}{\lVert \vx_{K+1}^0 - \vx_* \rVert}_2^2
  \leq
  {\left(1 - \gamma \mu\right)}^{p K}
  \mathbb{E} {\lVert \vx_{0}^0 - \vx_* \rVert}_2^2
  \\
  &+
  \left(
  L_{\mathrm{max}} \epsilon^2_{\mathrm{sfl}} \gamma 
  +
  2 \sigma^2 \gamma^2 
  \right)
  \left(\sum^{p-1}_{i=0} {\left(1 - \mu \gamma\right)}^{i} \right)
  \left(\sum^{pK-1}_{j=0} {\left(1 - \mu \gamma\right)}^{p j} \right).
\end{align*}
Note that \(T = pK\).

As done by \citet{mishchenko_random_2020}, the product of sums can be bounded as
\begin{align*}
  &\left(\sum^{p-1}_{i=0} {\left(1 - \mu \gamma\right)}^{i} \right)
  \left(\sum^{T-1}_{j=0} {\left(1 - \mu \gamma\right)}^{p j} \right)
  \\
  &\;=
  \sum^{p-1}_{i=0} \sum^{T-1}_{j=0} 
  {\left(1 - \mu \gamma\right)}^{i} {\left(1 - \mu \gamma\right)}^{p j}
  \\
  &\;=
  \sum^{p-1}_{i=0} \sum^{T-1}_{j=0} 
  {\left(1 - \mu \gamma\right)}^{i + pj}
  \\
  &\;=
  \sum^{Tp-1}_{i=0} 
  {\left(1 - \mu \gamma\right)}^{i}
  \\
  &\;\leq
  \sum^{\infty}_{i=0} 
  {\left(1 - \mu \gamma\right)}^{i}
  \\
  &\;\leq
  \frac{1}{\gamma \mu}.
\end{align*}
Then,
\begin{align*}
  &\mathbb{E}{\lVert \vx_{K+1}^0 - \vx_* \rVert}_2^2
  \\
  &\;\leq
  {\left(1 - \gamma \mu\right)}^{p K}
  \mathbb{E} {\lVert \vx_{0}^0 - \vx_* \rVert}_2^2
  +
  \frac{1}{\gamma \mu}
  \left(
  L_{\mathrm{max}} \epsilon^2_{\mathrm{sfl}}  \gamma 
  +
  2 \sigma^2 \gamma^2 
  \right)
  \\
  &=
  {\left(1 - \gamma \mu\right)}^{p K}
  \mathbb{E} {\lVert \vx_{0}^0 - \vx_* \rVert}_2^2
  +
  \frac{\epsilon^2_{\mathrm{sfl}} }{\mu}
  +
  \frac{2\sigma^2}{\mu} \gamma.
\end{align*}
Plugging in the value of \(\epsilon_{\mathrm{sfl}}^2\) from \cref{eq:reshuffling_epsilon}, we have
\begin{align*}
  \mathbb{E}{\lVert \vx_{K+1}^0 - \vx_* \rVert}_2^2
  &\leq
  {\left(1 - \gamma \mu\right)}^{p K}
  \mathbb{E} {\lVert \vx_{0}^0 - \vx_* \rVert}_2^2
  \\
  &\qquad+
  \frac{L_{\mathrm{max}} n \sigma^2_{\mathrm{sub}}}{4 b^2 \mu} \gamma^2
  +
  \frac{2\sigma^2}{\mu} \gamma.
\end{align*}
This implies
\begin{align*}
    &\mathbb{E}{\lVert \vx_{K+1}^{0} - \vx_* \rVert}_2^2
    \leq
    r^{K n/b}
    {\lVert \vx_1^{0} - \vx_* \rVert}_2^2
    +
    C_{\rm{var}}^{\mathrm{sub}} \, \gamma^2 
    +
    C_{\rm{var}}^{\mathrm{com}} \, \gamma,
  \end{align*}
  where \(r =  1 - \gamma \mu\),
\begin{alignat*}{3}
    C_{\rm{var}}^{\mathrm{sub}} &= \frac{1}{4} \frac{L_{\mathrm{max}}}{\mu} \frac{n}{b^2} \left( \frac{1}{n} \sum^n_{i=1} \norm{\nabla f_i\left(\vx_*\right)}_2^2 \right), \text{and}
    \\
    C_{\mathrm{var}}^{\rm{com}} &= 
    {
    \frac{2}{\mu} \frac{n-b}{(n-1) b} {\left( \frac{1}{n} \sum^n_{i=1} \sigma^2_{i}  \right)}
    + 
    \frac{2}{\mu} \frac{n \left(b-1\right)}{\left(n-1\right) b} 
    {\left(\frac{1}{n} \sum^n_{i=1} \sigma_{i}\right)}^2
    }.
\end{alignat*}
Applying the fact that \( (n-b)/n \leq (n-1)/n \leq 2 \) for all \( n \geq 2\) yields the simplified constants in the statement.
\end{proofEnd}

%% file: thm_strongly_convex_reshuffling_sgd_complexity.tex
\begin{theoremEnd}[category=stronglyconvexreshufflingsgdcomplexity]{theorem}\label{thm:strongly_convex_reshuffling_sgd_complexity}
  Let the objective \(F\) satisfy \cref{assumption:objective,assumption:components}, where each component \(f_i\) is additionally \(\mu\)-strongly convex, and \cref{assumption:montecarlo_er} (\(\rm{A}^{\rm{CVX}}\)), \labelcref{assumption:bounded_variance_both} hold.
  Then, the last iterate \(\vx_T\) of doubly SGD-RR is \(\epsilon\)-close to the global optimum \(\vx_* = \argmax_{\vx \in \mathcal{X}} F\left(\vx\right)\) such that \(\mathbb{E}\norm{\vx_T - \vx_*}_2^2 \leq \epsilon\) if
{%
\setlength{\abovedisplayskip}{1ex} \setlength{\abovedisplayshortskip}{1ex}
\setlength{\belowdisplayskip}{1ex} \setlength{\belowdisplayshortskip}{1ex}
  \begin{align*}
    \hspace{-.5ex}
    T \,\geq  \;
    &
    {\textstyle
    \max\left( 4 C_{\rm{var}}^{\mathrm{com}} \frac{1}{\epsilon} + C_{\rm{var}}^{\mathrm{sub}} \frac{1}{\sqrt{\epsilon}}, \;  C_{\mathrm{bias}} \right) 
    \log \left(2 \, {\small\lVert \vx_1^0 - \vx_* \rVert}_2^2 \frac{1}{\epsilon} \right)
    }
  \end{align*}
  }%
  for some fixed stepsize, where \(T = K p = K {n}/{b}\),
  {
\setlength{\abovedisplayskip}{1ex} \setlength{\abovedisplayshortskip}{1ex}
\setlength{\belowdisplayskip}{1ex} \setlength{\belowdisplayshortskip}{1ex}
  \begin{align*}
    C_{\mathrm{bias}} &= \left(\mathcal{L}_{\mathrm{max}} + L\right)/{\mu}
    \\
    C_{\rm{var}}^{\mathrm{com}} &= 
    {
    \frac{2}{b} \left(\frac{1}{n} \sum^n_{i=1} \frac{\sigma_{i}^2}{\mu^2} \right) 
    + { 2 \left( \frac{1}{n}\sum^n_{i=1} \frac{\sigma_{i}}{\mu} \right)}^2
    },
    \\
    C_{\rm{var}}^{\mathrm{sub}} &= \sqrt{\frac{L_{\mathrm{max}}}{\mu}} \frac{ \sqrt{n}}{b} \frac{\tau}{\mu}.
  \end{align*}
  }%
\end{theoremEnd}
\vspace{-1.5ex}
\begin{proofEnd}
From the result of \cref{thm:strongly_convex_reshuffling_sgd_convergence}, we can invoke \cref{thm:geometric_complexity_squared} with 
\begin{alignat*}{3}
  A &= \frac{L_{\mathrm{max}} n}{4 b^2 \mu} \tau^2,\quad
  \\
  B &= \frac{2}{\mu}\left( \frac{n-b}{(n-1) b} \left(\frac{1}{n} \sum^n_{i=1} \sigma_{i}^2\right) + \frac{n \left(b-1\right)}{\left(n-1\right) b} {\left(\frac{1}{n} \sum^n_{i=1} \sigma_{i}\right)}^2 \right),
  \\
  C &= \mathcal{L}_{\mathrm{max}} + L_{\mathrm{max}}.
\end{alignat*}
Then, an \(\epsilon\) accurate solution in expectation can be obtained after
\begin{align*}
   T 
   &\geq 
   \max\Bigg(
      \underbrace{\frac{2 B}{\mu}}_{\triangleq C_1}  \frac{1}{\epsilon}
      + \underbrace{\frac{\sqrt{2 A}}{\mu}}_{\triangleq C_2} \frac{1}{\sqrt{\epsilon}}, \;
       \frac{\mathcal{L}_{\mathrm{max}} + L_{\mathrm{max}}}{\mu}
   \Bigg)
   \log \left(2 r_0^2 \frac{1}{\epsilon} \right)
\end{align*}
iterations with a stepsize of
\begin{align*}
    \gamma = \min\left( \frac{- B + \sqrt{B^2 + 2 A \epsilon }}{2 A}, \frac{1}{C} \right).
\end{align*}
To make the iteration complexity more precise, the terms \(C_1, C_2\) can be organized as
\begin{align*}
    C_1
    &=
    \frac{2 B}{\mu}
    =
    \frac{2}{\mu} \Bigg( \, \frac{2}{\mu} \, \Bigg\{\, \frac{n-b}{(n-1) b} \left(\frac{1}{n} \sum^n_{i=1} \sigma_{i}^2\right) 
    \\
    &\qquad\qquad\qquad+ \frac{n \left(b-1\right)}{\left(n-1\right) b} {\left(\frac{1}{n} \sum^n_{i=1} \sigma_{i}\right)}^2  \,\Bigg\} \,\Bigg) 
    \\
    &=
    \frac{4}{\mu^2}  \Bigg( \frac{n-b}{(n-1) b} \left(\frac{1}{n} \sum^n_{i=1} \sigma_{i}^2\right) 
    + \frac{n \left(b-1\right)}{\left(n-1\right) b} \, {\left(\frac{1}{n} \sum^n_{i=1} \sigma_{i}\right)}^2  \Bigg)
    \\
    C_2
    &=
    \frac{\sqrt{2 A}}{\mu}
    \\
    &=
    \sqrt{2 \frac{L_{\mathrm{max}} n}{4 b^2 \mu} \tau^2 } \frac{1}{\mu^2}
    \\
    &=
    \frac{\sqrt{L_{\mathrm{max}}} \, \tau \, \sqrt{n}}{\sqrt{2} b \mu^{\nicefrac{3}{2}}} 
    \\
    &\leq
    \frac{\sqrt{L_{\mathrm{max}}}}{\mu^{\nicefrac{3}{2}}} \frac{ \sqrt{n}}{b} \, \tau.
\end{align*}
Applying the fact that \( (n-b)/n \leq (n-1)/n \leq 2 \) for all \( n \geq 2\) yields the simplified constants in the statement.
\end{proofEnd}

%% file: section_evaluation.tex
\paragraph{Setup}
We evaluate the insight on the tradeoff between \(b\) and \(m\) for correlated estimators on a synthetic problem.
In particular, we set
{%
\setlength{\abovedisplayskip}{1ex} \setlength{\abovedisplayshortskip}{1ex}
\setlength{\belowdisplayskip}{1ex} \setlength{\belowdisplayshortskip}{1ex}
\[
  f_i\left(\vx; \rvveta\right) = \frac{L_i}{2} \norm{ \vx - \vx_i^* + \rvveta }_2^2,
\]
}%
where the smoothness constants \(L_i \sim \text{Inv-Gamma}(\nicefrac{1}{2}, \nicefrac{1}{2})\) and the stationary points \(\vx_i^* \sim \mathcal{N}(\boldupright{0}_d, s^2\boldupright{I}_d)\) are sampled randomly, where \(\boldupright{0}_d\) is a vector of \(d\) zeros and \(\boldupright{I}_d\) is a \(d \times d\) identity matrix.
Then, we compute the gradient variance on the global optimum, corresponding to computing the BV (\cref{assumption:bounded_variance}) constant.
Note that \(s^2\) here corresponds to the ``heterogeneity'' of the data.
We make the estimators dependent by sharing \(\rvveta_1, \ldots, \rvveta_m\) across the batch.

\vspace{-1.5ex}
\paragraph{Results}
The results are shown in \cref{fig:simulation}.
At low heterogeneity, there exists a ``sweet spot'' between \(m\) and \(b\).
However, this sweet spot moves towards large values of \(b\), where, at high heterogeneity levels, the largest values of \(b\) are more favorable.
Especially in the low budget regime where \(m b \ll n\), the largest \(b\) values appear to achieve the lowest variance.
This confirms our theoretical results that a large \(b\) should be preferred on challenging (large number of datapoints, high heterogeneity) problems.


%% file: section_discussions.tex
\subsection{Applications}
\vspace{-0.5ex}
In \cref{section:applications}, we establish \cref{assumption:montecarlo_er,assumption:bounded_variance_both} on the following applications:
\begin{itemize}
    \vspace{-2ex}
    \setlength\itemsep{.2ex}
    \item \textbf{ERM with Randomized Smoothing}: 
    In this problem, we consider ERM, where the model weights are perturbed by noise. 
    This variant of ERM has recently gathered interest as it is believed to improve generalization performance~\citep{orvieto_explicit_2023,liu_noisy_2021a}.
    In \cref{section:erm_smoothing}, we establish \cref{assumption:montecarlo_er}~\((\rm{A}^{\rm{ITP}})\) under the interpolation assumption.
    
    \item \textbf{Reparameterization Gradient}: 
    In certain applications, \textit{e.g.}, variational inference, generative modeling, and reinforcement learning (see \citealp[\S 5]{mohamed_monte_2020}), the optimization problem is over the parameters of some distribution, which is taken expectation over.
    Among gradient estimators for this problem, the reparameterization gradient is widely used due to lower variance~\citep{xu_variance_2019}.
    For this, in \cref{section:reparam_gradient}, we establish \cref{assumption:montecarlo_er} (\(\rm{A}^{\rm{CVX}}\)) and (B) by assuming a convexity and smooth integrand.
\vspace{-2ex}
\end{itemize}

\vspace{-0.5ex}
\subsection{Related Works}\label{section:related}
\vspace{-0.5ex}
Unlike SGD in the finite sum setting, doubly SGD has received little interest.
Previously, \citet{bietti_stochastic_2017,zheng_lightweight_2018,kulunchakov_estimate_2020} have studied the convergence of variance-reduced gradients~\citep{gower_variancereduced_2020} specific to the doubly stochastic setting under the uniform Lipchitz integrand assumption (\(\vg_i(\cdot; \veta)\) is \(L\)-Lipschitz for all \(\veta\)).
Although this assumption has often been used in the stochastic optimization literature ~\citep{nemirovski_robust_2009,moulines_nonasymptotic_2011,shalev-shwartz_stochastic_2009,nguyen_sgd_2018}, it is easily shown to be restrictive: for some \(L\)-smooth \(\widehat{f_i}\left(\vx\right)\),  \(\nabla f_i\left(\vx; \rvveta\right) = \nabla \widehat{f}_i\left(\vx\right) + x_1 \rvveta \) is not \(L\)-Lipschitz unless the support of \(\veta\) is compact.
In contrast, we established results under weaker conditions. 
We also provide a discussion on the relationships of different conditions in \cref{section:gradient_variance_conditions}.


Furthermore, we extended doubly SGD to the case where random reshuffling is used in place of sampling independent batches.
In the finite-sum setting, the fact that SGD-RR converges faster than independent subsampling (SGD) has been empirically known for a long time~\citep{bottou_curiously_2009}.
While \citet{gurbuzbalaban_why_2021} first demonstrated that SGD-RR can be fast for quadratics, a proof under general conditions was demonstrated recently~\citep{haochen_random_2019}: In the strongly convex setting, \citet{mishchenko_random_2020}
\citet{ahn_sgd_2020,nguyen_unified_2021} establish a \(\mathcal{O}\left(1/\sqrt{\epsilon}\right)\) complexity to be \(\epsilon\)-accurate, which is tight in terms of asymptotic complexity~\citep{safran_how_2020,cha_tighter_2023,safran_random_2021}.

Lastly, \citet{dai_scalable_2014,xie_scale_2015,shi_triply_2021} provided convergence guarantees for doubly SGD for ERM of random feature kernel machines.
However, these analyses are based on concentration arguments that doubly SGD does not deviate too much from the optimization path of finite-sum SGD.
Unfortunately, concentration arguments require stronger assumptions on the noise, and their analysis is application-specific.
In contrast, we provide a general analysis under the general ER assumption.

\vspace{-.5ex}
\subsection{Conclusions}
\vspace{-.5ex}
In this work, we analyzed the convergence of SGD with doubly stochastic and dependent gradient estimators.
In particular, we showed that if the gradient estimator of each component satisfies the ER and BV conditions, the doubly stochastic estimator also satisfies both conditions; this implies convergence of doubly SGD.

\vspace{-2ex}
\paragraph{Practical Recommendations}
An unusual conclusion of our analysis is that when Monte Carlo is used with minibatch subsampling, it is generally more beneficial to increase the minibatch size \(b\) instead of the number of Monte Carlo samples \(m\).
That is, for both SGD and SGD-RR, increasing \(b\) decreases the variance in a rate close to \(1/b\) when 
\begin{enumerate*}[label=\textbf{(\roman*)}]
    \item the gradient variance of the component gradient estimators varies greatly such that 
      \({\left( \frac{1}{n} \sum^n_{i=1} \sigma_{i} \right)}^2
      \ll
      \frac{1}{n} \sum^n_{i=1} \sigma_{i}^2\)
    or when 
    \item the estimators are independent as \(\rho = 0\).
\end{enumerate*}
Surprisingly, such a benefit persists even in the interpolation regime \(\tau^2 = 0\).
On the contrary, when the estimators are both dependent \textit{and}
have similar variance, it is necessary to increase both \(m\) and \(b\), where a sweet spot between the two exists.
However, such a regime is unlikely to occur in practice; in statistics and machine learning applications, the variance of the gradient estimators tends to vary greatly due to the heterogeneity of data.

%% file: section_relationships_conditions.tex
\newcommand{\dashedarrow}{\raisebox{.5ex}{\tikz{\draw[-latex,dashed,thick](0,0) -- (2em,0);}}}
\newcommand{\dottedarrow}{\raisebox{.5ex}{\tikz{\draw[-latex,dotted,thick](0,0) -- (2em,0);}}}

\begin{figure}[t]
    \centering
\begin{tikzpicture}[
  every text node part/.style={align=center},
  every node/.style={node distance=3em},
]

\node (er)  [draw, fill=blue!20, rounded corners] {ER};
\node (abc)  [draw, fill=gray!20, rounded corners, below=of er] {ABC};
\node (es) [draw, fill=gray!20, rounded corners, above=of er]  {ES};
\node (wg) [draw, fill=gray!20, rounded corners, above=of es] {WG};
\node (sg) [draw, fill=gray!20, rounded corners, above=of wg]  {SG + \(f\) is smooth};

\node (cer) [draw, fill=gray!20, rounded corners, below right=3ex and 4em of er]  {CER};
\node (ces) [draw, fill=gray!20, rounded corners, above=of cer]  {CES};

\node (qes)    [draw, fill=gray!20, rounded corners, above right=2ex and 8em of es]  {QES};
\node (qv)     [draw, fill=gray!20, rounded corners, above=of qes]  {QV with \(\beta = 0\)};

\node (us) [draw, fill=gray!20, rounded corners, below right=3ex and 2em of qes]  {\(f\left(\vx; \veta\right)\) is\\uniformly\\smooth};

\draw[-latex, thick] (sg) -- (wg) node[pos=0.5, left] {(1)};
\draw[-latex, thick] (sg) -- (qv)  node[pos=0.6, above] {(2)};
\draw[-latex, dashed, thick] (qv) -- (wg) node[pos=0.5, below] {(3)};
\draw[-latex, thick] (qv) -- (qes) node[pos=0.5, left] {(4)};

\draw[-latex, thick] (wg) -- (es)  node[pos=0.5, left] {(5)};

\draw[-latex, dashed, thick] (qes) -- (es) node[pos=0.5,above] {(6)};

\draw[-latex, thick] (us) -- (qes) node[pos=0.3, above] {(7)};
\draw[-latex, dotted, thick] (us) -- (ces) node[pos=0.5, below] {(8)};

\draw[-latex, thick] (es) -- (er)  node[pos=0.5, left] {(9)};

\draw[-latex, thick] (ces) -- (es) node[pos=0.5, below] {(10)};
\draw[-latex, thick] (ces) -- (cer) node[pos=0.5, left] {(11)};

\draw[-latex, thick] (cer) -- (er) node[pos=0.5, below] {(12)};

\draw[-latex, thick] (er) -- (abc) node[pos=0.5, left] {(13)};

\end{tikzpicture}
    \caption{
    \textbf{Implications between general gradient variance conditions for some unbiased estimator \(\rvvg\left(\vx\right) = \nabla f \left(\vx; \rvveta\right)\) of \(\nabla f\left(\vx\right) = \mathbb{E} \rvvg\left(\vx\right)\)}. 
    The dashed arrows (\protect\dashedarrow) hold if \(f\) is further assumed to be QFG; the dotted arrow (\protect\dottedarrow) holds if the integrand \(f(\vx; \veta)\) is uniformly convex such that it is convex with respect to \(\vx\) for any fixed \(\veta\).
    (1), (5), (9), (13) are established by \citet[Theorem 3.4]{gower_sgd_2021};
    (2) is proven in \cref{thm:sgisqv};
    (3) is proven in \cref{thm:qviswg};
    (4) is proven in \cref{thm:qvisqes};
    (7) is proven in \cref{thm:usqes};
    (8) is proven in \cref{thm:usces};
    (6) is proven by \citet[Lemma 2]{nguyen_sgd_2018} but we restate the proof in \cref{thm:qesises};
    (11) is proven in \cref{thm:cescer};
    (10), (12) hold trivially if \(\vx_* \in \argmin_{\vx \in \mathcal{X}} f\left(\vx\right)\) are all stationary points.
    }\label{fig:conditions_relationships}
\end{figure}

In this section, we will discuss some additional aspects of the ER and ES conditions introduced in \cref{section:conditions}.
We will also look into alternative gradient variance conditions that have been proposed in the literature and their relationship with the ER condition.

\subsection{Definitions}
For this section, we will use the following additional definitions:
\begin{definition}[Quadratic Functional Growth; QFG]
We say \(f : \mathcal{X} \to \mathbb{R}\) satisfies \(\mu\)-quadratic functional growth if there exists some \(\mu > 0\)  such that
\[
    \frac{\mu}{2}\norm{\vx - \vx_*}_2^2 \leq f\left(\vx\right) - f\left(\vx_*\right)
\]
holds for all \(\vx \in \mathcal{X}\), where \(\vx_* \in \argmin_{\vx \in \mathcal{X}} f\left(\vx\right)\).
\end{definition}
This condition implies that \(f\) grows at least as fast as some quadratic and is weaker than the Polyak-\L{}ojasiewicz.
However, for any convex function \(f\) that satisfies this condition also means that \(f\) is \(\mu\)-strongly convex~\citep{karimi_linear_2016}.

\begin{definition}[Uniform Smoothness]
For the unbiased estimator \(\rvvg\left(\vx\right) = \nabla f\left(\vx; \rvveta\right) \) of \(\nabla F(\vx) = \mathbb{E} \nabla f\left(\vx; \rvveta\right) = \mathbb{E} \nabla f\left(\vx; \rvveta\right)\), we say the integrand \(\nabla f_i\left(\vx; \veta\right)\) satisfies uniform \(L\)-smoothness if there exist some \(L < \infty\) such that, for any fixed \(\veta\),
\[
  \norm{\nabla f\left(\vx; \veta\right) - \nabla f\left(\vx'; \veta\right)}_2 \leq L \norm{\vx - \vx'}_2
\]
holds for all \((\vx, \vx') \in \mathcal{X}^2\) simultaneously.
\end{definition}
As discussed in \cref{section:introduction,section:related}, this condition is rather strong: it does not hold for multiplicative noise unless the support is bounded.

\begin{definition}[Uniform Convexity]
For the unbiased estimator \(\rvvg\left(\vx\right) = \nabla f\left(\vx; \rvveta\right) \) of \(\nabla F(\vx) = \mathbb{E}\nabla f\left(\vx; \rvveta\right)\), we say the integrand \(f\left(\vx; \veta\right)\) is uniformly convex if it is convex for any \(\veta\) such that, for any fixed \(\veta\), 
\[
  f\left(\vx; \veta\right) - f\left(\vx'; \veta\right)
  \leq
  \inner{\nabla f\left(\vx; \veta\right)}{\vx - \vx'}
\]
holds for all \((\vx, \vx') \in \mathcal{X}^2\) simultaneously.
\end{definition}

\subsection{Additional Gradient Variance Conditions}\label{section:additional_conditions}
For some estimator \(\rvvg\) of \(\nabla f\), the following conditions have been considered in the literature:
\begin{itemize}
    \setlength\itemsep{-1ex}
    \item Strong growth condition (SG):
\begin{align*}
    \mathbb{E}\norm{\rvvg\left(\vx\right)}_2^2 \leq \rho \norm{\nabla f\left(\vx\right) }_2^2
\end{align*}

    \item Weak growth condition (WG):
\begin{align*}
    \mathbb{E}\norm{\rvvg\left(\vx\right)}_2^2 \leq \rho \left( f\left(\vx\right) - f\left(\vx_*\right)\right)
\end{align*}

    \item Quadratic variance (QV):
\begin{align*}
    \mathbb{E}\norm{\rvvg\left(\vx\right)}_2^2 \leq \alpha \, \norm{\vx - \vx_*}_2^2 + \beta
\end{align*}

    \item Convex expected smoothness (CES):
\begin{align*}
    \mathbb{E}\norm{\rvvg\left(\vx\right) - \rvvg\left(\vy\right)}_2^2 \leq 2 \mathcal{L} \mathrm{D}_{f}\left(\vx, \vy\right)
\end{align*}

    \item Convex expected residual (CER):
\begin{align*}
    \mathrm{tr}\V{\rvvg\left(\vx\right) - \rvvg\left(\vy\right)} \leq 2 \mathcal{L} \mathrm{D}_{f}\left(\vx, \vy\right)
\end{align*}

    \item Quadratic expected smoothness (QES):
\begin{align*}
    \mathbb{E}\norm{\rvvg\left(\vx\right) - \rvvg\left(\vy\right)}_2^2 \leq \mathcal{L}^2 \norm{\vx - \vy}_2^2
\end{align*}

    \item ABC:
\begin{align*}
    \mathbb{E}\norm{\rvvg\left(\vx\right)}_2^2 \leq A \left( f\left(\vx\right) - f\left(\vx_*\right)\right) + B \norm{\nabla f\left(\vx\right)}_2^2 + C
\end{align*}
\end{itemize}
Here, \(\vx_* \in \argmin_{\vx \in \mathcal{X}} f\left(\vx\right)\) is any stationary point of \(f\) and the stated conditions should hold for all \((\vx, \vy) \in \mathcal{X}^2\).

\newpage
\textbf{SG} was used by \citet{schmidt_fast_2013} to establish the linear convergence of SGD for strongly convex objectives, and \(\mathcal{O}(1/T)\) convergence for convex objectives; \textbf{WG} was proposed by \citet{vaswani_fast_2019} to establish similar guarantees to SG under a verifiable condition; \textbf{QV} was used to establish the non-asymptotic convergence on strongly convex functions by \citet{wright_optimization_2021}, while convergence on general convex functions was established by \citet{domke_provable_2023}, including stochastic proximal gradient descent; \textbf{QES} was used by \cite{moulines_nonasymptotic_2011} to establish one of the earliest general non-asymptotic convergence results for SGD on strongly convex objectives; \textbf{ABC} was used by \citet{khaled_better_2023} to establish convergence of SGD for non-convex smooth functions.
(See also \citet{khaled_better_2023} for a comprehensive overview of these conditions.)
The relationship of these conditions with the ER condition are summarized in \cref{fig:conditions_relationships}.

As demonstrated in \cref{fig:conditions_relationships} and discussed by \citet{khaled_better_2023}, the ABC condition is the weakest of all.
However, the convergence guarantees for problems that exclusively satisfy the ABC condition are weaker than others.
(For instance, the number of iterations \(T\) has to be fixed \textit{a priori}.)
On the other hand, the ER condition retrieves most of the strongest known guarantees for SGD; some of which were listed in \cref{section:sgd_convergence}.

\begin{figure}[t]
\begin{tikzpicture}[
  every text node part/.style={align=center},
  every node/.style={node distance=2em},
]
\node (er)          [draw, fill=blue!20, rounded corners] {\(\nabla f_{\rvB}\) is ER};
\node (es)          [draw, fill=gray!20, rounded corners, above=of er] {\(\nabla f_{\rvB}\) is ES};
\node (smoothstar)  [draw, fill=gray!20, rounded corners, above=of es] {\(f_i\) is smooth + \(\vx_*\)-convex};
\node (smoothintp)  [draw, fill=gray!20, rounded corners, above=of smoothstar] {\(f_i\) is smooth + Interp.};
\node (smoothconvex)  [draw, fill=gray!20, rounded corners, right=of smoothstar] {\(f_i\) is smooth + convex};

\draw[-latex, thick] (es) -- (er)                   node[pos=0.5,left]  {(4)};
\draw[-latex, thick] (smoothstar)   -- (es)         node[pos=0.5,left]  {(2)};
\draw[-latex, thick] (smoothintp)   -- (smoothstar) node[pos=0.5,left]  {(1)};
\draw[-latex, thick] (smoothconvex) -- (smoothstar) node[pos=0.5,below] {(3)};
\end{tikzpicture}
    \caption{\textbf{Implications of assumptions on the components \(f_1, \ldots, f_n\) to the minibatch subsampling gradient estimator \(\nabla f_{\rvB}\) of \(F = \frac{1}{n}\left(f_1 + \ldots + f_n\right)\)}.
    (1), (4) are established by \citet[Theorem 3.4]{gower_sgd_2021}, while (3) trivially follows from the fact that \(\vx_*\)-convexity is strictly weaker than (global) convexity, and (2) was established by \citet[Proposition 3.10]{gower_sgd_2019}.
    }\label{fig:component_conditions}
\end{figure}

\subsection{Establishing the ER Condition}
For subsampling estimators, it is possible to establish some of the gradient variance conditions through general assumptions on the components.
See some examples in~\cref{fig:component_conditions}.
Here, we use the following definitions:
\begin{definition}
For the finite sum objective \(F = \frac{1}{n}\left(f_1 + \ldots + f_n\right)\), we say interpolation holds if, for all \(i = 1, \ldots, n\),
\[
    f_i\left(\vx_*\right) \leq f_i\left(\vx\right),
\]
holds for all \(\vx \in \mathcal{X}\), where \(\vx_* \in \argmin_{\vx \in \mathcal{X}} F\left(\vx\right)\).
\end{definition}

\begin{definition}
For the finite sum objective \(F = \frac{1}{n}\left(f_1 + \ldots + f_n\right)\), we say the components are \(\vx_*\)-convex if, for all \(i = 1, \ldots, n\), 
\[
  f_i\left(\vx_*\right) - f_i\left(\vx\right)
  \leq
  \inner{\nabla f_i\left(\vx_*\right)}{\vx_* - \vx}
\]
holds for all \(\vx \in \mathcal{X}\), where \(\vx_* \in \argmin_{\vx \in \mathcal{X}} F\left(\vx\right)\).
\end{definition}
This assumption is a weaker version of convexity; convexity needs to hold with respect to \(\vx_*\) only. 
It is closely related to star~\citep{nesterov_cubic_2006} and quasar convexity~\citep{hinder_nearoptimal_2020,guminov_accelerated_2023}.

\subsection{Proofs of Implications in \cref{fig:conditions_relationships}}
We prove new implication results between some of the gradient variance conditions discussed in \cref{section:additional_conditions}.
In particular, the relationship between the QES and QV against other conditions has not been considered before.

\begin{proposition}\label{thm:cescer}
Let \(\rvvg\) be an unbiased estimator of \(\nabla f\). Then,
    \begin{center}
    \centering
        \begin{tikzpicture}[
            every text node part/.style={align=center},
            every node/.style={node distance=2em},
        ]
            \node (cond)  [draw, rounded corners] {\(\rvvg\) is CES};
            \node (impl)  [draw, rounded corners, right=of cond] {\(\rvvg\) is CER};
            \draw[-implies,double equal sign distance, thick] (cond) -- (impl);
        \end{tikzpicture}
    \end{center}
\end{proposition}
\begin{proof}
The result immediately follows from the fact that 
\[
  \mathrm{tr}\V{ \rvvg\left(\vx\right) - \rvvg\left(\vx'\right) } \leq \mathbb{E}\norm{ \rvvg\left(\vx\right) - \rvvg\left(\vx'\right) }_2^2
\]
holds for all \(\vx, \vx' \in \mathcal{X}\).
\end{proof}

\begin{proposition}\label{thm:sgisqv}
Let \(\rvvg\) be an unbiased estimator of \(\nabla f\). Then,
    \begin{center}
    \centering
        \begin{tikzpicture}[
            every text node part/.style={align=center},
            every node/.style={node distance=2em},
        ]
            \node (cond)  [draw, rounded corners] {\(\rvvg\) is SG\\+\\\(f\) is \(L\)-smooth};
            \node (impl)  [draw, rounded corners, right=of cond] {\(\rvvg\) is QV with \(\beta=0\)};
            \draw[-implies,double equal sign distance, thick] (cond) -- (impl);
        \end{tikzpicture}
    \end{center}
\end{proposition}
\begin{proof}
    Notice that, by definition, \(\nabla f\left(\vx_*\right) = 0\).
    Then,
    \begin{align*}
        \mathbb{E}\norm{\rvvg\left(\vx\right)}_2^2
        &\leq
        \rho \norm{\nabla f\left(\vx\right)}_2^2
        \\
        &=
        \rho \norm{\nabla f\left(\vx\right) - \nabla f\left(\vx_*\right) }_2^2,
\shortintertext{applying \(L\)-smoothness of \(f\),}
        &\leq
        L^2 \rho \norm{\vx - \vx_*}_2^2.
    \end{align*}
\end{proof}

\begin{proposition}\label{thm:usqes}
Let \(\rvvg\left(\vx\right) = \nabla f\left(\vx; \rvveta\right)\) be an unbiased estimator of \(\nabla f\left(\vx\right) = \mathbb{E} \nabla f\left(\vx; \rvveta\right)\). Then,
    \begin{center}
        \centering
        \begin{tikzpicture}[
            every text node part/.style={align=center},
            every node/.style={node distance=2em},
        ]
            \node (cond)  [draw, rounded corners] {Integrand is uniformly \(L\)-smooth};
            \node (impl)  [draw, rounded corners, right=of cond] {QES};
            \draw[-implies,double equal sign distance, thick] (cond) -- (impl);
        \end{tikzpicture}
    \end{center}
\end{proposition}
\begin{proof}
The result immediately follows from the fact that the integrand \(f\left(\vx; \veta\right)\) is \(L\)-smooth with respect to \(\vx\) uniformly over \(\veta\) as
\begin{align*}
    \mathbb{E}\norm{ \rvvg\left(\vx\right) - \rvvg\left(\vx'\right) }_2^2
    &=
    \mathbb{E}\norm{ \nabla f\left(\vx; \rvveta\right) - \nabla f\left(\vx'; \rvveta\right) }_2^2
    \\
    &\leq
    L^2 \norm{ \vx - \vx' }_2^2.
\end{align*}
\end{proof}

\begin{proposition}\label{thm:usces}
Let \(\rvvg\left(\vx\right) = \nabla f\left(\vx; \rvveta\right)\) be an unbiased estimator of \(\nabla f\left(\vx\right) = \mathbb{E} \nabla f\left(\vx; \rvveta\right)\). Then,
    \begin{center}
        \centering
        \begin{tikzpicture}[
            every text node part/.style={align=center},
            every node/.style={node distance=2em},
        ]
            \node (cond)  [draw, rounded corners] {Integrand is uniformly \(L\)-smooth\\+\\\(f\) is uniformly convex};
            \node (impl)  [draw, rounded corners, right=of cond] {ES};
            \draw[-implies,double equal sign distance, thick] (cond) -- (impl);
        \end{tikzpicture}
    \end{center}
\end{proposition}
\begin{proof}
Since the integrand \(f\left(\vx; \veta\right)\) is both uniformly smooth and convex with respect to \(\vx\) for a any fixed \(\veta\), we have
\begin{align*}
  &\norm{\nabla f\left(\vx; \veta\right) - \nabla f\left(\vx'; \veta\right)}_2
  \\
  &\qquad\leq
  2 L \left(f\left(\vx; \veta\right) - f\left(\vx'; \veta\right) - \inner{\nabla f\left(\vx'; \veta\right)}{\vx - \vx'} \right).
\end{align*}
Then, 
\begin{align*}
    &\mathbb{E}\norm{ \rvvg\left(\vx\right) - \rvvg\left(\vx_*\right) }_2^2
    \\
    &\;=
    \mathbb{E}\norm{ \nabla f\left(\vx; \rvveta\right) - \nabla f\left(\vx_*; \rvveta\right) }_2^2
    \\
    &\;\leq
    2 L \,
    \mathbb{E}\left(f\left(\vx; \veta\right) - f\left(\vx_*; \veta\right) - \inner{\nabla f\left(\vx_*; \veta\right)}{\vx - \vx'} \right)
    \\
    &\;=
    2 L \,
    \left(f\left(\vx\right) - f\left(\vx_*\right) - \inner{\nabla f\left(\vx_*\right)}{\vx - \vx'} \right)
    \\
    &\;=
    2 L \, \left( f\left(\vx\right) - f\left(\vx_*\right)\right)
\end{align*}
holds for all \(\vx \in \mathcal{X}\).
\end{proof}

\begin{proposition}\label{thm:qvisqes}
Let \(\rvvg\) be an unbiased estimator of \(\nabla F\). Then,
    \begin{center}
        \centering
        \begin{tikzpicture}[
            every text node part/.style={align=center},
            every node/.style={node distance=2em},
        ]
            \node (cond)  [draw, rounded corners] {\(\rvvg\) is QV with \(\beta=0\)};
            \node (impl)  [draw, rounded corners, right=of cond] {QES};
            \draw[-implies,double equal sign distance, thick] (cond) -- (impl);
        \end{tikzpicture}
    \end{center}
\end{proposition}
\begin{proof}
From the classic inequality \({(a + b)}^2 \leq 2 a^2 + 2 b^2\), we have
\begin{align*}
    \mathbb{E}\norm{\rvvg(\vx) - \rvvg\left(\vx_*\right)}_2^2
    \leq
    2 \, \mathbb{E}\norm{\rvvg(\vx)}_2^2
    +
    2 \, \mathbb{E}\norm{\rvvg(\vx_*)}_2^2.
\end{align*}
Now, since QV holds with \(\beta = 0\), we have \(\mathbb{E}\norm{\rvvg(\vx_*)}_2^2 = 0\).
Therefore, 
\begin{align*}
    \mathbb{E}\norm{\rvvg(\vx) - \rvvg\left(\vx_*\right)}_2^2
    \leq
    2 \, \mathbb{E}\norm{\rvvg(\vx)}_2^2
    \leq
    2 \, \alpha \norm{ \vx - \vx_* }_2^2,
\end{align*}
where we have applied QV at the last inequality.
\end{proof}

\newpage
\begin{proposition}\label{thm:qviswg}
Let \(\rvvg\) be an unbiased estimator of \(\nabla f\). Then,
    \begin{center}
        \centering
        \begin{tikzpicture}[
            every text node part/.style={align=center},
            every node/.style={node distance=2em},
        ]
            \node (cond)  [draw, rounded corners] {\(\rvvg\) is QV with \(\beta=0\)\\+\\\(f\) is \(\mu\)-QFG};
            \node (impl)  [draw, rounded corners, right=of cond] {WG};
            \draw[-implies,double equal sign distance, thick] (cond) -- (impl);
            
        \end{tikzpicture}
    \end{center}
\end{proposition}
\begin{proof}
The result immediately follows from QV as
\begin{align*}
    \mathbb{E}\norm{\rvvg(\vx)}_2^2
    &\leq
    \alpha \norm{\vx - \vx_*}_2^2,
\shortintertext{applying \(\mu\)-quadratic functional growth,} 
    &\leq
    \frac{2\alpha}{\mu} \left(f\left(\vx\right) - f\left(\vx_*\right) \right).
\end{align*}
\end{proof}

\begin{proposition}\label{thm:qesises}
Let \(\rvvg\) be an unbiased estimator of \(\nabla f\). Then,
    \begin{center}
        \centering
        \begin{tikzpicture}[
            every text node part/.style={align=center},
            every node/.style={node distance=2em},
        ]
            \node (cond)  [draw, rounded corners] {\(\rvvg\) is QES\\+\\\(f\) is \(\mu\)-QFG};
            \node (impl)  [draw, rounded corners, right=of cond] {ES};
            \draw[-implies,double equal sign distance, thick] (cond) -- (impl);
            
        \end{tikzpicture}
    \end{center}
\end{proposition}
\begin{proof}
From QV, we have
\begin{align*}
    \mathbb{E}\norm{\rvvg(\vx) - \rvvg\left(\vx_*\right)}_2^2
    &\leq
    \mathcal{L}^2 \norm{\vx - \vx_*}_2^2
\shortintertext{and \(\mu\)-quadratic functional growth yields} 
    &\leq
    \frac{2 \mathcal{L}^2}{\mu} \left(f\left(\vx\right) - f\left(\vx_*\right) \right).
\end{align*}
\end{proof}

The strategy applying QFG when proving \cref{thm:qesises,thm:qviswg} establishes the stronger variant of the ER condition: \cref{assumption:montecarlo_er} (B).
However, the price for this is that one has to pay for an excess \(\kappa = \mathcal{L} / \mu\) factor, and this strategy works only works for quadratically growing objectives.




%% file: thm_expected_variance_lemma.tex
\begin{theoremEnd}[restated before]{lemma}
expectedvariancelemma
\end{theoremEnd}
\vspace{-1.5ex}
\begin{proof}
From the formula for the variance of sums,
\begin{alignat}{2}
    &\mathrm{tr}\V{ \sum_{i=1}^b \rvvx_i \,\middle|\, \rvB }
    \nonumber
    \\
    &\;=
    \sum_{i=1}^b
    \mathrm{tr}\, \V{\rvvx_i \mid \rvB}
    +
    \sum_{i=1}^b \sum_{j \neq i}
    \mathrm{tr}\, \Cov{\rvvx_i, \rvvx_j \mid \rvB }.
    \nonumber
    \\
    &\;\leq
    \sum_{i=1}^b
    \mathrm{tr}\, \V{\rvvx_i \mid \rvB}
    +
    \sum_{i=1}^b \sum_{j \neq i}
    \rho \, \sqrt{\mathrm{tr}\V{\rvvx_i \mid \rvB}} \sqrt{\mathrm{tr}\V{\rvvx_j \mid \rvB}}
    \nonumber
    \\
    &\;=
    \left(1 - \rho\right) \sum_{i=1}^b \mathrm{tr}\, \V{\rvvx_i \mid \rvB}
    \nonumber
    \\
    &\;\quad+
    \rho
    \sum_{i=1}^b \sum_{j=1}^b
    \sqrt{\mathrm{tr}\V{\rvvx_i \mid \rvB}} \sqrt{\mathrm{tr}\V{\rvvx_j \mid \rvB}}
    \nonumber
    \\
    &\;=
    \left(1 - \rho\right) \, \sum_{i=1}^b \mathrm{tr}\, \V{\rvvx_i \mid \rvB}
    +
    \rho \,
    {\textstyle\left(
      \sum_{i=1}^b  
      \sqrt{\mathrm{tr}\V{\rvvx_i \mid \rvB}} 
    \right)}^2
    \nonumber
    \\
    &\;=
    \left(1 - \rho\right) \, \rvV
    +
    \rho \, \rvS^2.
    \nonumber
\end{alignat}
Then, it follows that
\begin{align*}
    \E{ \mathrm{tr}\V{ \sum_{i=1}^b \rvvx_i \,\middle|\, \rvB } }
    &\leq
    \rho \E{ \rvS^2 } + \left(1 -\rho\right) \E{\rvV}
    \\
    &=
    \rho \V{ \rvS } + \rho {(\mathbb{E}\rvS)}^2 + \left(1 -\rho\right) \E{\rvV},
\end{align*}
from the basic property of the variance:
\[
    \V{ \rvS } = \E{ \rvS^2 } - {(\mathbb{E}\rvS)}^2.
\]

Since \cref{assumption:correlation} is the only inequality we use, the equality in the statement holds whenever the equality in \cref{assumption:correlation} holds.
\end{proof}

%% file: section_erm_noisy_data.tex
\subsubsection{Description}
Randomized smoothing was originally considered by \citet{polyak_optimal_1990,nesterov_smooth_2005,duchi_randomized_2012} in the nonsmooth convex optimization context, where the function is ``smoothed'' through random perturbation.
This scheme has recently renewed interest in the non-convex ERM context as it has been found to improve generalization performance~\citep{orvieto_explicit_2023,liu_noisy_2021a}.
Here, we will focus on the computational aspect of this scheme.
In particular, we will see if we can obtain similar computational guarantees already established in the finite-sum ERM setting, such as those by \citet[Lemma 5.2]{gower_sgd_2021}.

Consider the canonical ERM problem, where we are given a dataset \(\mathcal{D} = {\left\{(\vx_i, y_i)\right\}}_{i = 1}^n \in {(\mathcal{X} \times \mathcal{Y})}^{n}\) and solve
\begin{align*}
  \minimize_{\vw \in \mathcal{W}}\;\; L\left(\vw\right) 
  &= 
  \frac{1}{n} \sum^n_{i=1} \ell\left(f_{\vw}\left(\vx_i\right), y_i\right) + h\left(\vw\right),
\end{align*}
where \((\vx_i, y_i) \in \mathcal{X} \times \mathcal{Y}\) are the feature and label of the \(i\)th instance, \(f_{\vw} : \mathcal{X} \to \mathcal{Y}\) is the model, \(\ell : \mathcal{Y} \times \mathcal{Y} \to \mathbb{R}_{\geq 0}\) is a non-negative loss function, and \(h : \mathcal{W} \to \mathbb{R}\) is a regularizer.

For randomized smoothing, we instead minimize 
\[
  L\left(\vw\right) = \frac{1}{n} \sum^n_{i=1} R_i\left(\vw\right),
\]
where the instance risk is defined as
\[
  r_i\left(\vw\right) = \mathbb{E}_{\rvvepsilon \sim \varphi} \ell\left( f_{\vw + \rvvepsilon}\left(\vx_i\right), y_i\right)
\]
for some noise distribution \(\vepsilon \sim \varphi\).
The goal is to obtain a solution \(\vw^* = \argmin_{\vw \in \mathcal{W}} L\left(\vw\right)\) that is robust to such perturbation.

The integrand of the gradient estimator of the instance risk is defined as
\begin{align*}
  \vg_{i}\left(\vw; \veta\right) 
  &= 
  \nabla_{\vw} \ell\left(f_{\vw + \rvvepsilon}\left(\vx_i\right), y_i\right)
  \\
  &= \frac{\partial f_{\vw + \rvvepsilon}\left(\vx_i\right)}{\partial \vw} \ell'\left(f_{\vw + \rvvepsilon}\left(\vx_i\right), y_i\right),
\end{align*}
where it is an unbiased estimate of the instance risk such that
\[
  \mathbb{E}\rvvg_{i}\left(\vw\right)
  =
  \nabla R_i\left(\vw\right).
\]
The key challenge in analyzing the convergence of SGD in the ERM setting is dealing with the Jacobian
\(
\frac{\partial f_{\vw + \rvvepsilon}\left(\vx_i\right)}{\partial \vw}.
\)
Even for simple toy models, analyzing the Jacobian without relying on strong assumptions is hard.
In this work, we will assume that it is bounded by an instance-dependent constant.

\newpage
\subsubsection{Preliminaries}
We use the following assumptions:

\begin{assumption}\label{assumption:erm}
~
\begin{enumerate}[label=(\alph*)]
    \item Let the mapping \(\hat{y} \mapsto \ell\left(\hat{y}, y\right)\) is convex and \(L\)-smooth for any \(y_i\) \(\forall i = 1, \ldots, n\).
    
    \item The Jacobian of the model with respect to its parameters for all \(i = 1, \ldots, n\) is bounded almost surely as
    \[
      \norm{ \frac{ f_{\vw + \rvvepsilon}\left(\vx_i\right) }{ 
      \partial \vw }  }_2 \leq G_i
    \]
    for all \(\vw \in \mathcal{W}\).
    
    \item Interpolation holds on the solution set such that, for all \(\vw_* \in \argmin_{\vw \in \mathcal{W}} L\left(\vw\right)\), the loss minimized as \[\ell\left(f_{\vw_* + \rvvepsilon}\left(\vx_i\right), y_i\right) = \ell'\left(f_{\vw_* + \rvvepsilon}\left(\vx_i\right), y_i\right) = 0\] for all \((\vx_{i}, y_i) \in \mathcal{D}\).
\end{enumerate}
\end{assumption}
(a) holds for the squared loss, (c) basically assumes that the model is overparameterized and there exists a set of optimal weights that are robust with respect to perturbation.
The has recently gained popularity as it qualitatively explains some of the empirical phenomenons of non-convex SGD~\citep{vaswani_fast_2019,gower_sgd_2021,ma_power_2018}.
(b) is a strong assumption but is commonly used to establish convergence guarantees of ERM~\citep{gower_sgd_2021}.

\vspace{1ex}
\begin{remark}
    Under \cref{assumption:erm} (c), \cref{assumption:bounded_variance_both} holds with arbitrarily small \(\sigma_{i}^2, \tau^2\) .
\end{remark}

\newpage
\subsubsection{Theoretical Analysis}
\vspace{2ex}
\begin{proposition}\label{eq:noisy_erm_er}
    Let \cref{assumption:erm} hold.
    Then, \cref{assumption:montecarlo_er} (\(\rm{A}^{\rm{ITP}}\)) holds.
\end{proposition}
\begin{proof}
    \begin{align*}
        &\mathbb{E}\norm{ \rvvg_i\left(\vw\right) - \rvvg_i\left(\vw_*\right) }_2^2
        \\
        &\;=
        \mathbb{E}{\bigg\lVert} \frac{\partial f_{\vw + \rvvepsilon}\left(\vx_i\right)}{\partial \vw} \ell'\left(f_{\vw + \rvvepsilon}\left(\vx_i\right), y_i\right) 
        \\
        &\quad\qquad-  
        \frac{\partial f_{\vw + \rvvepsilon}\left(\vx_i\right)}{\partial \vw}
        \ell'\left(f_{\vw_* + \rvvepsilon}\left(\vx_i\right), y_i\right) {\bigg\rVert}_2^2,
\shortintertext{from the interpolation assumption (\cref{assumption:erm} (c)),}
        &\;=
        \mathbb{E}\norm{\frac{\partial f_{\vw + \rvvepsilon}\left(\vx_i\right)}{\partial \vw} \ell'\left(f_{\vw + \rvvepsilon}\left(\vx_i\right), y_i\right) }_2^2
        \\
        &\;\leq
        \mathbb{E}\norm{\frac{\partial f_{\vw + \rvvepsilon}\left(\vx_i\right)}{\partial \vw}}_2^2 \abs{\ell'\left(f_{\vw + \rvvepsilon}\left(\vx_i\right), y_i\right) }^2,
\shortintertext{applying \cref{assumption:erm} (b),}
        &\;\leq
        G_i^2 \mathbb{E}\abs{\ell'\left(f_{\vw + \rvvepsilon}\left(\vx_i\right), y_i\right) }^2.
\shortintertext{and then the interpolation assumption (\cref{assumption:erm} (c)),}
        &\;=
        G_i^2 \mathbb{E}\abs{
        \ell'\left(f_{\vw + \rvvepsilon}\left(\vx_i\right), y_i\right)
        -
        \ell'\left(f_{\vw_* + \rvvepsilon}\left(\vx_i\right), y_i\right)
        }.
\shortintertext{From \cref{assumption:erm} (a),}
        &\;\leq
        2 L G_i^2 \mathbb{E}\Big(
          \ell\left(f_{\vw + \rvvepsilon}\left(\vx_i\right), y_i\right)
        -
        \ell\left(f_{\vw_* + \rvvepsilon}\left(\vx_i\right), y_i\right)
        \\
        &\qquad\qquad-
        \inner{
          \ell'\left(f_{\vw_* + \rvvepsilon}\left(\vx_i\right), y_i\right)
        }{
          f_{\vw + \rvvepsilon}\left(\vx_i\right) - f_{\vw_* + \rvvepsilon}\left(\vx_i\right)
        }
        \Big)
\shortintertext{and interpolation (\cref{assumption:erm} (c)),}
        &\;=
        2 L G_i^2 \left(
          \mathbb{E} \ell\left(f_{\vw + \rvvepsilon}\left(\vx_i\right), y_i\right)
          -
          \mathbb{E} \ell\left(f_{\vw_* + \rvvepsilon}\left(\vx_i\right), y_i\right)
        \right)
        \\
        &\;=
        2 L G_i^2 \left(
          R_i\left(\vw\right)
          -
          R_i\left(\vw_*\right)
        \right).
    \end{align*}
\end{proof}

\newpage
\begin{proposition}
  Let \cref{assumption:erm} hold.
  Then, \cref{assumption:subsampling_er} holds.
\end{proposition}
\begin{proof}
    \begin{alignat*}{2}
        &{\textstyle\frac{1}{n}\sum_{i=1}^n}
        \norm{ \nabla R_{i}\left(\vw\right) - \nabla R_{i}\left(\vw_*\right) }_2^2
        \\
        &\;=
        {\textstyle\frac{1}{n}\sum_{i=1}^n}
        \norm{ \mathbb{E} \rvvg_{i}\left(\vw\right) - \mathbb{E} \rvvg_{i}\left(\vw_*\right) }_2^2,
\shortintertext{and from Jensen's inequality,}        
        &\;\leq
        {\textstyle\frac{1}{n}\sum_{i=1}^n}
        \mathbb{E} \norm{ \rvvg_i\left(\vw\right) - \rvvg_i\left(\vw_*\right) }_2^2.
\shortintertext{We can now reuse \cref{eq:noisy_erm_er} as}
        &\;\leq
        \frac{2}{n}\sum_{i=1}^n 2 L G_i^2
        \left(R_i\left(\vw\right) - R_i\left(\vw_*\right)\right)
\shortintertext{and taking \(G_{\rm{max}} > G_i\) for all \(i = 1, \ldots, n\) as}
        &\;\leq
        2 L G_{\rm{max}}^2 \frac{1}{n}\sum_{i=1}^n         \left(R_i\left(\vw\right) - R_i\left(\vw_*\right)\right)
        \\
        &\;=
        2 L G_{\rm{max}}^2 \left(L\left(\vw\right) - L\left(\vw_*\right)\right).
    \end{alignat*}
\end{proof}

%% file: section_variational_inference.tex
\subsubsection{Description}
The reparameterization gradient estimator~\citep{kingma_autoencoding_2014,rezende_stochastic_2014,titsias_doubly_2014} is a gradient estimator for problems of the form of
\[
  f_i\left(\vw\right) = \mathbb{E}_{\rvvz \sim q_{\vw}} \ell_i\left(\rvvz\right),
\]
where \(\ell : \mathbb{R}^{d_{\rvvz} \to \mathbb{R}}\) is some integrand, such that the derivative is taken with respect to the parameters of the distribution\(q_{\vw}\) we are integrating over.
It was independently proposed by ~\citet{kingma_autoencoding_2014,rezende_stochastic_2014} in the context of variational expectation maximization of deep latent variable models (a setup commonly known as variational autoencoders) and by \citet{titsias_doubly_2014} for variational inference of Bayesian models.

Consider the case where the generative process of \(q_{\vw}\) can be represented as
\[
  \rvvz \sim q_{\vw} \quad\Leftrightarrow\quad \rvvz \stackrel{d}{=} \mathcal{T}_{\vw}\left(\rvvu\right); \quad \rvvu \sim \varphi,
\]
where \(\stackrel{d}{=}\) is equivalence in distribution, \(\varphi\) is some \textit{base distribution} independent of \(\vw\), and \(\mathcal{T}_{\vw}\) is a \textit{reparameterization function} measurable with respect to \(\varphi\) and differentiable with respect to all \(\vw \in \mathcal{W}\).
Then, the reparameterization gradient is given by the integrand
\begin{align*}
  \vg_i\left(\vw; \vu\right) 
  = \nabla_{\vw} \; \ell_i\left(\mathcal{T}_{\vw}\left(\vu\right)\right),
\end{align*}
which is unbiased, and often results in lower variance~\citep{kucukelbir_automatic_2017,xu_variance_2019} compared to alternatives such as the score gradient estimator. 
(See \citet{mohamed_monte_2020} for an overview of such estimators.)

The reparameterization gradient is primarily used to solve problems in the form of
\begin{align*}
    \minimize_{\vw \in \mathcal{W}}\;\; F\left(\vw\right)
    &= \sum^n_{i=1} f_i\left(\vw\right) + h\left(\vw\right) 
    \\
    &= \mathbb{E}_{\rvvz \sim q_{\vw}} \ell_i\left(\rvvz\right) + h\left(\vw\right),
\end{align*}
where \(h\) is some convex regularization term.

Previously, \citet[Theorem 6]{domke_provable_2019} established a bound on the gradient variance of the reparameterization gradient~\citep{kingma_autoencoding_2014,rezende_stochastic_2014,titsias_doubly_2014} under the doubly stochastic setting. 
This bound also incorporates more advanced subsampling strategies such as importance sampling~\citet{gower_sgd_2019,gorbunov_unified_2020,csiba_importance_2018,needell_stochastic_2016,needell_batched_2017}.
However, he did not extend the analysis to a complexity analysis of SGD and left out the effect of correlation between components.

\newpage
\subsubsection{Preliminaries}
The properties of the reparameterization gradient for when \(q_{\vw}\) is in the location-scale family were studied by \citet{domke_provable_2019}.
\begin{assumption}\label{assumption:variational_family}
    We assume the variational family \[\mathcal{Q} \triangleq \{q_{\vw} \mid \vw \in \mathcal{W}\}\] satisfies the following:
    \begin{enumerate}[label=(\alph*)]
        \item \(\mathcal{Q}\) is part of the location-scale family such that \(\mathcal{T}_{\vw}\left(\vu\right) = \mC \vu + \vm\).
        
        \item The scale matrix is positive definite such that \(\mC \succ 0\).
        
        \item \(\rvvu = (\rvu_1, \ldots, \rvu_{d_{\vz}})\) constitute of \textit{i.i.d.} components, where each component is standardized, symmetric, and finite kurtosis such that \(\mathbb{E}\rvu_i = 0\), \(\mathbb{E}\rvu_i^2 = 1\), \(\mathbb{E}\rvu_i^3 = 0\), and \(\mathbb{E}\rvu_i^4 = k_{\varphi}\), where \(k_{\varphi}\) is the kurtosis.
    \end{enumerate}
\end{assumption}
Under these conditions, \citet{domke_provable_2019} proves the following:
\begin{lemma}[\citealp{domke_provable_2019}; Theorem 3]\label{thm:reparam_bv}
    Let \cref{assumption:variational_family} hold and \(\ell_i\) be \(L_i\)-smooth.
    Then, the squared norm of the reparameterization gradient is bounded:
    \[
      \mathbb{E}\norm{\rvvg_i\left(\vw\right)}_2^2
      \leq
      \left(d + 1\right) \norm{\vm - \bar{\vz}_i}_2^2
      +
      \left(d + k_{\varphi}\right) \norm{\mC}_{\mathrm{F}}^2
    \]
    for all \(\vw = (\vm, \mC) \in \mathcal{W}\) and all stationary points of \(\ell_i\) denoted with \(\bar{\vz}_i\).
\end{lemma}

\newpage
Similarly, \citet{kim_convergence_2023} establish the QES condition as part of Lemma 3~\citep{kim_convergence_2023}.
We refine this into statement we need:
\begin{lemma}\label{thm:reparam_qes}
    Let \cref{assumption:variational_family} hold and \(\ell_i\) be \(L_i\)-smooth.
    Then, the squared norm of the reparameterization gradient is bounded:
    \[
      \mathbb{E}\norm{\rvvg_i\left(\vw\right) - \rvvg_i\left(\vw'\right)}_2^2
      \leq
      L_i^2 \left(d + k_{\varphi}\right) \norm{\vw - \bar{\vw}_i}_2^2
    \]
    for all \(\vw, \vw' \in \mathcal{W}\).
\end{lemma}
\begin{proof}
  \begin{align*}
    &\mathbb{E}\norm{ \rvvg_i\left(\vw\right) - \rvvg_i\left(\vw'\right) }
    \\
    &\;=
    \mathbb{E}\norm{ \nabla_{\vw} \ell_i\left(\mathcal{T}_{\vw}\left(\rvvu\right)\right) - \nabla_{\vw} \ell_i\left(\mathcal{T}_{\vw'}\left(\rvvu\right)\right) }_2^2
    \\
    &\;=
    \mathbb{E}\norm{ 
      \frac{\partial \mathcal{T}_{\vw} \left(\rvvu\right)}{\partial \vw} \nabla \ell_i\left(\mathcal{T}_{\vw}\left(\rvvu\right)\right) 
      - \frac{\partial \mathcal{T}_{\vw'} \left(\rvvu\right)}{\partial {\vw'}} \nabla \ell_i\left(\mathcal{T}_{\vw'}\left(\rvvz\right)\right) 
    }_2^2
    \\
    &\;=
    \mathbb{E}
    {\left(
      \nabla \ell_i\left(\mathcal{T}_{\vw}\left(\rvvu\right)\right)
      -
      \nabla \ell_i\left(\mathcal{T}_{\vw'}\left(\rvvu\right)\right)
    \right)}^{\top}
    {\left(
    \frac{\partial \mathcal{T}_{\vw} \left(\rvvu\right)}{\partial \vw}
    \right)}^{\top}
    \frac{\partial \mathcal{T}_{\vw'} \left(\rvvu\right)}{\partial \vw'}
    \\
    &\qquad
    \times
    \left(
      \nabla \ell_i\left(\mathcal{T}_{\vw}\left(\rvvu\right)\right)
      -
      \nabla \ell_i\left(\mathcal{T}_{\vw'}\left(\rvvu\right)\right)
    \right).
\shortintertext{As shown by \citet[Lemma 6]{kim_convergence_2023}, the squared Jacobian is an identity matrix scaled with a scalar-valued function independent of \(\vw\), \(J_{\mathcal{T}}\left(\vu\right) = \norm{\vu}_2^2 + 1\), such that}
    &\;=
    \mathbb{E}
    J_{\mathcal{T}}\left(\rvvu\right)
    \norm{
      \nabla \ell_i\left(\mathcal{T}_{\vw}\left(\rvvu\right)\right)
      -
      \nabla \ell_i\left(\mathcal{T}_{\vw'}\left(\rvvu\right)\right)
    }_2^2,
\shortintertext{applying the \(L_i\)-smoothness of \(\ell_i\),}
    &\;=
    L_i^2 \mathbb{E} J_{\mathcal{T}}\left(\rvvu\right)\norm{ \mathcal{T}_{\vw}\left(\rvvu\right) - \mathcal{T}_{\vw'}\left(\rvvu\right) }_2^2,
\shortintertext{and \citet[Corollary 2]{kim_convergence_2023} show that,}
    &\;\leq
    L_i^2 \left(d + k_{\varphi}\right) \norm{\vw - \vw'}_2^2.
  \end{align*}
\end{proof}

Lastly, the properties of \(\ell_i\) are known to transfer to the expectation \(f_i\) as follows:
\begin{lemma}\label{thm:energy_regular}
    Let \cref{assumption:variational_family} hold.
    Then we have the following:
    \begin{enumerate}[label=(\roman*)]
        \item Let \(\ell_i\) be \(L_i\) smooth. Then, \(f_i\) is also \(L_i\)-smooth\label{item:energy_regular1}
        
        \item Let \(\ell_i\) be convex. Then, \(f_i\) and \(F\) are also convex.\label{item:energy_regular2}
        
        \item Let \(\ell_i\) be \(\mu\)-strongly convex. Then, \(f_i\) and \(F\) are also \(\mu\)-strongly convex.\label{item:energy_regular3}
    \end{enumerate}
\end{lemma}
\begin{proof}
    \labelcref{item:energy_regular1} is proven by \citet[Theorem 1]{domke_provable_2020}, while a more general result is provided by \citet[Theorem 1]{kim_convergence_2023};
    \labelcref{item:energy_regular2,item:energy_regular3} are proven by \citet[Theorem 9]{domke_provable_2020} and follow from the fact that \(h\) is convex.
\end{proof}

\newpage
\subsubsection{Theoretical Analysis}
We now conclude that the reparameterization gradient fits the framework of  this work:
\begin{proposition}
    Let \cref{assumption:variational_family} hold and \(\ell_i\) be convex and \(L_i\)-smooth.
    Then, \cref{assumption:subsampling_er} holds.
\end{proposition}
\begin{proof}
    The result follows from combining \cref{thm:energy_regular} and \cref{thm:expected_residual_without_replacement}.
\end{proof}

From \cref{thm:reparam_bv}, we satisfy \ref{assumption:bounded_variance_both}.
\begin{proposition}
    Let \cref{assumption:variational_family} hold, \(\ell_i\) be \(L_i\)-smooth, the solutions \(\vw_* \in \argmin_{\vw \in \mathcal{W}} F\left(\vw\right)\) and the stationary points of \(\ell_i\), \(\bar{\vz}\), be bounded such that \(\norm{\vw_*}_2 < \infty\) and \(\norm{\bar{\vz}}_2 < \infty\).
    Then, \cref{assumption:bounded_variance_both} holds.
\end{proposition}
\begin{proof}
    \cref{thm:reparam_bv} implies that, as long as the \(\vw_*\) and \(\bar{\vz}\) are bounded, we satisfy the component gradient estimator part of \cref{assumption:bounded_variance_both}, where the constant is given as
    \[
      \sigma_i^2 = L_i^2 \left(d + 1\right) \norm{\vm_* - \bar{\vz}_i}_2^2 + L_i^2 \left(d+k_{\varphi}\right) \norm{\mC_*}_{\mathrm{F}}^2,
    \]
    where \(\vw_* = (\vm_*, \mC_*)\).
\end{proof}

From \cref{thm:reparam_qes}, we can conclude that the reparameterization gradient satisfies \cref{assumption:montecarlo_er}:
\begin{proposition}
  Let \cref{assumption:variational_family} hold and \(\ell_i\) be \(L_i\)-smooth and \(\mu\)-strongly convex.
  Then, \cref{assumption:montecarlo_er}~(\(\rm{A}^{\rm{CVX}}\)) and \cref{assumption:montecarlo_er}~(B) hold.
\end{proposition}
\begin{proof}
    Notice the following:
    \begin{enumerate}
        \item \cref{assumption:estimator_correlations} always holds for \(\rho = 1\).
        \item From the stated conditions, \cref{thm:energy_regular} establishes that both \(f_i\) and \(F\) are \(\mu\)-strongly convex.
        \item \(\mu\)-strong convexity of \(f\) and \(F\) implies that both are \(\mu\)-QFG~\citep[Appendix A]{karimi_linear_2016}.
        \item The reparameterization gradient satisfies the QES condition by  \cref{thm:reparam_qes}.
    \end{enumerate}
    Item 1, 2 and 3 combined imply the ES condition by \cref{thm:qesises}, which immediately implies the ER condition with the same constant.
    Therefore, we satisfy both \cref{assumption:montecarlo_er}~(\(\rm{A}^{\rm{CVX}}\)), \cref{assumption:montecarlo_er}~(B) where the ER constant \(\mathcal{L}_i\) is given as
    \[
        \mathcal{L}_i =  \frac{L_i^2}{\mu} \left(d + k_{\varphi}\right).
    \]
\end{proof}

%% file: main.bbl
\begin{thebibliography}{67}
\providecommand{\natexlab}[1]{#1}
\providecommand{\url}[1]{\texttt{#1}}
\expandafter\ifx\csname urlstyle\endcsname\relax
  \providecommand{\doi}[1]{doi: #1}\else
  \providecommand{\doi}{doi: \begingroup \urlstyle{rm}\Url}\fi

\bibitem[Ahn et~al.(2020)Ahn, Yun, and Sra]{ahn_sgd_2020}
Ahn, K., Yun, C., and Sra, S.
\newblock {{SGD}} with shuffling: Optimal rates without component convexity and
  large epoch requirements.
\newblock In \emph{Advances in {{Neural Information Processing Systems}}},
  volume~33, pp.\  17526--17535. {Curran Associates, Inc.}, 2020.

\bibitem[Bassily et~al.(2014)Bassily, Smith, and
  Thakurta]{bassily_private_2014}
Bassily, R., Smith, A., and Thakurta, A.
\newblock Private empirical risk minimization: Efficient algorithms and tight
  error bounds.
\newblock In \emph{Proceedings of the {{IEEE Annual Symposium}} on
  {{Foundations}} of {{Computer Science}}}, {{FOCS}} '14, pp.\  464--473,
  {USA}, October 2014. {IEEE Computer Society}.

\bibitem[Bietti \& Mairal(2017)Bietti and Mairal]{bietti_stochastic_2017}
Bietti, A. and Mairal, J.
\newblock Stochastic optimization with variance reduction for infinite datasets
  with finite sum structure.
\newblock In \emph{Advances in {{Neural Information Processing Systems}}},
  volume~30, pp.\  1623--1633. {Curran Associates, Inc.}, 2017.

\bibitem[Bottou(1999)]{bottou_online_1999}
Bottou, L.
\newblock On-line learning and stochastic approximations.
\newblock In \emph{On-{{Line Learning}} in {{Neural Networks}}}, pp.\  9--42.
  {Cambridge University Press}, 1 edition, January 1999.

\bibitem[Bottou(2009)]{bottou_curiously_2009}
Bottou, L.
\newblock Curiously fast convergence of some stochastic gradient descent
  algorithms.
\newblock Unpublished open problem at the {{International Symposium}} on
  {{Statistical Learning}} and {{Data Sciences}} ({{SLDS}}), 2009.

\bibitem[Bottou et~al.(2018)Bottou, Curtis, and
  Nocedal]{bottou_optimization_2018}
Bottou, L., Curtis, F.~E., and Nocedal, J.
\newblock Optimization methods for large-scale machine learning.
\newblock \emph{SIAM Review}, 60\penalty0 (2):\penalty0 223--311, January 2018.

\bibitem[Cha et~al.(2023)Cha, Lee, and Yun]{cha_tighter_2023}
Cha, J., Lee, J., and Yun, C.
\newblock Tighter lower bounds for shuffling {{SGD}}: Random permutations and
  beyond.
\newblock In \emph{Proceedings of the {{International Conference}} on {{Machine
  Learning}}}, volume 202 of \emph{{{PMLR}}}, pp.\  3855--3912. {JMLR}, July
  2023.

\bibitem[Csiba \& Richt{\'a}rik(2018)Csiba and
  Richt{\'a}rik]{csiba_importance_2018}
Csiba, D. and Richt{\'a}rik, P.
\newblock Importance sampling for minibatches.
\newblock \emph{Journal of Machine Learning Research}, 19\penalty0
  (27):\penalty0 1--21, 2018.

\bibitem[Dai et~al.(2014)Dai, Xie, He, Liang, Raj, Balcan, and
  Song]{dai_scalable_2014}
Dai, B., Xie, B., He, N., Liang, Y., Raj, A., Balcan, M.-F.~F., and Song, L.
\newblock Scalable kernel methods via doubly stochastic gradients.
\newblock In \emph{Advances in {{Neural Information Processing Systems}}},
  volume~27, pp.\  3041--3049. {Curran Associates, Inc.}, 2014.

\bibitem[Domke(2019)]{domke_provable_2019}
Domke, J.
\newblock Provable gradient variance guarantees for black-box variational
  inference.
\newblock In \emph{Advances in {{Neural Information Processing Systems}}},
  volume~32, pp.\  329--338. {Curran Associates, Inc.}, 2019.

\bibitem[Domke(2020)]{domke_provable_2020}
Domke, J.
\newblock Provable smoothness guarantees for black-box variational inference.
\newblock In \emph{Proceedings of the International Conference on Machine
  Learning}, volume 119 of \emph{{{PMLR}}}, pp.\  2587--2596. {JMLR}, July
  2020.

\bibitem[Domke et~al.(2023)Domke, Gower, and Garrigos]{domke_provable_2023}
Domke, J., Gower, R., and Garrigos, G.
\newblock Provable convergence guarantees for black-box variational inference.
\newblock In \emph{Advances in Neural Information Processing Systems},
  volume~36, pp.\  66289--66327. {Curran Associates, Inc.}, 2023.

\bibitem[Duchi et~al.(2012)Duchi, Bartlett, and
  Wainwright]{duchi_randomized_2012}
Duchi, J.~C., Bartlett, P.~L., and Wainwright, M.~J.
\newblock Randomized smoothing for stochastic optimization.
\newblock \emph{SIAM Journal on Optimization}, 22\penalty0 (2):\penalty0
  674--701, January 2012.

\bibitem[Garrigos \& Gower(2023)Garrigos and Gower]{garrigos_handbook_2023}
Garrigos, G. and Gower, R.~M.
\newblock Handbook of convergence theorems for (stochastic) gradient methods.
\newblock Preprint arXiv:2301.11235, {arXiv}, February 2023.

\bibitem[Gorbunov et~al.(2020)Gorbunov, Hanzely, and
  Richtarik]{gorbunov_unified_2020}
Gorbunov, E., Hanzely, F., and Richtarik, P.
\newblock A unified theory of {{SGD}}: {{Variance}} reduction, sampling,
  quantization and coordinate descent.
\newblock In \emph{Proceedings of the {{International Conference}} on
  {{Artificial Intelligence}} and {{Statistics}}}, volume 108 of
  \emph{{{PMLR}}}, pp.\  680--690. {JMLR}, June 2020.

\bibitem[Gower et~al.(2021{\natexlab{a}})Gower, Sebbouh, and
  Loizou]{gower_sgd_2021}
Gower, R., Sebbouh, O., and Loizou, N.
\newblock {{SGD}} for structured nonconvex functions: {{Learning}} rates,
  minibatching and interpolation.
\newblock In \emph{Proceedings of the {{International Conference}} on
  {{Artificial Intelligence}} and {{Statistics}}}, volume 130 of
  \emph{{{PMLR}}}, pp.\  1315--1323. {JMLR}, March 2021{\natexlab{a}}.

\bibitem[Gower et~al.(2019)Gower, Loizou, Qian, Sailanbayev, Shulgin, and
  Richt{\'a}rik]{gower_sgd_2019}
Gower, R.~M., Loizou, N., Qian, X., Sailanbayev, A., Shulgin, E., and
  Richt{\'a}rik, P.
\newblock {{SGD}}: {{General}} analysis and improved rates.
\newblock In \emph{Proceedings of the International Conference on Machine
  Learning}, volume~97 of \emph{{{PMLR}}}, pp.\  5200--5209. {JMLR}, June 2019.

\bibitem[Gower et~al.(2020)Gower, Schmidt, Bach, and
  Richtarik]{gower_variancereduced_2020}
Gower, R.~M., Schmidt, M., Bach, F., and Richtarik, P.
\newblock Variance-reduced methods for machine learning.
\newblock \emph{Proceedings of the IEEE}, 108\penalty0 (11):\penalty0
  1968--1983, November 2020.

\bibitem[Gower et~al.(2021{\natexlab{b}})Gower, Richt{\'a}rik, and
  Bach]{gower_stochastic_2021}
Gower, R.~M., Richt{\'a}rik, P., and Bach, F.
\newblock Stochastic quasi-gradient methods: {{Variance}} reduction via
  {{Jacobian}} sketching.
\newblock \emph{Mathematical Programming}, 188\penalty0 (1):\penalty0 135--192,
  July 2021{\natexlab{b}}.

\bibitem[Guminov et~al.(2023)Guminov, Gasnikov, and
  Kuruzov]{guminov_accelerated_2023}
Guminov, S., Gasnikov, A., and Kuruzov, I.
\newblock Accelerated methods for weakly-quasi-convex optimization problems.
\newblock \emph{Computational Management Science}, 20\penalty0 (1):\penalty0
  36, December 2023.

\bibitem[G{\"u}rb{\"u}zbalaban et~al.(2021)G{\"u}rb{\"u}zbalaban, Ozdaglar, and
  Parrilo]{gurbuzbalaban_why_2021}
G{\"u}rb{\"u}zbalaban, M., Ozdaglar, A., and Parrilo, P.~A.
\newblock Why random reshuffling beats stochastic gradient descent.
\newblock \emph{Mathematical Programming}, 186\penalty0 (1):\penalty0 49--84,
  March 2021.

\bibitem[Haochen \& Sra(2019)Haochen and Sra]{haochen_random_2019}
Haochen, J. and Sra, S.
\newblock Random shuffling beats {{SGD}} after finite epochs.
\newblock In \emph{Proceedings of the {{International Conference}} on {{Machine
  Learning}}}, volume~97 of \emph{{{PMLR}}}, pp.\  2624--2633. {JMLR}, May
  2019.

\bibitem[Hinder et~al.(2020)Hinder, Sidford, and
  Sohoni]{hinder_nearoptimal_2020}
Hinder, O., Sidford, A., and Sohoni, N.
\newblock Near-optimal methods for minimizing star-convex functions and beyond.
\newblock In \emph{Proceedings of {{Conference}} on {{Learning Theory}}},
  volume 125 of \emph{{{PMLR}}}, pp.\  1894--1938. {JMLR}, July 2020.

\bibitem[Ho et~al.(2020)Ho, Jain, and Abbeel]{ho_denoising_2020}
Ho, J., Jain, A., and Abbeel, P.
\newblock Denoising diffusion probabilistic models.
\newblock In \emph{Advances in {{Neural Information Processing Systems}}},
  volume~33, pp.\  6840--6851. {Curran Associates, Inc.}, 2020.

\bibitem[Johnson \& Zhang(2013)Johnson and Zhang]{johnson_accelerating_2013}
Johnson, R. and Zhang, T.
\newblock Accelerating stochastic gradient descent using predictive variance
  reduction.
\newblock In \emph{Advances in {{Neural Information Processing Systems}}},
  volume~26, pp.\  315--323. {Curran Associates, Inc.}, 2013.

\bibitem[Karimi et~al.(2016)Karimi, Nutini, and Schmidt]{karimi_linear_2016}
Karimi, H., Nutini, J., and Schmidt, M.
\newblock Linear convergence of gradient and proximal-gradient methods under
  the {{Polyak-{\L}ojasiewicz}} condition.
\newblock In \emph{Machine {{Learning}} and {{Knowledge Discovery}} in
  {{Databases}}}, Lecture {{Notes}} in {{Computer Science}}, pp.\  795--811,
  {Cham}, 2016. {Springer International Publishing}.

\bibitem[Khaled \& Richt{\'a}rik(2023)Khaled and
  Richt{\'a}rik]{khaled_better_2023}
Khaled, A. and Richt{\'a}rik, P.
\newblock Better theory for {{SGD}} in the nonconvex world.
\newblock \emph{Transactions of Machine Learning Research}, 2023.

\bibitem[Kim et~al.(2023)Kim, Oh, Wu, Ma, and Gardner]{kim_convergence_2023}
Kim, K., Oh, J., Wu, K., Ma, Y., and Gardner, J.~R.
\newblock On the convergence of black-box variational inference.
\newblock In \emph{Advances in {{Neural Information Processing Systems}}},
  volume~36, pp.\  44615--44657, {New Orleans, LA, USA}, December 2023. {Curran
  Associates Inc.}

\bibitem[Kingma \& Welling(2014)Kingma and Welling]{kingma_autoencoding_2014}
Kingma, D.~P. and Welling, M.
\newblock Auto-encoding variational {{Bayes}}.
\newblock In \emph{Proceedings of the {{International Conference}} on
  {{Learning Representations}}}, {Banff, AB, Canada}, April 2014.

\bibitem[Kingma et~al.(2015)Kingma, Salimans, and
  Welling]{kingma_variational_2015}
Kingma, D.~P., Salimans, T., and Welling, M.
\newblock Variational dropout and the local reparameterization trick.
\newblock In \emph{Advances in {{Neural Information Processing Systems}}},
  volume~28, pp.\  2575--2583. {Curran Associates, Inc.}, 2015.

\bibitem[Kucukelbir et~al.(2017)Kucukelbir, Tran, Ranganath, Gelman, and
  Blei]{kucukelbir_automatic_2017}
Kucukelbir, A., Tran, D., Ranganath, R., Gelman, A., and Blei, D.~M.
\newblock Automatic differentiation variational inference.
\newblock \emph{Journal of Machine Learning Research}, 18\penalty0
  (14):\penalty0 1--45, 2017.

\bibitem[Kulunchakov \& Mairal(2020)Kulunchakov and
  Mairal]{kulunchakov_estimate_2020}
Kulunchakov, A. and Mairal, J.
\newblock Estimate sequences for stochastic composite optimization: Variance
  reduction, acceleration, and robustness to noise.
\newblock \emph{Journal of Machine Learning Research}, 21\penalty0
  (155):\penalty0 1--52, 2020.

\bibitem[Liu et~al.(2021)Liu, Li, Wei, Zhou, and Zhao]{liu_noisy_2021a}
Liu, T., Li, Y., Wei, S., Zhou, E., and Zhao, T.
\newblock Noisy gradient descent converges to flat minima for nonconvex matrix
  factorization.
\newblock In \emph{Proceedings of the {{International Conference}} on
  {{Artificial Intelligence}} and {{Statistics}}}, volume 130 of
  \emph{{{PMLR}}}, pp.\  1891--1899. {JMLR}, March 2021.

\bibitem[Ma et~al.(2018)Ma, Bassily, and Belkin]{ma_power_2018}
Ma, S., Bassily, R., and Belkin, M.
\newblock The power of interpolation: {{Understanding}} the effectiveness of
  {{SGD}} in modern over-parametrized learning.
\newblock In \emph{Proceedings of the {{International Conference}} on {{Machine
  Learning}}}, volume~80 of \emph{{{PMLR}}}, pp.\  3325--3334. {JMLR}, July
  2018.

\bibitem[Mishchenko et~al.(2020)Mishchenko, Khaled, and
  Richtarik]{mishchenko_random_2020}
Mishchenko, K., Khaled, A., and Richtarik, P.
\newblock Random reshuffling: Simple analysis with vast improvements.
\newblock In \emph{Advances in {{Neural Information Processing Systems}}},
  volume~33, pp.\  17309--17320. {Curran Associates, Inc.}, 2020.

\bibitem[Mohamed et~al.(2020)Mohamed, Rosca, Figurnov, and
  Mnih]{mohamed_monte_2020}
Mohamed, S., Rosca, M., Figurnov, M., and Mnih, A.
\newblock Monte {{Carlo}} gradient estimation in machine learning.
\newblock \emph{Journal of Machine Learning Research}, 21\penalty0
  (132):\penalty0 1--62, 2020.

\bibitem[Moulines \& Bach(2011)Moulines and Bach]{moulines_nonasymptotic_2011}
Moulines, E. and Bach, F.
\newblock Non-asymptotic analysis of stochastic approximation algorithms for
  machine learning.
\newblock In \emph{Advances in {{Neural Information Processing Systems}}},
  volume~24, pp.\  451--459. {Curran Associates, Inc.}, 2011.

\bibitem[Needell \& Ward(2017)Needell and Ward]{needell_batched_2017}
Needell, D. and Ward, R.
\newblock Batched stochastic gradient descent with weighted sampling.
\newblock In \emph{Approximation {{Theory XV}}: {{San Antonio}} 2016}, Springer
  {{Proceedings}} in {{Mathematics}} \& {{Statistics}}, pp.\  279--306, {Cham},
  2017. {Springer International Publishing}.

\bibitem[Needell et~al.(2016)Needell, Srebro, and
  Ward]{needell_stochastic_2016}
Needell, D., Srebro, N., and Ward, R.
\newblock Stochastic gradient descent, weighted sampling, and the randomized
  {{Kaczmarz}} algorithm.
\newblock \emph{Mathematical Programming}, 155\penalty0 (1):\penalty0 549--573,
  January 2016.

\bibitem[Nemirovski et~al.(2009)Nemirovski, Juditsky, Lan, and
  Shapiro]{nemirovski_robust_2009}
Nemirovski, A., Juditsky, A., Lan, G., and Shapiro, A.
\newblock Robust stochastic approximation approach to stochastic programming.
\newblock \emph{SIAM Journal on Optimization}, 19\penalty0 (4):\penalty0
  1574--1609, January 2009.

\bibitem[Nesterov \& Polyak(2006)Nesterov and Polyak]{nesterov_cubic_2006}
Nesterov, Y. and Polyak, B.
\newblock Cubic regularization of {{Newton}} method and its global performance.
\newblock \emph{Mathematical Programming}, 108\penalty0 (1):\penalty0 177--205,
  August 2006.

\bibitem[Nesterov(2005)]{nesterov_smooth_2005}
Nesterov, {\relax Yu}.
\newblock Smooth minimization of non-smooth functions.
\newblock \emph{Mathematical Programming}, 103\penalty0 (1):\penalty0 127--152,
  May 2005.

\bibitem[Nguyen et~al.(2018)Nguyen, Nguyen, {van Dijk}, Richtarik, Scheinberg,
  and Takac]{nguyen_sgd_2018}
Nguyen, L., Nguyen, P.~H., {van Dijk}, M., Richtarik, P., Scheinberg, K., and
  Takac, M.
\newblock {{SGD}} and {{Hogwild}}! {{Convergence}} without the bounded
  gradients assumption.
\newblock In \emph{Proceedings of the International Conference on Machine
  Learning}, volume~80 of \emph{{{PMLR}}}, pp.\  3750--3758. {JMLR}, July 2018.

\bibitem[Nguyen et~al.(2021)Nguyen, {Tran-Dinh}, Phan, Nguyen, and
  Van~Dijk]{nguyen_unified_2021}
Nguyen, L.~M., {Tran-Dinh}, Q., Phan, D.~T., Nguyen, P.~H., and Van~Dijk, M.
\newblock A unified convergence analysis for shuffling-type gradient methods.
\newblock \emph{The Journal of Machine Learning Research}, 22\penalty0
  (1):\penalty0 207:9397--207:9440, January 2021.

\bibitem[Orvieto et~al.(2023)Orvieto, Raj, Kersting, and
  Bach]{orvieto_explicit_2023}
Orvieto, A., Raj, A., Kersting, H., and Bach, F.
\newblock Explicit regularization in overparametrized models via noise
  injection.
\newblock In \emph{Proceedings of the {{International Conference}} on
  {{Artificial Intelligence}} and {{Statistics}}}, volume 206 of
  \emph{{{PMLR}}}, pp.\  7265--7287. {JMLR}, April 2023.

\bibitem[Polyak \& {d Aleksandr Borisovich}(1990)Polyak and {d Aleksandr
  Borisovich}]{polyak_optimal_1990}
Polyak, B.~T. and {d Aleksandr Borisovich}, T.
\newblock Optimal order of accuracy of search algorithms in stochastic
  optimization.
\newblock \emph{Problemy Peredachi Informatsii}, 26\penalty0 (2):\penalty0
  45--53, 1990.

\bibitem[Ranganath et~al.(2014)Ranganath, Gerrish, and
  Blei]{ranganath_black_2014}
Ranganath, R., Gerrish, S., and Blei, D.
\newblock Black box variational inference.
\newblock In \emph{Proceedings of the International Conference on Artificial
  Intelligence and Statistics}, volume~33 of \emph{{{PMLR}}}, pp.\  814--822.
  {JMLR}, April 2014.

\bibitem[Rezende et~al.(2014)Rezende, Mohamed, and
  Wierstra]{rezende_stochastic_2014}
Rezende, D.~J., Mohamed, S., and Wierstra, D.
\newblock Stochastic backpropagation and approximate inference in deep
  generative models.
\newblock In \emph{Proceedings of the {{International Conference}} on {{Machine
  Learning}}}, volume~32 of \emph{{{PMLR}}}, pp.\  1278--1286. {JMLR}, June
  2014.

\bibitem[Richt{\'a}rik \& Tak{\'a}{\v c}(2016)Richt{\'a}rik and Tak{\'a}{\v
  c}]{richtarik_parallel_2016}
Richt{\'a}rik, P. and Tak{\'a}{\v c}, M.
\newblock Parallel coordinate descent methods for big data optimization.
\newblock \emph{Mathematical Programming}, 156\penalty0 (1-2):\penalty0
  433--484, March 2016.

\bibitem[Robbins \& Monro(1951)Robbins and Monro]{robbins_stochastic_1951}
Robbins, H. and Monro, S.
\newblock A stochastic approximation method.
\newblock \emph{The Annals of Mathematical Statistics}, 22\penalty0
  (3):\penalty0 400--407, September 1951.

\bibitem[Safran \& Shamir(2020)Safran and Shamir]{safran_how_2020}
Safran, I. and Shamir, O.
\newblock How good is {{SGD}} with random shuffling?
\newblock In \emph{Proceedings of {{Conference}} on {{Learning Theory}}},
  volume 125 of \emph{{{PMLR}}}, pp.\  3250--3284. {JMLR}, July 2020.

\bibitem[Safran \& Shamir(2021)Safran and Shamir]{safran_random_2021}
Safran, I. and Shamir, O.
\newblock Random shuffling beats {{SGD}} only after many epochs on
  ill-conditioned problems.
\newblock In \emph{Advances in {{Neural Information Processing Systems}}},
  volume~34, pp.\  15151--15161. {Curran Associates, Inc.}, 2021.

\bibitem[Schmidt \& Roux(2013)Schmidt and Roux]{schmidt_fast_2013}
Schmidt, M. and Roux, N.~L.
\newblock Fast convergence of stochastic gradient descent under a strong growth
  condition.
\newblock {{arXiv}} Preprint arXiv:1308.6370, {arXiv}, August 2013.

\bibitem[{Shalev-Shwartz} et~al.(2009){Shalev-Shwartz}, Shamir, Srebro, and
  Sridharan]{shalev-shwartz_stochastic_2009}
{Shalev-Shwartz}, S., Shamir, O., Srebro, N., and Sridharan, K.
\newblock Stochastic convex optimization.
\newblock In \emph{Proceedings of the Conference on {{Computational}} Learning
  Theory}, June 2009.

\bibitem[{Shalev-Shwartz} et~al.(2011){Shalev-Shwartz}, Singer, Srebro, and
  Cotter]{shalev-shwartz_pegasos_2011}
{Shalev-Shwartz}, S., Singer, Y., Srebro, N., and Cotter, A.
\newblock Pegasos: Primal estimated sub-gradient solver for {{SVM}}.
\newblock \emph{Mathematical Programming}, 127\penalty0 (1):\penalty0 3--30,
  March 2011.

\bibitem[Shi et~al.(2021)Shi, Gu, Li, Deng, and Huang]{shi_triply_2021}
Shi, W., Gu, B., Li, X., Deng, C., and Huang, H.
\newblock Triply stochastic gradient method for large-scale nonlinear similar
  unlabeled classification.
\newblock \emph{Machine Learning}, 110\penalty0 (8):\penalty0 2005--2033,
  August 2021.

\bibitem[{Sohl-Dickstein} et~al.(2015){Sohl-Dickstein}, Weiss, Maheswaranathan,
  and Ganguli]{sohl-dickstein_deep_2015}
{Sohl-Dickstein}, J., Weiss, E., Maheswaranathan, N., and Ganguli, S.
\newblock Deep unsupervised learning using nonequilibrium thermodynamics.
\newblock In \emph{Proceedings of the {{International Conference}} on {{Machine
  Learning}}}, volume~37 of \emph{{{PMLR}}}, pp.\  2256--2265. {JMLR}, June
  2015.

\bibitem[Song et~al.(2013)Song, Chaudhuri, and Sarwate]{song_stochastic_2013}
Song, S., Chaudhuri, K., and Sarwate, A.~D.
\newblock Stochastic gradient descent with differentially private updates.
\newblock In \emph{Proceedings of the {{IEEE Global Conference}} on {{Signal}}
  and {{Information Processing}}}, pp.\  245--248, {Austin, TX, USA}, December
  2013. {IEEE}.

\bibitem[Song \& Ermon(2019)Song and Ermon]{song_generative_2019}
Song, Y. and Ermon, S.
\newblock Generative modeling by estimating gradients of the data distribution.
\newblock In \emph{Advances in {{Neural Information Processing Systems}}},
  volume~32, pp.\  11918--11930. {Curran Associates, Inc.}, 2019.

\bibitem[Stich(2019)]{stich_unified_2019}
Stich, S.~U.
\newblock Unified optimal analysis of the (stochastic) gradient method.
\newblock Preprint arXiv:1907.04232, {arXiv}, December 2019.

\bibitem[Titsias \& {L{\'a}zaro-Gredilla}(2014)Titsias and
  {L{\'a}zaro-Gredilla}]{titsias_doubly_2014}
Titsias, M. and {L{\'a}zaro-Gredilla}, M.
\newblock Doubly stochastic variational {{Bayes}} for non-conjugate inference.
\newblock In \emph{Proceedings of the {{International Conference}} on {{Machine
  Learning}}}, volume~32 of \emph{{{PMLR}}}, pp.\  1971--1979. {JMLR}, June
  2014.

\bibitem[Vapnik(1991)]{vapnik_principles_1991}
Vapnik, V.
\newblock Principles of risk minimization for learning theory.
\newblock In \emph{Advances in {{Neural Information Processing Systems}}},
  volume~4, pp.\  831--838. {Morgan-Kaufmann}, 1991.

\bibitem[Vaswani et~al.(2019)Vaswani, Bach, and Schmidt]{vaswani_fast_2019}
Vaswani, S., Bach, F., and Schmidt, M.
\newblock Fast and faster convergence of {{SGD}} for over-parameterized models
  and an accelerated perceptron.
\newblock In \emph{Proceedings of the {{International Conference}} on
  {{Artificial Intelligence}} and {{Statistics}}}, volume~89 of
  \emph{{{PMLR}}}, pp.\  1195--1204. {JMLR}, April 2019.

\bibitem[Wright \& Recht(2021)Wright and Recht]{wright_optimization_2021}
Wright, S.~J. and Recht, B.
\newblock \emph{Optimization for Data Analysis}.
\newblock {Cambridge University Press}, {New York}, 2021.

\bibitem[Xie et~al.(2015)Xie, Liang, and Song]{xie_scale_2015}
Xie, B., Liang, Y., and Song, L.
\newblock Scale up nonlinear component analysis with doubly stochastic
  gradients.
\newblock In \emph{Advances in {{Neural Information Processing Systems}}},
  volume~28, pp.\  2341--2349. {Curran Associates, Inc.}, 2015.

\bibitem[Xu et~al.(2019)Xu, Quiroz, Kohn, and Sisson]{xu_variance_2019}
Xu, M., Quiroz, M., Kohn, R., and Sisson, S.~A.
\newblock Variance reduction properties of the reparameterization trick.
\newblock In \emph{Proceedings of the {{International Conference}} on
  {{Artificial Intelligence}} and {{Statistics}}}, volume~89 of
  \emph{{{PMLR}}}, pp.\  2711--2720. {JMLR}, April 2019.

\bibitem[Zheng \& Kwok(2018)Zheng and Kwok]{zheng_lightweight_2018}
Zheng, S. and Kwok, J. T.-Y.
\newblock Lightweight stochastic optimization for minimizing finite sums with
  infinite data.
\newblock In \emph{Proceedings of the {{International Conference}} on {{Machine
  Learning}}}, volume~80 of \emph{{{PMLR}}}, pp.\  5932--5940. {JMLR}, July
  2018.

\end{thebibliography}
